\documentclass[sigconf]{acmart}

\def\thisismainpaper{0}   

\usepackage{graphicx}
\usepackage{amsmath}
\usepackage{blkarray,booktabs,bigstrut} 
\usepackage{multirow}
\usepackage{epsfig}
\usepackage{mathrsfs}
\usepackage{comment}
\usepackage{array}
\usepackage{amsthm}
\usepackage{blkarray}
\usepackage{enumerate}
\usepackage{url}
\usepackage{hyperref}
\usepackage{chemarrow}
\usepackage{color}
\usepackage{xcolor}
\usepackage{bbm}

\DeclareMathOperator{\argmin}{\arg\min}

\def\ddefloop#1{\ifx\ddefloop#1\else\ddef{#1}\expandafter\ddefloop\fi}
\def\ddef#1{\expandafter\def\csname bb#1\endcsname{\ensuremath{\mathbb{#1}}}}
\ddefloop ABCDEFGHIJKLMNOPQRSTUVWXYZ\ddefloop
\def\ddef#1{\expandafter\def\csname c#1\endcsname{\ensuremath{\mathcal{#1}}}}
\ddefloop ABCDEFGHIJKLMNOPQRSTUVWXYZ\ddefloop
\def\ddef#1{\expandafter\def\csname v#1\endcsname{\ensuremath{\boldsymbol{#1}}}}
\ddefloop ABCDEFGHIJKLMNOPQRSTUVWXYZabcdefghijklmnopqrstuvwxyz\ddefloop
\def\ddef#1{\expandafter\def\csname u#1\endcsname{\ensuremath{\underline{#1}}}}
\ddefloop ABCDEFGHIJKLMNOPQRSTUVWXYZabcdefghijklmnopqrstuvwxyz\ddefloop
\def\ddef#1{\expandafter\def\csname v#1\endcsname{\ensuremath{\boldsymbol{\csname #1\endcsname}}}}
\ddefloop {alpha}{beta}{gamma}{delta}{epsilon}{varepsilon}{zeta}{eta}{theta}{vartheta}{iota}{kappa}{lambda}{mu}{nu}{xi}{pi}{varpi}{rho}{varrho}{sigma}{varsigma}{tau}{upsilon}{phi}{varphi}{chi}{psi}{omega}{Gamma}{Delta}{Theta}{Lambda}{Xi}{Pi}{Sigma}{varSigma}{Upsilon}{Phi}{Psi}{Omega}{ell}\ddefloop

\usepackage[lined,ruled,linesnumbered,noend]{algorithm2e}
\makeatletter
\newcommand{\nosemic}{\renewcommand{\@endalgocfline}{\relax}}
\newcommand{\dosemic}{\renewcommand{\@endalgocfline}{\algocf@endline}}

\let\oldnl\nl
\newcommand{\nonl}{\renewcommand{\nl}{\let\nl\oldnl}}
\makeatother

\usepackage{epsf,psfrag}
\usepackage{epsfig}
\usepackage{bm}
\usepackage{mathtools}
\usepackage{multicol,lipsum} 

\def\ting#1{\textcolor{red}{Ting: #1}}

\newcommand{\rev}[1]{\textcolor{black}{#1}} 

\theoremstyle{definition}
\newtheorem{theorem}{Theorem}[section]
\newtheorem{lemma}[theorem]{Lemma}
\newtheorem{corollary}[theorem]{Corollary}

\newtheorem{definition}{Definition}

\newcommand{\E}{{\rm I\kern-.3em E}}

\def\diag{\operatorname{diag}}
\def\uG{\underline{G}}
\def\uV{\underline{V}}
\def\uE{\underline{E}}
\def\up{\underline{p}}
\def\ue{\underline{e}}

\usepackage[tight,footnotesize]{subfigure}

\allowdisplaybreaks

\usepackage{setspace}

\acmYear{2024}\copyrightyear{2024}
\setcopyright{acmlicensed}
\acmConference[MobiHoc '24]{International Symposium on Theory, Algorithmic Foundations, and Protocol Design for Mobile Networks and Mobile Computing}{October 14--17, 2024}{Athens, Greece}
\acmBooktitle{International Symposium on Theory, Algorithmic Foundations, and Protocol Design for Mobile Networks and Mobile Computing (MobiHoc '24), October 14--17, 2024, Athens, Greece}
\acmDOI{10.1145/3641512.3686364}
\acmISBN{979-8-4007-0521-2/24/10}

\begin{document}

\title{
Overlay-based Decentralized Federated Learning in Bandwidth-limited Networks
}

\author{Yudi Huang}
\authornote{Both authors contributed equally to the paper. 
This work was supported by the National Science Foundation under award CNS-2106294 and CNS-1946022.}
\affiliation{%
  \institution{Pennsylvania State University}
  \city{State College, PA}
  \country{USA}}
\email{yxh5389@psu.edu}

\author{Tingyang Sun}
\authornotemark[1]
\affiliation{%
  \institution{Pennsylvania State University}
  \city{State College, PA}
  \country{USA}}
\email{tfs5679@psu.edu}

\author{Ting He}
\if\thisismainpaper0
\authornote{This is the extended version with proofs and additional evaluations.}
\fi
\affiliation{%
  \institution{Pennsylvania State University}
  \city{State College, PA}
  \country{USA}}
\email{tinghe@psu.edu}

\renewcommand{\shortauthors}{Huang et al.}


\begin{abstract}
The emerging machine learning paradigm of decentralized federated learning (DFL) has the promise of greatly boosting the deployment of artificial intelligence (AI) by directly learning across distributed agents without centralized coordination. Despite significant efforts on improving the communication efficiency of DFL, most existing solutions were based on the simplistic assumption that neighboring agents are physically adjacent in the underlying communication network, which fails to correctly capture the communication cost when learning over a general bandwidth-limited network, as encountered in many edge networks. In this work, we address this gap by leveraging recent advances in network tomography to jointly design the communication demands and the communication schedule for overlay-based DFL in bandwidth-limited networks without requiring explicit cooperation from the underlying network. By carefully analyzing the structure of our problem, we decompose it into a series of optimization problems that can each be solved efficiently, to collectively minimize the total training time. Extensive data-driven simulations show that our solution can significantly accelerate DFL in comparison with state-of-the-art designs. \looseness=-1
\end{abstract}



\if\thisismainpaper1
\settopmatter{printfolios=false} 
\else
\settopmatter{printfolios=true} 
\fi

\maketitle

\section{Introduction}\label{sec:Introduction}

\if\thisismainpaper1

As a new machine learning paradigm, \emph{decentralized federated learning (DFL)}~\cite{Lian17NIPS} allows multiple learning agents to collaboratively learn a shared model from the union of their local data without directly sharing the data. To achieve this goal, the agents repeatedly exchange model updates with their neighbors through peer-to-peer connections, which are then aggregated locally~\cite{Kairouz21book}. Since its introduction, DFL has attracted significant attention, because compared to centralized \emph{federated learning (FL)}~\cite{McMahan17AISTATS}, DFL can avoid a single point of failure and reduce the communication complexity at the busiest node without increasing the computational complexity~\cite{Lian17NIPS}.

\else
As a new machine learning paradigm, \emph{federated learning (FL)} allows multiple learning agents to collaboratively learn a shared model from the union of their local data without directly sharing the data \cite{McMahan17AISTATS}. To achieve this goal, the agents repeatedly exchange model updates, through a centralized parameter server~\cite{McMahan17AISTATS}, a hierarchy of parameter servers~\cite{Liu23TWC}, or peer-to-peer links between neighboring agents~\cite{Kairouz21book}, which are then aggregated to update the shared model. 
Due to its promise in protecting data privacy, FL has found many applications such as improving browsers \cite{FLoC,BraveFL}.
%
%
In particular, \emph{decentralized federated learning (DFL)}~\cite{Lian17NIPS} has attracted significant attention. Instead of forming a star \cite{McMahan17AISTATS} or hierarchical topology \cite{Liu23TWC}, agents in DFL can communicate along arbitrary topology, where parameter exchanges only occur between neighbors. Compared to centralized FL, DFL can avoid a single point of failure and reduce the communication complexity at the busiest node without increasing the computational complexity~\cite{Lian17NIPS}.
\fi

Meanwhile, FL including DFL faces significant performance challenges due to the extensive data transfer. Although the training data stay local, the agents still need to communicate frequently to exchange local model updates, which often incurs a nontrivial communication cost due to the large model size. Such communication cost can dominate the total cost of the learning task, e.g., accounting for up to $90\%$ of time in cloud-based FL~\cite{Luo20MLsys}, and the problem is exacerbated in other networks that are more bandwidth-limited (e.g., wireless mesh networks~\cite{chen2022federated}).  
This issue has attracted tremendous interests in reducing the communication cost, including compression-based methods for reducing the amount of data per communication such as \cite{Compression1} and methods for reducing the number of communications through hyperparameter optimization such as \cite{Wang19JSAC,MATCHA22,Chiu23JSAC} or adaptive communications such as \cite{Singh20CDC}. 

\begin{figure}[t!]
\vspace{-.0em}
\centerline{\mbox{\includegraphics[width=.9\linewidth]{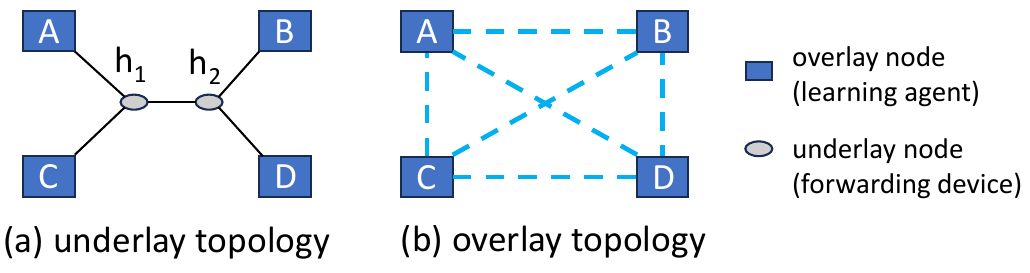}}}
\vspace{-1em}
\caption{\small Overlay-based DFL. }
\label{fig:overlay_underlay}
\vspace{-.5em}
\end{figure}

However, most existing works made the simplistic assumption that each pair of {logically adjacent} agents are also {physically adjacent} in the underlying communication network. 
This is not true in \emph{overlay-based DFL}, where the connections between logically adjacent agents (i.e., \emph{overlay links}) can map to arbitrary routing paths in the underlying communication network that may share links (i.e., \emph{underlay links}). 
For example, a set of learning agents $\{A,B,C,D\}$ may have the physical connectivity in Fig.~\ref{fig:overlay_underlay}a and the logical connectivity in Fig.~\ref{fig:overlay_underlay}b. Although connections $(A,B)$ and $(C,D)$ appear disjoint in the overlay, they actually map to paths sharing link $(h_1,h_2)$ in the underlay. Ignoring such link sharing can cause incorrect prediction of the communication time, because concurrent communications over connections with shared links can take longer than stand-alone communication on each of them, and the problem can be exacerbated by heterogeneous capacities and background loads at underlay links. 
Most existing works ignored such complications by assuming the communication time to be proportional to the maximum number of neighbors an agent communicates with   \cite{MATCHA22,Chiu23JSAC,hua2022efficient,le2023refined}, which generally leads to suboptimal designs in the case of overlay-based DFL.  

We want to address this gap \emph{without requiring explicit cooperation from the underlay network}, i.e., the agents can neither directly observe the internal state of the underlay (e.g., routing topology, link capacities) nor control its internal behavior. 
Such scenarios arise naturally when the agents are interconnected by a public network.  
In particular, we are interested in running DFL over bandwidth-limited networks, which models many application scenarios such as HetNets \cite{chen2022federated}, device-to-device networks \cite{xing2021federated}, IoT networks~\cite{pinyoanuntapong2022toward}, underwater networks \cite{pei2023fed}, and power line communication networks \cite{jia2022dispatching}. 
In contrast to high-bandwidth networks such as inter-datacenter networks \cite{marfoq2020throughput}, bandwidth-limited networks are more sensitive to communication demands generated by DFL and are thus in greater needs of proper designs. 
To this end, we propose an optimization framework for overlay-based DFL that jointly designs the \emph{communication demands} and the \emph{communication schedule} within the overlay, without explicit cooperation from the underlay. 
Building upon recent advances in network tomography~\cite{Huang23MobiHoc} and mixing matrix design~\cite{Chiu23JSAC}, we cast the problem into a set of tractable optimizations that collectively minimize the total time in achieving a given level of convergence. \looseness=0

\subsection{Related Work}\label{subsec:Related Work}

\if\thisismainpaper1 

\textbf{Decentralized federated learning.} Initially proposed under a centralized architecture \cite{McMahan17AISTATS}, FL was later extended to a fully decentralized architecture~\cite{Lian17NIPS}, which was shown to achieve the same computational complexity but a lower communication complexity. 
Since then a number of improvements such as \cite{ICMLhonor} have been developed, but these works only focused on reducing the number of iterations. \looseness=-1

\textbf{Communication cost reduction.} 
There are two general approaches for reducing the communication cost in FL: reducing the amount of data per communication through compression, e.g., \cite{Compression1}, 
and reducing the number of communications, e.g.,  \cite{Wang19JSAC}. 
The two approaches can be combined for further improvement \cite{Singh20CDC}. 
Instead of either activating all the links or activating none, it has been recognized that better efficiency can be achieved by activating subsets of links. To this end, \cite{Singh20CDC} proposed an event-triggered mechanism and \cite{MATCHA22,Chiu23JSAC} proposed to activate links with predetermined probabilities. 
In this regard, our work designs predetermined link activation as in \cite{MATCHA22,Chiu23JSAC}, which provides more predictable performance than event-triggered mechanisms, but \emph{we consider a cost model tailored to overlay-based DFL}: instead of measuring the communication time by the number of matchings \cite{MATCHA22,Chiu23JSAC} or the maximum degree  \cite{hua2022efficient,le2023refined}, we evaluate the minimum time to complete all the activated agent-to-agent communications over a bandwidth-limited underlay, while taking into account heterogeneous capacities and possibly shared links. \looseness=-1 

\textbf{Topology design in DFL.} The logical topology defining the neighborhoods of learning agents is an important design parameter in DFL that controls the communication demands during training. The impact of this topology on the convergence rate of DFL has been mostly captured through the {spectral gap} of the mixing matrix~\cite{Lian17NIPS,neglia2020decentralized,jiang2023joint} or equivalent parameters~\cite{MATCHA22}. 
Although recent works have identified other parameters that can impact the convergence rate, such as the effective number of neighbors~\cite{vogels2022beyond} and the neighborhood heterogeneity~\cite{le2023refined}, these results just pointed out additional factors and did not invalidate the impact of spectral gap. 
Based on the identified convergence parameters, several solutions have been proposed to design the logical topology to balance the convergence rate and the cost per communication round~\cite{MATCHA22,Chiu23JSAC,le2023refined}, and some solutions combined topology design with other optimizations (e.g., model pruning~\cite{jiang2023joint}) for further improvement. 
In this regard, our work also includes topology design based on a parameter related to the spectral gap, but \emph{we explicitly consider the communication schedule to serve the demands triggered by the designed topology over a bandwidth-limited underlay} to optimize the overall wall-clock time of overlay-based DFL. 
To our knowledge, the only existing work addressing overlay-based DFL is \cite{marfoq2020throughput}. However, it assumed a special underlay where the paths connecting learning agents only share links at the first and the last hops, whose capacities are assumed to be known. 
While this model may suit high-bandwidth underlays such as inter-datacenter networks, it fails to capture the communication cost in  bandwidth-limited underlays as addressed in our work. \looseness=-1

\textbf{Network-aware distributed computing.}
It was known that awareness to the state of the communication underlay is important for data-intensive distributed computing tasks~\cite{Luo20MLsys}. Several works attempted to solve this problem in the context of cloud networks, assuming either a black-box network \cite{Luo20MLsys} or a white-box network \cite{Chen22NSDI}. 
%
In this regard, our work assumes a black-box underlay as in \cite{Luo20MLsys}, but unlike the simple heuristics used in these works, we leverage state-of-the-art techniques from \emph{network tomography}~\cite{Huang23MobiHoc} to estimate the necessary parameters about the underlay. 

\else 

\textbf{Decentralized federated learning.} Initially proposed under a centralized architecture \cite{McMahan17AISTATS}, FL was later extended to a fully decentralized architecture~\cite{Lian17NIPS}, which was shown to achieve the same computational complexity but a lower communication complexity. 
Since then a number of improvements have been developed, e.g., \cite{D2ICML18,ICMLhonor}, but these works only focused on reducing the number of iterations. \looseness=-1
%

\textbf{Communication cost reduction.} 
There are two general approaches for reducing the communication cost in FL: reducing the amount of data per communication through compression, e.g., \cite{Compression1,Compression2,Compression3}, 
and reducing the number of communications, e.g.,  \cite{sysml19,Ngu19INFOCOM,Wang19JSAC}. 
The two approaches can be combined for further improvement \cite{Singh20CDC,Singh21JSAIT}. 
Instead of either activating all the links or activating none, it has been recognized that better efficiency can be achieved by activating subsets of links. To this end, \cite{Singh20CDC,Singh21JSAIT} proposed an event-triggered mechanism and \cite{MATCHA22,Chiu23JSAC} proposed to activate links with predetermined probabilities. 
In this regard, our work designs predetermined link activation as in \cite{MATCHA22,Chiu23JSAC}, which provides more predictable performance than event-triggered mechanisms, but \emph{we consider a cost model tailored to overlay-based DFL}: instead of measuring the communication time by the number of matchings  \cite{MATCHA22,Chiu23JSAC} or the maximum degree  \cite{hua2022efficient,le2023refined}, we evaluate the minimum time to complete all the activated agent-to-agent communications over a bandwidth-limited underlay, while taking into account heterogeneous capacities and possibly shared links. 

\textbf{Topology design in DFL.} The logical topology defining the neighborhoods of learning agents is an important design parameter in DFL that controls the communication demands during training. The impact of this topology on the convergence rate of DFL has been mostly captured through the {spectral gap} of the mixing matrix~\cite{Lian17NIPS,Nedic18IEEE,Neglia19INFOCOM,neglia2020decentralized,jiang2023joint} or equivalent parameters~\cite{MATCHA22}. 
Although recent works have identified other parameters that can impact the convergence rate, such as the effective number of neighbors~\cite{vogels2022beyond} and the neighborhood heterogeneity~\cite{le2023refined}, these results just pointed out additional factors and did not invalidate the impact of spectral gap. 
Based on the identified convergence parameters, several solutions have been proposed to design the logical topology to balance the convergence rate and the cost per communication round~\cite{MATCHA22,Chiu23JSAC,le2023refined}, and some solutions combined topology design with other optimizations (e.g., bandwidth allocation~\cite{Wang22Networking}, model pruning~\cite{jiang2023joint}) for further improvement. 
In this regard, our work also includes topology design based on a parameter related to the spectral gap, but \emph{we explicitly consider the communication schedule to serve the demands triggered by the designed topology over a bandwidth-limited underlay} to optimize the overall wall-clock time of overlay-based DFL. 
To our knowledge, the only existing work addressing overlay-based DFL is \cite{marfoq2020throughput}. However, it assumed a special underlay where the paths connecting learning agents only share links at the first and the last hops, whose capacities are assumed to be known. 
While this model may suit high-bandwidth underlays such as inter-datacenter networks, it fails to capture the communication cost in  bandwidth-limited underlays as addressed in our work. \looseness=-1

\textbf{Network-aware distributed computing.}
It was known that awareness to the state of the communication underlay is important for data-intensive distributed computing tasks~\cite{Luo20MLsys}. Several works attempted to solve this problem in the context of cloud networks, assuming either a black-box network \cite{Luo20MLsys,Gong15TPDS,LaCurts2013IMC} or a white-box network \cite{Chen22NSDI}. 
%
In this regard, our work assumes a black-box underlay as in \cite{Luo20MLsys,Gong15TPDS,LaCurts2013IMC}, but unlike the simple heuristics used in these works, we leverage state-of-the-art techniques from \emph{network tomography}~\cite{He21book} to estimate the necessary parameters about the underlay. 
Notably, we recently discovered in \cite{Huang23MobiHoc} that network tomography can consistently estimate coarse parameters about the underlay  that allows the overlay to compute the capacity region for communications between the overlay nodes, which we will use to design the communication schedule for overlay-based DFL.  

\fi

\subsection{Summary of Contributions}

We jointly design the communication demands and the communication schedule for overlay-based DFL in a bandwidth-limited uncooperative underlay, with the following contributions:\looseness=-1
\begin{enumerate}
\item  We consider, for the first time, communication optimization in DFL on top of a bandwidth-limited underlay network with arbitrary topology. To this end, we propose a general framework for jointly designing the communication demands and the communication schedule (e.g., routing, rates) among the learning agents, without cooperation from the underlay. 

\item We tackle the complexity challenge by decomposing the overall problem into a series of smaller subproblems, that are collectively designed to minimize the total training time to achieve a given level of convergence. Through carefully designed relaxations, we convert each subproblem into a tractable optimization to develop efficient solutions.
    
\item We evaluate the proposed solution in comparison with benchmarks based on real network topologies and datasets. Our results show that (i) our design of the communication demands can already reduce the training time substantially without compromising the quality of the trained model, (ii) our design of the communication schedule further increases the improvement, and (iii) the observations remain valid under realistic inference errors about the underlay. 
\end{enumerate}

\textbf{Roadmap.} Section~\ref{sec:Background and Formulation} describes our overall problem, Section~\ref{sec:Proposed Solution} presents our solution and analysis, Section~\ref{sec:Performance Evaluation} presents our performance evaluation, and Section~\ref{sec:Conclusion} concludes the paper. 
\if\thisismainpaper1
\textbf{All the proofs can be found in \cite[Appendix~A]{Huang24:report}.}
\else
\textbf{All the proofs can be found in Appendix~\ref{appendix:Proofs}.} 
\fi


\section{Problem Formulation}\label{sec:Background and Formulation}

\subsection{Notations}\label{subsec:Notations}

Let $\bm{a}\in \mathbb{R}^m$ denote a vector and $\bm{A}\in \mathbb{R}^{m\times m}$ a matrix. We use $\|\bm{a}\|$ to denote the $\ell$-2 norm, $\|\bm{A}\|$ to denote the spectral norm, and $\|\bm{A}\|_F$ to denote the Frobenius norm. We use $\diag(\bm{a})$ to denote a diagonal matrix with the entries in $\bm{a}$ on the main diagonal, and $\diag(\bm{A})$ to denote a vector formed by the diagonal entries of $\bm{A}$. We use 
$\lambda_i(\bm{A})$ ($i=1,\ldots,m$) to denote the $i$-th smallest eigenvalue of $\bm{A}$. \looseness=-1

\subsection{Network Model}

Consider a network of $m$ learning agents connected through a logical \emph{base topology} $G=(V, E)$ ($|V|=m$), that forms an overlay on top of a communication underlay $\uG=(\uV, \uE)$. Unless otherwise stated, both overlay and underlay links are considered directed. 
Each underlay link $\ue\in \uE$ has a finite capacity $C_{\ue}$. 
Each overlay link $e=(i,j)\in E$ indicates that agent $i$ is \emph{allowed} to communicate to agent $j$ during learning, and is implemented via a routing path $\up_{i,j}$ from the node running agent $i$ to the node running agent $j$ in the underlay. We assume that if $(i,j)\in E$, then $(j,i)\in E$ (agents $i$ and $j$ are allowed to exchange results). The routing paths are determined by the topology and the routing protocol in the underlay. Let $l_{i, j}$ denote the propagation delay on $\up_{i, j}$. 
We assume that neither the routing paths nor the link capacities in the underlay are observable by the overlay, but the propagation delays between overlay nodes (e.g., $l_{i,j}$) are observable\footnote{This can be obtained by measuring the delays of small probing packets.}.

\subsection{Decentralized Federated Learning (DFL)}\label{subsec:Model of Decentralized Learning}

Consider a DFL task, where
each agent $i\in V$ has a possibly non-convex objective function $F_i(\bm{x})$ that depends on the parameter vector $\bm{x} \in \mathbb{R}^d$ and the local dataset $\mathcal{D}_i$, and the goal is to find the parameter vector $\bm{x}$ that minimizes the global objective function $F(\bm{x})$, defined as\looseness=-1
\begin{align}\label{eq:F(x)}
    F(\boldsymbol{x}) := \frac{1}{m}\sum_{i=1}^{m} F_{i}(\boldsymbol{x}).
\end{align}
For example, we can model the objective of empirical risk minimization by defining the local objective as $F_i(\bm{x}):= \sum_{\bm{s}\in \mathcal{D}_i}\ell(\bm{x},\bm{s})$, where $\ell(\bm{x},\bm{s})$ is the loss function for sample $\bm{s}$ under model $\bm{x}$, and the corresponding global objective is proportional to the empirical risk over all the samples.  
%
%

We consider a standard decentralized training algorithm called D-PSGD~\cite{Lian17NIPS}, where each agent repeatedly updates its own parameter vector and aggregates it with the parameter vectors of its neighbors to  minimize the global objective function. Specifically, let $\bm{x}_i^{(k)}$ ($k\geq 1$) denote the parameter vector at agent $i$ after $k-1$ iterations and $g(\bm{x}_i^{(k)}; \xi_i^{(k)})$ the stochastic gradient computed by agent $i$ in iteration $k$ (where $\xi_i^{(k)}$ is the mini-batch). In iteration $k$, agent $i$ updates its parameter vector by\looseness=-1
\begin{align}\label{eq:DecenSGD}
\boldsymbol{x}^{(k+1)}_i = \sum_{j=1}^{m}W^{(k)}_{ij}\boldsymbol{x}^{(k)}_j - \eta g(\boldsymbol{x}^{(k)}_i; \xi^{(k)}_i),
\end{align}
where $\bm{W}^{(k)}=(W^{(k)}_{ij})_{i,j=1}^m$ is the $m\times m$ \emph{mixing matrix} in iteration $k$, and $\eta>0$ is the learning rate. 
To be consistent with the base topology, $W^{(k)}_{ij}\neq 0$ only if $(i,j)\in E$. 
The update rule in \eqref{eq:DecenSGD} has the same convergence performance as $\boldsymbol{x}^{(k+1)}_i = \sum_{j=1}^{m}W^{(k)}_{ij}(\boldsymbol{x}^{(k)}_j - \eta g(\boldsymbol{x}^{(k)}_j; \xi^{(k)}_j))$~\cite{Lian17NIPS}, but \eqref{eq:DecenSGD} allows each agent to parallelize the parameter exchange with neighbors and the gradient computation. 



The mixing matrix $\bm{W}^{(k)}$ plays an important role in controlling 
the communication cost, as agent $j$ needs to send its parameter vector to agent $i$ in iteration $k$ if and only if $W^{(k)}_{ij}\neq 0$. According to \cite{Lian17NIPS}, the mixing matrix should be \emph{symmetric with each row/column summing up to one}\footnote{In \cite{Lian17NIPS}, the mixing matrix was assumed to be symmetric and \emph{doubly stochastic} with entries constrained to $[0,1]$, but we find this requirement unnecessary for the convergence bound we use from \cite[Theorem~2]{Koloskova20ICML}, which only requires the mixing matrix to be symmetric with each row/column summing up to one.} in order to ensure convergence for D-PSGD. 
The symmetry implies a one-one correspondence between distinct (possibly) non-zero entries in $\bm{W}^{(k)}$ and the \emph{undirected overlay links}, denoted by $\widetilde{E}$ (i.e., each $(i,j)\in \widetilde{E}$ represents a pair of directed links $\{(i,j), (j,i)\}\in E$), and thus $W^{(k)}_{ij}$ can be interpreted as the \emph{link weight} of the undirected overlay link $(i,j)\in \widetilde{E}$. The requirement of each row summing to one further implies that $W^{(k)}_{ii} = 1-\sum_{j=1}^m W^{(k)}_{ij}$. In the vector form, the above implies the following decomposition of the mixing matrix 
\begin{align}
\bm{W}^{(k)} := \bm{I} - \bm{B} \diag(\bm{\alpha}^{(k)})\bm{B}^\top, \label{eq:W}
\end{align}
where $\bm{I}$ is the $m\times m$ identity matrix, $\bm{B}$ is the $|V|\times|\widetilde{E}|$ incidence matrix\footnote{This is defined under an arbitrary orientation of each link $e_j\in \widetilde{E}$ as $B_{i j}= +1$ if $e_j$ starts from $i$, $-1$ if $e_j$ ends at $i$, and $0$ otherwise.} for the base topology $G$, and $\bm{\alpha}^{(k)}:=(\alpha^{(k)}_{ij})_{(i,j)\in \widetilde{E}}$ is the vector of link weights. It is easy to verify that $W^{(k)}_{ij} = \alpha^{(k)}_{ij}$. This decomposition reduces the design of mixing matrix to the design of link weights $\bm{\alpha}^{(k)}$ in the overlay, where agents $i$ and $j$ need to exchange parameter vectors in iteration $k$ \emph{if and only if} $\alpha^{(k)}_{ij}\neq 0$. Thus, we say that \emph{the (undirected) overlay link $(i,j)$ is activated} in iteration $k$ (i.e., both $(i,j)$ and $(j,i)$ are activated) if $\alpha^{(k)}_{ij}\neq 0$.

\subsection{Design Objective}
\label{subsec:Goal: Network-aware Communication Optimization for Decentralized Learning}

Our goal is to \emph{jointly design the communication demands between the agents and the communication schedule about how to service these demands} so as to minimize the total (wall-clock) time for the learning task to reach a given level of convergence. The challenges are two-fold: (i) the design of communication demands faces the tradeoff between communicating more per iteration and converging in fewer iterations versus communicating less per iteration and converging in more iterations, and (ii) the design of communication schedule faces the lack of observability and controllability within the underlay network. 
Below, we will tackle these challenges by combining techniques from network tomography and mixing matrix design. 

\section{Proposed Solution}\label{sec:Proposed Solution}

Our approach is to first characterize the total training time as an explicit function of the set of activated links in the overlay, and then optimize this set. We will focus on a deterministic design that can give a predictable training time, and thus the iteration index $k$ will be omitted. For ease of presentation, we will consider the set of activated overlay links, denoted by $E_a\subseteq \widetilde{E}$, as \emph{undirected links}, as the pair of links between two agents must be activated at the same time.   
\looseness=-1

\subsection{Communication Schedule Optimization}\label{subsec:Communication Schedule Optimization}

\begin{figure}[t!]
\vspace{-.0em}
\centerline{\mbox{\includegraphics[width=.8\linewidth]{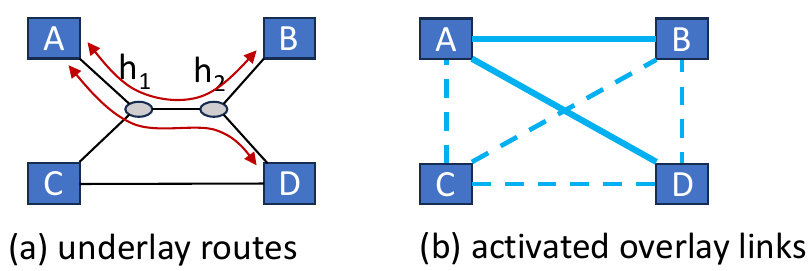}}}
\vspace{-1em}
\caption{\small Underlay-aware communication schedule optimization (learning agents: $\{A,B,C,D\}$; underlay nodes: $\{h_1,h_2\}$). }
\label{fig:overlay_underlay_routing}
\vspace{-.25em}
\end{figure}

Given a set of overlay links $E_a\subseteq \widetilde{E}$ activated in an iteration, each $(i,j)\in E_a$ triggers two communications, one for agent $i$ to send its parameter vector to agent $j$ and the other for agent $j$ to send its parameter vector to agent $i$. However, directly sending the parameter vectors along the underlay routing paths can lead to suboptimal performance. For example, consider Fig.~\ref{fig:overlay_underlay_routing}. If $E_a=\{(A,B),(A,D)\}$ but both $\up_{A,B}$ and $\up_{A,D}$ traverse the same underlay link $(h_1,h_2)$, directly communicating between the activated agent pairs can take longer than redirecting part of the traffic through other agents (e.g., redirecting $A\to D$ traffic through the overlay path $A\to C\to D$). The same holds if the capacity of the direct path is low, but the capacity through other agents is higher (e.g., if $(h_2,D)$ is a slow link, then redirecting $A\to D$ traffic through $C$ can bypass it to achieve a higher rate). 
This observation motivates the need of optimizing how to serve the demands triggered by the activated links by routing within the overlay. 

\subsubsection{Demand Model}

Let $\kappa_i$ denote the size of the parameter vector (or its compressed version if model compression is used) at agent $i$. A straightforward way to model the communication demands triggered by a set of activated links $E_a$ is to generate two unicast flows for each activated link $(i,j)\in E_a$, one in each direction. However, this model will lead to a suboptimal communication schedule as it ignores the fact that some flows carry identical content. Specifically, all flows originating from the same agent will carry the latest parameter vector at this agent. Thus, the actual communication demands is a set of multicast flows, each for distributing the parameter vector of an activated agent (incident to at least one activated link) to the agents it needs to share parameters with. Let $N_{E_a}(i):=\{j\in V:\: (i,j)\in E_a\}$. We can express the demands triggered by the activated links $E_a$ as 
\begin{align}\label{eq:demands H}
    H=\{(i,N_{E_a}(i),\kappa_i):\: \forall i\in V \mbox{ with } N_{E_a}(i)\neq \emptyset\},
\end{align} 
where each $h=(s_h,T_h,\kappa_h)\in H$ represents a multicast flow with source $s_h$, destinations $T_h$, and data size $\kappa_h$.

\subsubsection{Baseline Formulation}

To help towards minimizing the total training time, the communication schedule should minimize the time for completing all the communication demands triggered by the activated links, {within the control of the overlay}. 
To this end, we \emph{jointly optimize the routing and the flow rate within the overlay}. The former is represented by decision variables $z^h_{ij}\in \{0,1\}$ that indicates whether overlay link $(i,j)$ is traversed by the multicast flow $h$ and $r^{h,k}_{ij}\in \{0,1\}$ that indicates whether $(i,j)$ is traversed by the flow from $s_h$ to $k\in T_h$, both in the direction of $i\to j$. The latter is represented by decision variables $d_h\geq 0$ that denotes the rate of flow $h$ and $f^h_{ij}\geq 0$ that denotes the rate of flow $h$ on overlay link $(i,j)$ in the direction of $i\to j$. Define constant $b^{h,k}_i$ as $1$ if $i=s_h$, $-1$ if $i=k$, and $0$ otherwise. 
We can formulate the objective of serving all the multicast flows in $H$ \eqref{eq:demands H} within the minimum amount of time 
as the following optimization: 
\begin{subequations}\label{eq:min-time}
\begin{align}
   \min_{\bm{z},\bm{r},\bm{d},\bm{f}} &\quad \tau\\
   \mbox{s.t.} &\quad  \tau \geq {\kappa_h\over d_h}+ \sum_{(i,j)\in E}l_{i, j} r_{ij}^{h,k} ,~~\forall h\in H, k\in T_h,
   \label{min-time:time}\\
&   \sum_{(i,j)\in E: \ue\in \up_{i,j}}\sum_{h\in H} f^h_{ij}\leq C_{\ue},~~\forall \ue\in \uE,   \label{min-time:link capacity}\\   
   & \hspace{-2em}\sum_{j\in V}r^{h,k}_{ij} = \sum_{j\in V}r^{h,k}_{ji} + b^{h,k}_i, ~~ \forall h\in H, k\in T_h, i\in V, \label{min-time:flow conservation}\\
   & r^{h,k}_{ij}\leq z^h_{ij}, ~~ \forall h\in H, k\in T_h, (i,j)\in E, \label{min-time:tree} \\
   &\hspace{-1em} d_h-M(1-z^h_{ij})\leq f^h_{ij}\leq d_h,~\forall h\in H, (i,j)\in E, \label{min-time:big-M}\\
   &f_{ij}^h \le M z_{ij}^h,~\forall h\in H, (i,j)\in E,\label{noFlow:big-M}\\ 
   & \hspace{-0em} r^{h,k}_{ij}, z^h_{ij}\in \{0,1\},~d_h\in [0,M],~f^h_{ij}\geq 0,\nonumber\\
   & \hspace{8em}\forall h\in H, k\in T_h, (i,j)\in E, \label{min-time:bound}
\end{align}
\end{subequations}
where $M$ is an upper bound on $d_h$ ($\forall h\in H$). 
Constraint \eqref{min-time:time} makes $\tau$ an upper bound on the completion time of the slowest flow; \eqref{min-time:link capacity} ensures that the total traffic rate imposed by the overlay on any underlay link is within its capacity; \eqref{min-time:flow conservation}--\eqref{min-time:tree} are the \emph{Steiner arborescence} constraints \cite{Goemans93Networks} that guarantee the set of overlay links with $z^h_{ij}=1$ will form a Steiner arborescence (i.e., a directed Steiner tree) that is the union of paths from $s_h$ to each $k\in T_h$ (where each path is formed by the links with $r^{h,k}_{ij}=1$); \eqref{min-time:big-M} implies that $f^h_{ij}=d_h$ if $z^h_{ij}=1$ and \eqref{noFlow:big-M} together with \eqref{min-time:bound} implies that $f^h_{ij}= 0$ if $z^h_{ij}=0$, which allows the capacity constraint to be formulated as a linear inequality \eqref{min-time:link capacity} instead of a bilinear inequality $\sum_{(i,j)\in E: \ue\in \up_{i,j}}\sum_{h\in H} d_h z^h_{ij}\leq C_{\ue}$. 
The optimal solution $(\bm{z}^*, \bm{r}^*, \bm{d}^*, \bm{f}^*)$ to \eqref{eq:min-time} provides an overlay communication schedule that minimizes the communication time in a given iteration when the set of activated links is $E_a$.

\emph{Complexity:} 
As $|H|\leq |V|$, the optimization \eqref{eq:min-time} contains $O(|V|^2 |E|)$ variables (dominated by $\bm{r}$), and $O(|\uE| + |V|^2(|V|+|E|))$ constraints. 
Since constraints \eqref{min-time:link capacity}--\eqref{min-time:big-M} are linear and constraint \eqref{min-time:time} is convex, the optimization \eqref{eq:min-time} is a mixed integer convex programming (MICP) problem and thus can be solved by existing MICP solvers such as Pajarito~\cite{Lubin17thesis} at a super-polynomial complexity or approximate MICP algorithms such as convex relaxation plus randomized rounding at a polynomial complexity.  

\if\thisismainpaper0

\subsubsection{Important Special Case} 

In the practical application scenario where the underlay is an edge network spanning a relatively small area, the propagation delay $l_{i,j}$ will be negligible compared to the data transfer time. Moreover, since all the agents are training the same model, the sizes of the local parameter vectors will be identical without compression, i.e., $\kappa_i\equiv \kappa$ ($\forall i\in V$). Even if compression is used, since the compressed model sizes will vary dynamically across iterations, tailoring the communication schedule to the compressed model sizes will require re-solving \eqref{eq:min-time} for every iteration, which is expensive in terms of both the overhead in collecting all the compressed model sizes and the computation cost. It is thus more practical to optimize the communication schedule only once at the beginning of the learning task and use it throughout the task, and setting $\kappa_i\equiv \kappa$ (the uncompressed model size) in solving \eqref{eq:min-time} will provide a guaranteed per-iteration time, even though compression can be applied during training to further reduce the time. 
In this scenario, the completion time for a given multicast flow $h\in H$ is reduced to $\kappa/d_h$, and the overall completion time of an iteration is reduced to $\kappa/(\min_{h\in H} d_h)$. 
For this simplified objective, we observe the following property.

\begin{lemma}\label{lem:equal bandwidth allocation}
Given a feasible routing solution $\bm{z}$ to \eqref{eq:min-time}, define 
\begin{align}
t_{\ue}:= |\{h\in H, (i,j)\in E:\: z^h_{ij}=1, \ue\in \up_{ij}\}|
\end{align}
as the number of activated unicast flows\footnote{Although the logical demands is a set of multicast flows as in \eqref{eq:demands H}, the multicast operations can only be performed at overlay nodes, according to logical multicast trees formed by overlay links. Each hop in such a tree, corresponding to some $z^h_{ij}=1$, is implemented by a unicast flow through the underlay, hereafter referred to as an ``activated unicast flow''.} traversing each underlay link $\ue\in \uE$. If $l_{i,j}=0$ $\forall (i,j)\in E$ and $\kappa_h\equiv \kappa$ $\forall h\in H$, then the optimal value of \eqref{eq:min-time} under the given routing solution $\bm{z}$ is 
\begin{align}\label{eq:tau - special case, per-link}
\tau = {\kappa\over \min_{\ue\in \uE} C_{\ue}/t_{\ue}},
\end{align}
achieved by equally sharing the bandwidth at every underlay link among the activated unicast flows traversing it. 
\end{lemma}

\emph{Remark:} Lemma~\ref{lem:equal bandwidth allocation} implies that for $l_{i,j}=0$ $\forall (i,j)\in E$ and $\kappa_h\equiv \kappa$ $\forall h\in H$, solving \eqref{eq:min-time} is reduced to optimizing the routing $\bm{z}$ among the overlay nodes, after which the flow rates can be easily determined as $d_h\equiv \min_{\ue\in \uE} C_{\ue}/t_{\ue}$ $\forall h\in H$ (note that $\bm{r}, \bm{f}$ are dependent variables determined by $\bm{z}, \bm{d}$). More importantly, Lemma~\ref{lem:equal bandwidth allocation} implies that if each activated unicast flow is implemented as a TCP flow (using the same congestion control algorithm), then the minimum completion time will be automatically achieved in the steady state as long as the routing is optimal.  

\fi

\subsubsection{Handling Uncooperative Underlay}

When learning over an uncooperative underlay as considered in this work, the overlay cannot directly solve \eqref{eq:min-time}, because the capacity constraint \eqref{min-time:link capacity} requires the knowledge of the routing in the underlay and the capacities of the underlay links. In absence of such knowledge, we leverage a recent result from \cite{Huang23MobiHoc} to convert this constraint into an equivalent form that can be consistently estimated by the overlay. To this end, we introduce the following notion from \cite{Huang23MobiHoc}, adapted to our problem setting.  

\begin{definition}[\cite{Huang23MobiHoc}]\label{def: category}
A \textbf{category of underlay links $\Gamma_F$} for a set of overlay links $F$ ($F\subseteq E$) is the set of underlay links traversed \emph{by and only by} the underlay routing paths for the overlay links in $F$ out of all the paths for $E$, i.e,\footnote{Here $\up$ is interpreted as the set of underlay links traversed by path $\up$.} 
\begin{align}\label{eq:category}
\Gamma_F\coloneqq \Big(\bigcap_{(i,j)\in {F}}\up_{i,j}\Big) \setminus \Big(\bigcup_{(i,j)\in E\setminus F}\up_{i,j}\Big). 
\end{align}
\end{definition}

The key observation is that since all the underlay links in the same category are traversed by the same set of overlay links, they must carry the same traffic load from the overlay. Therefore, we can reduce the per-link capacity constraint \eqref{min-time:link capacity} into the following \emph{per-category capacity constraint}:
\begin{align}\label{eq:capacity constraint - category}
\sum_{(i,j)\in {F}}\sum_{h\in H}f^h_{ij} \leq C_F, ~~~\forall F\subseteq E \mbox{ with }\Gamma_F\neq \emptyset,    
\end{align}
where $C_F\coloneqq \min_{\ue \in \Gamma_F}C_{\ue}$, referred to as the \emph{category capacity}, is the minimum capacity of all the links in category $\Gamma_F$. 
The new constraint \eqref{eq:capacity constraint - category} is equivalent to the original constraint \eqref{min-time:link capacity}, as an overlay communication schedule satisfies one of these constraints if and only if it satisfies the other. However, instead of requiring detailed internal information about the underlay (i.e., $(\up_{i,j})_{(i,j)\in E}$ and $(C_{\ue})_{\ue\in \uE}$), constraint \eqref{eq:capacity constraint - category} only requires the knowledge of the \emph{nonempty categories} and the corresponding \emph{category capacities}. 

\if\thisismainpaper0

In the special case of $l_{i,j}=0$ $\forall (i,j)\in E$ and $\kappa_h\equiv \kappa$ $\forall h\in H$, we have a category-based counterpart of Lemma~\ref{lem:equal bandwidth allocation} as follows. 

\begin{lemma}\label{lem:equal bandwidth allocation - category}
Given a feasible routing solution $\bm{z}$ to \eqref{eq:min-time}, define
\begin{align}
t_F := |\{h\in H, (i,j)\in F:\: z^h_{ij}=1\}|
\end{align}
as the number of activated unicast flows traversing the links in category $\Gamma_F$. If $l_{i,j}=0$ $\forall (i,j)\in E$ and $\kappa_h\equiv \kappa$ $\forall h\in H$, then the optimal value of \eqref{eq:min-time} under the given routing solution $\bm{z}$ is
\begin{align}\label{eq:tau - special case, per-category}
\tau = {\kappa\over \min_{F\in \mathcal{F}} C_F/t_F},
\end{align}
achieved by equally sharing the bandwidth at every underlay link among the activated unicast flows traversing it, where $\mathcal{F}:=\{F\subseteq E: \Gamma_F\neq \emptyset\}$.
\end{lemma}

\fi

Under the assumption that every underlay link introduces a nontrivial performance impact (e.g., non-zero loss/queueing probability), \cite{Huang23MobiHoc} provided an algorithm that can consistently infer the nonempty categories from losses/delays of packets sent concurrently through the overlay links, under the assumption that concurrently sent packets will experience the same performance when traversing a shared underlay link. Moreover, by leveraging state-of-the-art single-path residual capacity estimation methods, \cite{Huang23MobiHoc} gave a simple algorithm that can accurately estimate the \emph{effective category capacity $\widetilde{C}_F$} for each detected nonempty category, that can be used in place of $C_F$ in \eqref{eq:capacity constraint - category} without changing the feasible region. Given the indices of inferred nonempty categories $\widehat{\mathcal{F}}$ and their inferred effective capacities $(\widehat{C}_F)_{F\in \widehat{\mathcal{F}}}$, we can construct the per-category capacity constraint as
\begin{align}\label{eq:capacity constraint - category inferred}
\sum_{(i,j)\in {F}}\sum_{h\in H}f^h_{ij} \leq \widehat{C}_F, ~~~\forall F\in \widehat{\mathcal{F}},    
\end{align}
which can then be used in place of \eqref{min-time:link capacity} in \eqref{eq:min-time} to compute an optimized overlay communication schedule.  
\if\thisismainpaper0
In the special case considered in Lemma~\ref{lem:equal bandwidth allocation - category}, we can estimate the minimum completion time under a given routing solution by 
\begin{align}\label{eq:tau - special case, per-category, inferred}
\widehat{\tau} := {\kappa\over \min_{F\in \widehat{\mathcal{F}}} \widehat{C}_F/t_F}. 
\end{align}
\fi

\emph{Remark:} First, the traversal of overlay paths through the overlay links is directional, and the traversal of underlay routing paths through the underlay links is also directional. Correspondingly, the overlay links in a category index $F$ should be treated as directed links (i.e., $(i,j)\in F$ only implies that the underlay links in $\Gamma_F$ are traversed by the path $\up_{i,j}$). This is not to be confused with treating the activated links in $E_a$ as undirected links, because each $(i,j)\in E_a$ stands for a parameter exchange between agents $i$ and $j$. 
Moreover, we only use the activated links $E_a$ to determine the flow demands $H$, but any overlay link in $E$ can be used in serving these flows. 


\subsubsection{Additional Optimization Opportunities and Challenges}

\begin{figure}[t!]
\vspace{-.0em}
\centerline{\mbox{\includegraphics[width=.8\linewidth]{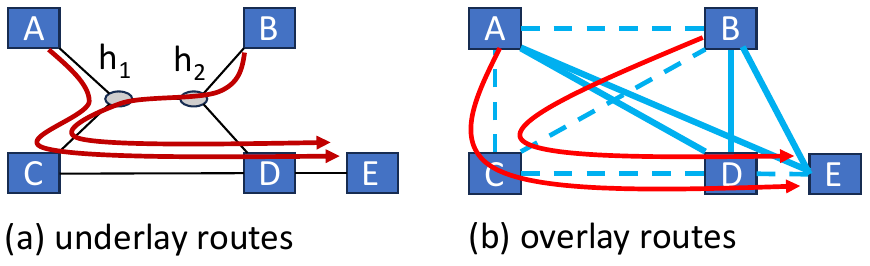}}}
\vspace{-1em}
\caption{\small Challenge for in-overlay aggregation (learning agents: $\{A,B,C,D,E\}$; underlay nodes: $\{h_1,h_2\}$). }
\label{fig:overlay_underlay_routing_aggregation}
\vspace{-.25em}
\end{figure}

The formulation \eqref{eq:min-time} treats each overlay node that is neither the source nor one of the destinations of a multicast flow as a pure relay, but this node is actually a learning agent capable of aggregating the parameter vectors. This observation raises two questions: (i) Can an agent include  parameter vectors relayed through it in its own parameter aggregation? (ii) If an agent relays multiple parameter vectors for different sources, can it forward the aggregated vector instead of the individual vectors?\looseness=-1

To answer the first question, consider the case in Fig.~\ref{fig:overlay_underlay_routing} when $A$ sends its parameter vector $\bm{x}_A$ to $D$ through the overlay path $A\to C\to D$. If $C$ includes $\bm{x}_A$ in its own parameter aggregation with a non-zero weight $W_{CA}$, then by the symmetry of the mixing matrix, $A$ must also include $\bm{x}_C$ in its parameter aggregation with weight $W_{AC}=W_{CA}$, which is equivalent to activating the overlay link $(A,C)$. As we have left the optimization of the activated links $E_a$ to another subproblem (Section~\ref{subsec:Link Activation Optimization}), there is no need to include relayed parameter vectors in parameter aggregation when optimizing the communication schedule for a given set of activated links. 

To answer the second question, consider the case in Fig.~\ref{fig:overlay_underlay_routing_aggregation} when the overlay routes the multicast from $A$ to $\{D,E\}$ (for disseminating $\bm{x}_A$) over $A\to C\to D\to E$, and the multicast from $B$ to $\{D, E\}$ (for disseminating $\bm{x}_B$) over $B\to C\to D \to E$. Although instead of separately forwarding $\bm{x}_A$ and $\bm{x}_B$, $C$ could aggregate them before forwarding, the aggregation will not save bandwidth for $C$, as $D$ needs $W_{DA}\bm{x}_A+ W_{DB}\bm{x}_B$ but $E$ needs $W_{EA}\bm{x}_A+W_{EB}\bm{x}_B$, which are generally not the same. Another issue with in-network aggregation (within the overlay) is the synchronization delay introduced at the point of aggregation, and thus in-network aggregation may not reduce the completion time even when it can save bandwidth, e.g., at $D$. We thus choose not to consider in-network aggregation in our formulation \eqref{eq:min-time}. Further optimizations exploiting such capabilities are left to future work. 

\subsection{Link Weight Optimization}\label{subsec:Conditional Link Weight Optimization}

Given the set of activated links $E_a\subseteq \widetilde{E}$, the communication time per iteration has been determined as explained in Section~\ref{subsec:Communication Schedule Optimization}, but the number of iterations has not, and is heavily affected by the weights of the activated links. This leads to the question of how to minimize the number of iterations for achieving a desired level of convergence, under the constraint that only the activated links can have non-zero weights. 

To answer this question, we leverage a state-of-the-art convergence bound for D-PSGD under the following assumptions:
\begin{enumerate}[(1)]
    \item Each local objective function $F_i(\boldsymbol{x})$ is $l$-Lipschitz smooth, i.e.,    
    $\|\nabla F_i(\bm{x})-\nabla F_i(\bm{x}')\|\leq l\|\bm{x}-\bm{x}'\|,\: \forall i\in V$.
\item There exist constants $M_1, \widehat{\sigma}$ such that    
    $\hspace{-0em} {1\over m}\sum_{i\in V}\E[\|g(\bm{x}_i;\xi_i)-\nabla F_i(\bm{x}_i) \|^2] \leq \widehat{\sigma}^2 + {M_1\over m}\sum_{i\in V}\|\nabla F(\bm{x}_i)\|^2$, 
    $\forall \bm{x}_1,\ldots,\bm{x}_m \in \mathbb{R}^d$.    
\item There exist constants $M_2, \widehat{\zeta}$ such that 
    ${1\over m}\sum_{i\in V}\|\nabla F_i(\bm{x})\|^2\leq \widehat{\zeta}^2 + M_2\|\nabla F(\bm{x}) \|^2, \forall \bm{x} \in \mathbb{R}^d$.    
\end{enumerate}
Let $\bm{J}:={1\over m}\bm{1} \bm{1}^\top$ denote an ideal $m\times m$ mixing matrix with all entries being ${1\over m}$.

\begin{theorem}\cite[Theorem~2]{Koloskova20ICML} \label{thm:new convergence bound}
Under assumptions (1)--(3), if there exist constants $p\in (0,1]$ and integer $t\geq 1$ such that the mixing matrices $\{\bm{W}^{(k)}\}_{k=1}^K$, each being symmetric with each row/column summing to one\footnote{Originally, \cite[Theorem~2]{Koloskova20ICML} had a stronger assumption that each mixing matrix is doubly stochastic, but we have verified that it suffices to have each row/column summing to one.}, satisfy
\begin{align}\label{eq:condition on p}
\hspace{-.5em}    \E[\|\bm{X} \hspace{-.25em}\prod_{k=k't+1}^{(k'+1)t}\hspace{-.25em}\bm{W}^{(k)} - \bm{X}\bm{J}\|_F^2] \leq (1-p) \|\bm{X}-\bm{X}\bm{J}\|_F^2
\end{align}
for all $\bm{X}:= [\bm{x}_1,\ldots,\bm{x}_m]$ and integer $k'\geq 0$, then D-PSGD can achieve $ \frac{1}{K} \sum_{k=1}^K \mathbbm{E}[\|\nabla F(\boldsymbol{\overline{x}}^k)\|^2]\leq \epsilon_0$ for any given $\epsilon_0>0$ ($\overline{\bm{x}}^{(k)}:={1\over m}\sum_{i=1}^m \bm{x}^{(k)}_i$) when the number of iterations reaches
\begin{align}
K(p,t)&:= l(F(\overline{\bm{x}}^{(1)})-F_{\inf}) \nonumber\\
&\hspace{-2.75em}   \cdot \hspace{-.15em}O\hspace{-.25em}\left(\hspace{-.25em}{\widehat{\sigma}^2\over m\epsilon_0^2} \hspace{-.15em}+\hspace{-.15em}{\widehat{\zeta}t\sqrt{M_1+1}+\widehat{\sigma}\sqrt{pt}\over p \epsilon_0^{3/2}} \hspace{-.15em}+\hspace{-.15em} {t\sqrt{(M_2+1)(M_1+1)}\over p\epsilon_0} \hspace{-.05em}\right), \label{eq:new bound on K}
\end{align}
where $\overline{\bm{x}}^{(1)}$ is the initial parameter vector, and $F_{\inf}$ is a lower bound on $F(\cdot)$. 
\end{theorem}

\if\thisismainpaper0
\emph{Remark:} 
While there exist other convergence bounds for D-PSGD such as \cite{MATCHA22,vogels2022beyond,neglia2020decentralized,le2023refined}, we choose Theorem~\ref{thm:new convergence bound} as the theoretical foundation of our design due to the generality of its assumptions. For example, 
assumption~(2) generalizes the assumption of uniformly bounded variance of stochastic gradients in \cite{MATCHA22,le2023refined}, assumption~(3) generalizes the assumption of bounded data heterogeneity in \cite{MATCHA22,vogels2022beyond,neglia2020decentralized}, and assumptions~(2-3) are easily implied by the bounded gradient assumption in \cite{neglia2020decentralized} (see explanations in \cite{Koloskova20ICML}).  
\fi

For tractability, we will focus on the case of i.i.d. mixing matrices. In this case, to achieve $\epsilon_0$-convergence, it suffices for the number of iterations to reach $K(p,t)$ as in \eqref{eq:new bound on K} for $t = 1$. We note that $K(p,1)$ depends on the mixing matrix only through the parameter $p$: the larger $p$, the smaller $K(p,1)$. 
Recall that as explained in Section~\ref{subsec:Model of Decentralized Learning}, the mixing matrix $\bm{W}$ is related to the link weights $\bm{\alpha}$ as $\bm{W} = \bm{I}-\bm{B}\diag(\bm{\alpha})\bm{B}^\top$. To restrict the activated links to $E_a$, we set $\alpha_{ij}=0$ for all $(i,j)\not\in E_a$. Below, we will show that the following optimization gives a good design of the link weights: \looseness=0
\begin{subequations}\label{eq:min rho wo cost}
\begin{align}
& \min_{\bm{\alpha}}\:  \rho \label{wo cost:obj}\\
\mbox{s.t. }
& - \rho\bm{I} \preceq \bm{I} - \bm{B} \diag(\bm{\alpha}) \bm{B}^\top - \bm{J} \preceq  \rho\bm{I}, \label{wo cost:matrix}\\
& \alpha_{ij} = 0,~~\forall (i,j)\not\in E_a. \label{wo cost:activation}
\end{align}
\end{subequations}

\begin{corollary}\label{cor:optimal link weights}
Under assumptions (1)--(3) and i.i.d. mixing matrices $\bm{W}^{(k)} \stackrel{d}{=} \bm{W}$ that is symmetric with each row/column summing to one, D-PSGD achieves $\epsilon_0$-convergence as in Theorem~\ref{thm:new convergence bound} when the number of iterations reaches
\begin{align}
K(1-\E[\|\bm{W}-\bm{J}\|^2],1). \label{eq:relaxed bound on K}
\end{align}
Moreover, conditioned on the set of activated links being $E_a$, \eqref{eq:relaxed bound on K} $\geq K(1- \mathop{\rho^*}^2,1)$, where $\mathop{\rho^*}$ is the optimal value of \eqref{eq:min rho wo cost}, with ``$=$'' achieved at $\bm{W}^* = \bm{I}-\bm{B}\diag(\bm{\alpha}^*)\bm{B}^\top$ for $\bm{\alpha}^*$ being the optimal solution to \eqref{eq:min rho wo cost}. 
\end{corollary}

Corollary~\ref{cor:optimal link weights} implies that given the set of activated links, we can design the corresponding link weights by solving \eqref{eq:min rho wo cost}, which will minimize an upper bound \eqref{eq:relaxed bound on K} on the number of iterations to achieve $\epsilon_0$-convergence. 
Optimization \eqref{eq:min rho wo cost} is a semi-definite programming (SDP) problem that can be solved in polynomial time by existing algorithms \cite{Jiang20FOCS}. 

\emph{Remark:}
When $\bm{W}$ satisfies the additional property of $\bm{I} \succeq \bm{W} \succeq -\bm{I}$, the largest singular value of $\bm{W}$ is $1$ \cite{hua2022efficient}, and thus $\| \bm{W}-\bm{J}\|$ is the second largest singular value of $\bm{W}$. In this case, minimizing $\rho$ in \eqref{eq:min rho wo cost} (which equals $\|\bm{W}-\bm{J}\|$ under the optimal solution) is equivalent to maximizing $\gamma(\bm{W}):= 1-\|\bm{W}-\bm{J}\|$, which is the \emph{spectral gap} of the mixing matrix $\bm{W}$~\cite{neglia2020decentralized}. The spectral gap is the most widely-used parameter to capture the impact of the mixing matrix on the convergence rate 
\if\thisismainpaper1
\cite{Lian17NIPS,neglia2020decentralized,jiang2023joint}.
\else
\cite{Lian17NIPS,Nedic18IEEE,Neglia19INFOCOM,neglia2020decentralized,jiang2023joint}. 
\fi
In this sense, our Corollary~\ref{cor:optimal link weights} extends the relationship between the spectral gap and the number of iterations to the case of random mixing matrices. As $\gamma(\bm{W})\to 0$ (in probability), the number of iterations according to \eqref{eq:relaxed bound on K} grows at
\begin{align}
K\left(1-\E[(1-\gamma(\bm{W}))^2],1\right) 
&=O\left({1\over \E[\gamma(\bm{W})]}\right),
\end{align}
which is consistent with the existing result of $O(1/\gamma(\bm{W}))$ in the case of deterministic mixing matrix~\cite{neglia2020decentralized}. 
While other parameters affecting the convergence rate have been identified, e.g., the effective number of neighbors~\cite{vogels2022beyond} and the neighborhood heterogeneity~\cite{le2023refined}, these parameters are just additional factors instead of replacements of the spectral gap. We thus leave the optimization of these other objectives to future work.

\subsection{Link Activation Optimization}\label{subsec:Link Activation Optimization}

Given how to optimize the communication schedule and the link weights for a given set $E_a$ of activated links as in Sections~\ref{subsec:Communication Schedule Optimization}--\ref{subsec:Conditional Link Weight Optimization}, what remains is to optimize $E_a$ itself, which is also known as ``topology design'' \cite{marfoq2020throughput} as $(V,E_a)$ depicts a subgraph of the base topology (of the overlay) that is used to determine which agents will exchange parameter vectors during DFL. 
The set $E_a$ affects both the communication demands (and hence the time per iteration) and the sparsity pattern of the mixing matrix (and hence the number of iterations required). As mentioned in Section~\ref{subsec:Goal: Network-aware Communication Optimization for Decentralized Learning}, our goal is to minimize the total training time. 
For learning over bandwidth-limited networks, the training time is dominated by the communication time~\cite{Luo20MLsys}. We thus model the total training time by \looseness=0
\begin{align}\label{eq:total time}
\tau(E_a) \cdot K(E_a),
\end{align}
where we have used $\tau(E_a)$ to denote the communication time per iteration according to \eqref{eq:min-time} (with \eqref{min-time:link capacity} replaced by \eqref{eq:capacity constraint - category inferred}), and $K(E_a):=K(1-\mathop{\rho^*}^2, 1)$ to denote the number of iterations to achieve a given level of convergence. 
Our goal is to minimize \eqref{eq:total time} over all the candidate values of $E_a\subseteq \widetilde{E}$. 


Directly solving this optimization is intractable 
because its solution space is exponentially large and its objective function \eqref{eq:total time} is not given explicitly. Our approach to address this challenge is to: (i) relax $\tau(E_a)$ and $K(E_a)$ into upper bounds that are explicit functions of $E_a$, (ii) decompose the optimization to separate the impacts of $\tau(E_a)$ and $K(E_a)$, and (iii) develop efficient solutions by identifying linkages to known problems. \looseness=-1


\subsubsection{Relaxed Objective Function}\label{subsubsec:Relaxed Objective Function}

We first upper-bound $\tau(E_a)$ by considering a suboptimal but analyzable communication schedule. Consider a solution to \eqref{eq:min-time} with $z^h_{ij}=1$ if $i=s_h, j\in T_h$ and $0$ otherwise, and $r^{h,k}_{ij}=1$ if $i=s_h, j=k$ and $0$ otherwise, i.e., each parameter exchange corresponding to $(i,j)\in E_a$ is performed directly along the underlay routing paths $\up_{i,j}$ and $\up_{j,i}$. To achieve a per-iteration communication time of $\overline{\tau}$, the rate $d_{i,j}$ of sending the parameter vector of agent $i$ to its activated neighbor $j$  must satisfy
$
d_{i,j} \geq {\kappa_i\over \overline{\tau}-l_{i,j}}.
$
This is feasible for \eqref{eq:min-time} (with \eqref{min-time:link capacity} replaced by \eqref{eq:capacity constraint - category inferred}) if and only if \looseness=-1
\begin{align}\label{eq:constraint on E_a}
\sum_{(i,j)\in E_a} \left({\kappa_i\over \overline{\tau}-l_{i,j}}\mathbbm{1}_{(i,j)\in F} + {\kappa_j\over \overline{\tau}-l_{j,i}}\mathbbm{1}_{(j,i)\in F}\right) \leq \widehat{C}_F,~~\forall F\in \widehat{\mathcal{F}}.
\end{align}
The minimum value of $\overline{\tau}$ satisfying \eqref{eq:constraint on E_a}, denoted by $\overline{\tau}(E_a)$, thus provides an upper bound on the minimum per-iteration time $\tau(E_a)$ under the set of activated links in $E_a$. 


We then upper-bound $K(E_a)$ by upper-bounding the optimal value $\rho^*$ of \eqref{eq:min rho wo cost}. Consider any predetermined link weight assignment $\bm{\alpha}^{(0)}$, and let $\bm{\alpha}^{(0)}(E_a)$ be the corresponding feasible solution to \eqref{eq:min rho wo cost} (i.e., $(\alpha^{(0)}(E_a))_{ij} = \alpha^{(0)}_{ij}$ if $(i,j)\in E_a$ and $0$ otherwise). 
Let $\bm{L}(E_a):= \bm{B}\diag(\bm{\alpha}^{(0)}(E_a))\bm{B}^\top$ denote the Laplacian matrix for the activated graph. Under this solution, the objective value of \eqref{eq:min rho wo cost} is 
\begin{align}
\overline{\rho} &:= \|\bm{I}-\bm{L}(E_a)-\bm{J} \| \label{eq:bound rho - 1} \\
&= \max(1-\lambda_2(\bm{L}(E_a)),\: \lambda_m(\bm{L}(E_a))-1 ), \label{eq:bound rho - 2}
\end{align}
where \eqref{eq:bound rho - 1} is from the proof of Corollary~\ref{cor:optimal link weights}, and \eqref{eq:bound rho - 2} is by \cite[Lemma~IV.2]{Chiu23JSAC} (where $\lambda_i(\bm{L}(E_a))$ denotes the $i$-th smallest eigenvalue of $\bm{L}(E_a)$). 
Since $K(1-\mathop{\rho}^2, 1)$ is an increasing function of $\rho$ and $\rho^*\leq \overline{\rho}$, we have
\begin{align}
K(E_a) := K\left(1-{\rho^*}^2, 1\right) \leq K\left(1-\overline{\rho}^2, 1\right) =: \overline{K}(E_a). \label{eq:K_bound}
\end{align}

\subsubsection{Bi-level Decomposition}

Relaxing \eqref{eq:total time} into its upper bound $\overline{\tau}(E_a) \cdot \overline{K}(E_a)$ provides an objective function that can be easily evaluated for any candidate $E_a$. However, we still face the exponentially large solution space of $E_a\subseteq \widetilde{E}$. 
To address this complexity challenge, we decompose the relaxed optimization into a bi-level optimization as follows. 

\begin{lemma}\label{lem:equivalence to bilevel}
Let $\beta$ be the maximum time per iteration. Then \looseness=-1
\begin{align}
\min_{E_a\subseteq \widetilde{E}} \overline{\tau}(E_a) \cdot \overline{K}(E_a) = \min_{\beta\geq 0} \beta \cdot \left(\min_{\overline{\tau}(E_a)\leq \beta} \overline{K}(E_a) \right), \label{eq:equivalance to bilevel}
\end{align}
and the optimal solution $E_a^*$ to the RHS of \eqref{eq:equivalance to bilevel}  is also optimal for the LHS of \eqref{eq:equivalance to bilevel}. 
\end{lemma}

The bi-level decomposition in \eqref{eq:equivalance to bilevel} allows us to focus on the lower-level optimization 
\begin{align}
\min_{\overline{\tau}(E_a)\leq \beta} \overline{K}(E_a), \label{eq:lower-level optimization}
\end{align}
as the upper-level optimization only has a scalar variable $\beta$ that can be optimized numerically once we have an efficient solution to \eqref{eq:lower-level optimization}. \looseness=-1

\subsubsection{Algorithms}

To solve \eqref{eq:lower-level optimization}, we encode $E_a$ by binary variables $\bm{y}:=(y_{ij})_{(i,j)\in \widetilde{E}}$, where $y_{ij}=1$ if $(i,j)\in E_a$ and $0$ otherwise. By \eqref{eq:constraint on E_a}, $\bm{y}$ is feasible for \eqref{eq:lower-level optimization} if and only if $\forall F\in \widehat{\mathcal{F}}$,
\begin{align}
\hspace{-.75em}\sum_{(i,j)\in \widetilde{E}} y_{ij} \left({\kappa_i\over \beta-l_{i,j}}\mathbbm{1}_{(i,j)\in F} + {\kappa_j\over \beta-l_{j,i}}\mathbbm{1}_{(j,i)\in F} \right) \leq \widehat{C}_F, 
\label{eq:independence constraint}
\end{align}
where $\mathbbm{1}_{\cdot}$ denotes the indicator function. 

\textbf{Exact solution:} We can directly try to solve the lower-level optimization \eqref{eq:lower-level optimization} as a convex programming problem with integer variables. Due to the monotone relationship between  $\overline{K}(E_a)$ and $\overline{\rho}$, we can equivalently minimize $\overline{\rho}$ as defined in \eqref{eq:bound rho - 1} by solving
\begin{subequations}\label{eq:lower-level ICP}
    \begin{align}
\min_{\bm{y}} &\quad \rho \label{lower-ICP:obj} \\
\mbox{s.t. } &  - \rho\bm{I} \preceq \bm{I} - \sum_{(i,j)\in \widetilde{E}}y_{ij} \bm{L}_{ij} -\bm{J} \preceq  \rho\bm{I}, \label{lower-ICP:linear matrix inequality} \\
& \eqref{eq:independence constraint},~~\forall F\in \widehat{\mathcal{F}}, \label{lower-ICP:capacity} \\
& y_{ij}\in \{0,1\},~~\forall (i,j)\in \widetilde{E}, \label{lower-ICP:integer}
    \end{align}
\end{subequations}
where $\bm{L}_{ij}$ is the Laplacian matrix representation of link $(i,j)$ with weight $\alpha^{(0)}_{ij}$, i.e., entries $(i,j)$ and $(j,i)$ are $-\alpha^{(0)}_{ij}$ and entries $(i,i)$ and $(j,j)$ are $\alpha^{(0)}_{ij}$ (rest are zero), 
and \eqref{lower-ICP:linear matrix inequality} ensures that the auxiliary variable $\rho = \overline{\rho}$ as  in \eqref{eq:bound rho - 1} under the optimal solution. 

The optimization \eqref{eq:lower-level ICP} is an integer convex programming (ICP) problem, and thus in theory can be solved by ICP solvers such as \cite{Lubin17thesis}. In practice, however, ICP solvers can be  slow due to their super-polynomial complexity, and in contrast to the communication schedule optimization \eqref{eq:min-time} that only needs to be solved once, \eqref{eq:lower-level ICP} has to be solved many times to optimize the variable $\beta$ in the upper-level optimization. Thus, we need more efficient algorithms. \looseness=-1

\textbf{Efficient heuristics:} 
To improve the computational efficiency, we have explored two approaches: (i) developing heuristics for \eqref{eq:lower-level ICP}, and (ii) further simplifying the objective function.

\begin{algorithm}[tb]
\small
\SetKwInOut{Input}{input}\SetKwInOut{Output}{output}
\Input{Initial link weights $\valpha^{(0)}$, candidate links $\widetilde{E}$, threshold $\epsilon$. }
\Output{Set of activated links $E_a$.}
initialize $E_s \leftarrow \emptyset$ and $E_o \leftarrow \emptyset$\;
\While{True}
{
obtain $\vy^*$ by solving the SDP relaxation of \eqref{eq:lower-level ICP} with additional constraints that $y_{\tilde{e}} = 0, \forall \tilde{e}\in E_o$ and $y_{\tilde{e}} = 1, \forall \tilde{e}\in E_s$\; \label{SCA:3}
\If{\eqref{lower-ICP:capacity} is satisfied by $\vy$ such that $y_{\tilde{e}}=1$ iff $y^*_{\tilde{e}} \geq \epsilon$ \label{SCA:4}}
{Break\;}
\Else
{
Find $\tilde{e}_s = \arg\max\{ y^*_{\tilde{e}}: \tilde{e}\in \widetilde{E}\setminus(E_s\cup E_o), \eqref{lower-ICP:capacity} \text{ is satisfied by } $\vy$ \text{ corresponding to } E_s \cup \{\tilde{e}\}\}$\; \label{SCA:7}
$E_s \leftarrow E_s \cup \{\tilde{e}_s\}$\; \label{SCA:8}
Find $\tilde{e}_o = \arg\min\{ y^*_{\tilde{e}}: \tilde{e}\in \widetilde{E}\setminus(E_s\cup E_o)\}$\; \label{SCA:9}
$E_o \leftarrow E_o \cup \{\tilde{e}_o\}$\; \label{SCA:10}
}
}
return $E_a \leftarrow \{\tilde{e}\in \widetilde{E}:\: y^*_{\tilde{e}} \geq \epsilon\}$\; 
\caption{Topology Design via SCA}
\vspace{-.0em}
\label{Alg:SCA_rouding_min_rho}
\end{algorithm}
\normalsize

$\bullet$ \emph{Heuristics for \eqref{eq:lower-level ICP}:} Once we relax the integer constraint \eqref{lower-ICP:integer} into $y_{ij}\in [0,1]$, \eqref{eq:lower-level ICP} becomes an SDP that can be solved in polynomial time \cite{Jiang20FOCS}. We can thus round the fractional solution into a feasible solution. However, we observe that simple rounding schemes such as greedily activating links with the largest fractional $y$-values do not yield a good solution (see results for `Relaxation-$\rho$' in Section~\ref{sec:Performance Evaluation}). Thus, we propose an iterative rounding algorithm based on successive convex approximation (SCA) as in Algorithm~\ref{Alg:SCA_rouding_min_rho}. The algorithm gradually rounds the $y$-values for a subset of links $E_s$ to $1$ and those for another subset of links $E_o$ to $0$ to satisfy feasibility. In each iteration, it solves the SDP relaxation of \eqref{eq:lower-level ICP} with rounding constraints according to the previously computed $E_s$ and $E_o$ (line~\ref{SCA:3}), and then adds the link with the largest fractional $y$-value to $E_s$ (lines~\ref{SCA:7}--\ref{SCA:8}) and the link with the smallest fractional $y$-value to $E_o$ (lines~\ref{SCA:9}--\ref{SCA:10}). The iteration repeats until the integer solution rounded according to a given threshold $\epsilon$ is feasible for \eqref{eq:lower-level ICP} (line~\ref{SCA:4}). 

$\bullet$ \emph{Heuristics based on algebraic connectivity:} Another approach is to simplify the objective based on the following observation. 

\begin{lemma}\label{lem:reduction to algebraic connectivity}
Minimizing $\overline{\rho}$ in \eqref{eq:bound rho - 2} is equivalent to maximizing $\lambda_2(\bm{L}(E_a))$ (the second smallest eigenvalue) if  $\bm{\alpha}^{(0)}$ satisfies \looseness=0
\begin{align}
\max_{i\in V}\sum_{j: (i,j)\in \widetilde{E}}\alpha^{(0)}_{ij} + m\cdot \max_{(i,j)\in \widetilde{E}}|\alpha^{(0)}_{ij}|\leq 1. \label{eq:requirement on alpha^0}
\end{align}
\end{lemma}

\emph{Remark:} We can satisfy \eqref{eq:requirement on alpha^0} by making $\alpha^{(0)}_{ij}$'s sufficiently small, e.g., $\alpha^{(0)}_{ij}\equiv 1/(2m-1)$, $\forall (i,j)\in \widetilde{E}$. 

Under condition \eqref{eq:requirement on alpha^0}, we can convert the minimization of $\overline{\rho}$ into a maximization of $\lambda_2(\bm{L}(E_a))$, known as the \emph{algebraic connectivity}:
\begin{subequations}\label{eq:algebraic connectivity maximization}
\begin{align}
\max_{\bm{y}} &\quad \lambda_2(\sum_{(i,j)\in \widetilde{E}}y_{ij} \bm{L}_{ij}) \label{algebraic:obj} \\
\mbox{s.t. } & \eqref{eq:independence constraint},~~\forall F\in \widehat{\mathcal{F}}, \label{algebraic:capacity} \\
& y_{ij}\in \{0,1\},~~\forall (i,j)\in \widetilde{E}, \label{algebraic:integer}
\end{align}
\end{subequations}
which selects links to maximize the {algebraic connectivity} under the linear constraints \eqref{algebraic:capacity}. 
A similar problem of maximizing the algebraic connectivity of \emph{unweighted} graphs under \emph{cardinality constraint} (i.e., $\sum_{(i,j)\in \widetilde{E}}y_{ij}\leq k$) has been studied with some efficient heuristics \cite{Ghosh06CDC,He19arXiv}. Although our problem \eqref{eq:algebraic connectivity maximization} addresses a weighted graph and more general linear constraints, the existing heuristics can still be adapted for our problem. 

Specifically, as $\lambda_2(\sum_{(i,j)\in \widetilde{E}} y_{ij}\bm{L}_{ij})$ is a concave function of $\bm{y}$~\cite{Ghosh06CDC}, relaxing the integer constraint \eqref{algebraic:integer} into $y_{ij}\in [0,1]$ turns \eqref{eq:algebraic connectivity maximization} into a convex optimization that can be solved in polynomial time, based on which we can extract an integer solution via rounding. 
We can also adapt the greedy perturbation heuristic in \cite{Ghosh06CDC} as follows. Let $\bm{v}(E_a)$ denote the \emph{Fiedler vector} of $\bm{L}(E_a)$ (i.e., the unit-norm eigenvector corresponding to $\lambda_2(\bm{L}(E_a))$). It is easy to extend \cite[(10)]{Ghosh06CDC} into 
\begin{align}
\hspace{-.75em}\lambda_2(\bm{L}(E_a\cup \{(i,j)\})) - \lambda_2(\bm{L}(E_a)) \leq \alpha^{(0)}_{ij}\left(v(E_a)_i - v(E_a)_j\right)^2.
\end{align}
Based on this bound, we can apply the greedy heuristic to \eqref{eq:algebraic connectivity maximization} by repeatedly: (i) computing the Fiedler vector $\bm{v}(E_a)$ based on the current $E_a$, and (ii) augmenting $E_a$ with the link $(i,j)\in \widetilde{E}\setminus E_a$ that maximizes $\alpha^{(0)}_{ij}\left(v(E_a)_i - v(E_a)_j\right)^2$ subject to \eqref{eq:independence constraint}. 

\emph{Remark:} 
Even if we can enforce the condition in Lemma~\ref{lem:reduction to algebraic connectivity} by suitably setting the initial link weights $\valpha^{(0)}$, the final link weights are determined by the optimization \eqref{eq:min rho wo cost} and thus cannot guarantee the equivalence between $\overline{\rho}$ minimization and algebraic connectivity maximization. As a result, our evaluation shows that $E_a$'s designed by the above heuristics based on the algebraic connectivity are less effective than those designed based on $\overline{\rho}$ (see Section~\ref{sec:Performance Evaluation}). However, connecting our problem with algebraic connectivity maximization opens the possibility of leveraging a rich set of existing results, for which the above is just a starting point. In this regard, our contribution is to formally establish this connection and the corresponding condition.\looseness=-1 

\subsection{Overall Solution}

\begin{figure}[t!]
\vspace{-0em}
\centerline{\mbox{\includegraphics[width=1\linewidth]{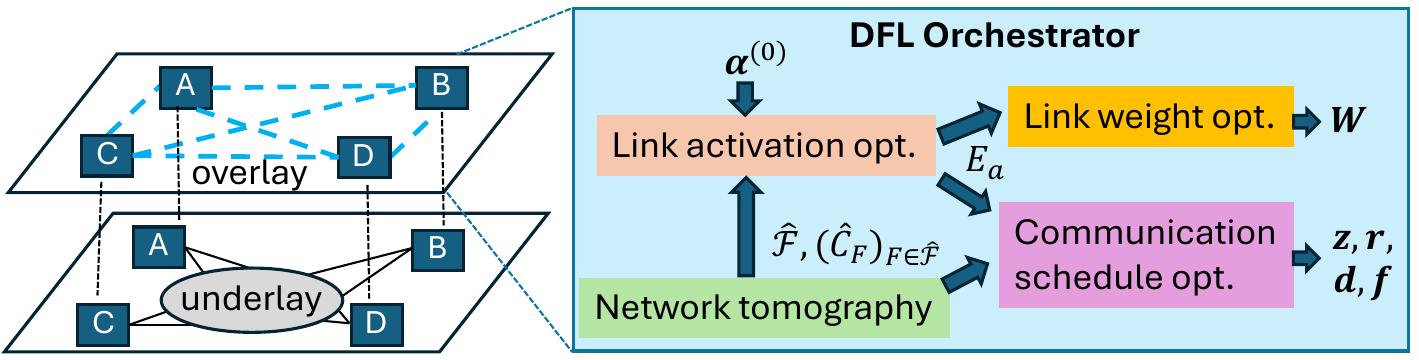}}}
\vspace{-1em}
\caption{\small \rev{Workflow of overall solution.} }
\label{fig:overall}
\vspace{-.25em}
\end{figure}

Fig.~\ref{fig:overall} illustrates the overall proposed solution \rev{as deployed on a centralized orchestrator when initializing DFL tasks}, which starts by inferring the necessary information about the underlay using network tomography \cite{Huang23MobiHoc}, and then selects the links to activate based on predetermined weights $\bm{\alpha}^{(0)}$ as in Section~\ref{subsec:Link Activation Optimization}, based on which the link weights are optimized as in Section~\ref{subsec:Conditional Link Weight Optimization} and the communication schedule is optimized as in Section~\ref{subsec:Communication Schedule Optimization}. 
\rev{Our solution is suitable for the centralized deployment as it only uses predetermined information.}\looseness=-1 

The performance of this solution is guaranteed as follows.

\begin{theorem}\label{thm:overall solution}
Under the assumption of $\widehat{\mathcal{F}}\supseteq \mathcal{F}$ and $\widehat{C}_F \leq C_F$ ($\forall F\in \mathcal{F}$), if the proposed solution activates a set of links $E_a^*$, then D-PSGD under the corresponding design is guaranteed to achieve $\epsilon_0$-convergence as defined in Theorem~\ref{thm:new convergence bound} within time $\overline{\tau}(E_a^*) \cdot \overline{K}(E_a^*)$ for $\overline{\tau}(\cdot)$ and $\overline{K}(\cdot)$ defined in Section~\ref{subsubsec:Relaxed Objective Function}. 
\end{theorem}

\emph{Remark:} Theorem~\ref{thm:overall solution} upper-bounds the total training time under our design if network tomography detects all the nonempty categories and does not overestimate the category capacities. According to \cite{Huang23MobiHoc}, the second assumption holds approximately as the estimation of category capacities is highly accurate, but the first assumption may not hold due to misses in detecting nonempty categories. However, in underlays following symmetric tree-based routing as commonly encountered in edge networks, a new algorithm in \cite{Huang24TONsub} can detect nearly all the nonempty categories in networks of moderate sizes. 
We will empirically evaluate the impact of inference errors on the final performance of DFL in Section~\ref{sec:Performance Evaluation}. 

\section{Performance Evaluation}\label{sec:Performance Evaluation}

We evaluate the proposed algorithms against benchmarks through realistic data-driven simulations in the context of bandwidth-limited wireless edge networks. \looseness=0

\subsection{Simulation Setup}

\subsubsection{Dataset and ML Model}

We train a ResNet-50 model with 23,616,394 parameters and a model size of 90.09 MB for image classification on the CIFAR-10 dataset, which consists of 60,000 color images divided into 10 classes. We use 50,000 images for training and the remaining 10,000 images for testing. The dataset undergoes standard preprocessing, including normalization and one-hot encoding of the labels.
We set the learning rate to $0.02$ and the mini-batch size to 64 for each agent. These settings are sufficient for D-PSGD to achieve convergence under all evaluated designs. As a sanity check, we also train a 4-layer CNN model based on \cite{McMahan17AISTATS}, which has 582,026 parameters and a model size of 2.22 MB, for digit recognition on the MNIST dataset, which comprises 60,000 training images and 10,000 testing images. 
\rev{In both cases, we evenly divide the training data among all the agents after a random shuffle.}

\subsubsection{Network Topology}

We simulate the underlay based on the topologies and link attributes of real wireless edge networks. We consider two important types of networks: (i) WiFi-based wireless mesh networks represented by the Roofnet~\cite{Roofnet}, which has 33 nodes, 187 links, and a data rate of 1 Mbps, and (ii) millimeter-wave-based Integrated Access and Backhaul (IAB) networks used to extend the coverage of 5G networks~\cite{3GPPTR38874}, represented by a hexagon topology with 19 nodes, 56 links, and a data rate of $0.4$ Gbps \cite{IAB20Ericsson}. In each topology, we select $10$ low-degree nodes as learning agents (i.e., overlay nodes). We assume the base topology to be a clique among the agents (i.e., any two agents are allowed to communicate), and use the shortest paths (based on hop count) between the agents as the underlay routing paths. 
\if\thisismainpaper1
See \cite{Huang24:report} for the simulated topologies. 
\else 
The simulated topologies are illustrated in Fig.~\ref{fig:underlay_topo}. 

\begin{figure}[t!]
\vspace{-0em}
\centerline{\mbox{\includegraphics[width=1\linewidth]{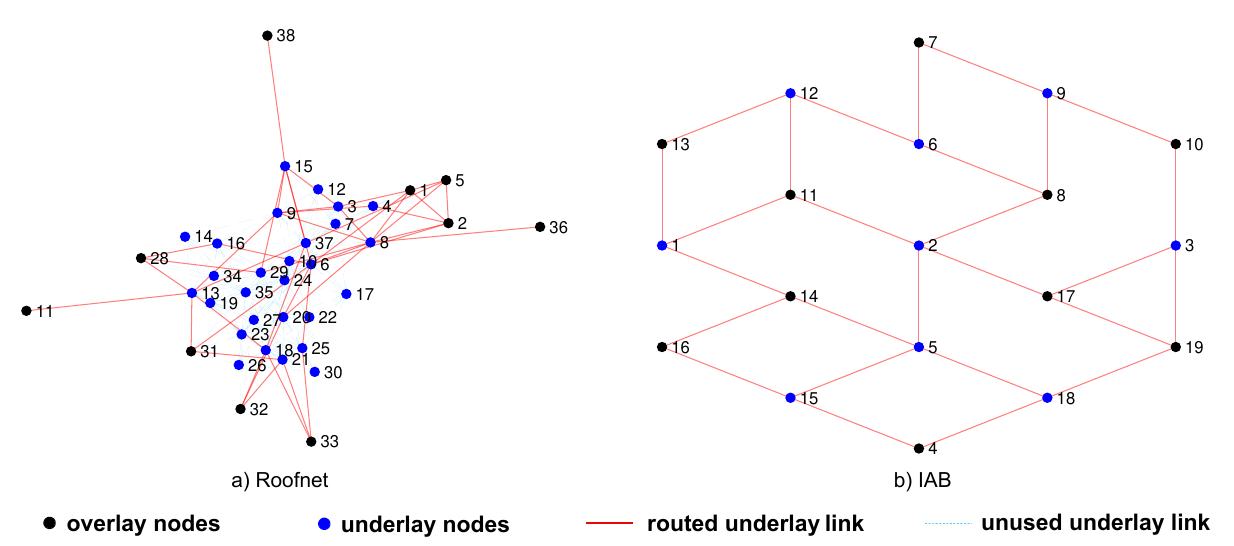}}}
\vspace{-1.25em}
\caption{\small Underlay network topologies. }
\label{fig:underlay_topo}
\vspace{-.05em}
\end{figure}

\fi


\subsubsection{Benchmarks}

We compare the proposed  Algorithm~\ref{Alg:SCA_rouding_min_rho} (`SCA') against the following benchmarks: \begin{itemize}
    \item the baseline of activating all the (overlay) links (`Clique');
    \item the ring topology (`Ring') commonly adopted by industry; 
    \item the minimum spanning tree computed by Prim's algorithm (`Prim'),  proposed by \cite{marfoq2020throughput} as the state of the art for overlay-based DFL; 
    \item the simplistic heuristic for \eqref{eq:lower-level ICP} based on SDP relaxation plus greedy rounding (`Relaxation-$\rho$');
    \item the heuristics for \eqref{eq:algebraic connectivity maximization} based on convex relaxation plus rounding (`Relaxation-$\lambda$') or greedy perturbation \cite{Ghosh06CDC} (`Greedy').
\end{itemize}
We will first evaluate a basic setting where we provide accurate information about the underlay to all the algorithms, use \eqref{eq:min rho wo cost} to optimize the link weights under each topology design, and let all the communications occur directly over the underlay routing paths (i.e., without overlay routing). We will separately evaluate the impact of overlay routing, inference errors, \rev{and weight design}. \looseness=0

\if\thisismainpaper1
\subsection{Simulation Results}
\else
\subsection{Simulation Results on Roofnet}
\fi 

\if\thisismainpaper1
Due to space limitation, we will only present the results based on CIFAR-10 and Roofnet, and defer the other results to \cite{Huang24:report}. 
\else
\fi

\subsubsection{Results without Overlay Routing}

\begin{figure}[t!]
\begin{minipage}{.495\linewidth}
\centerline{
\includegraphics[width=1\linewidth]{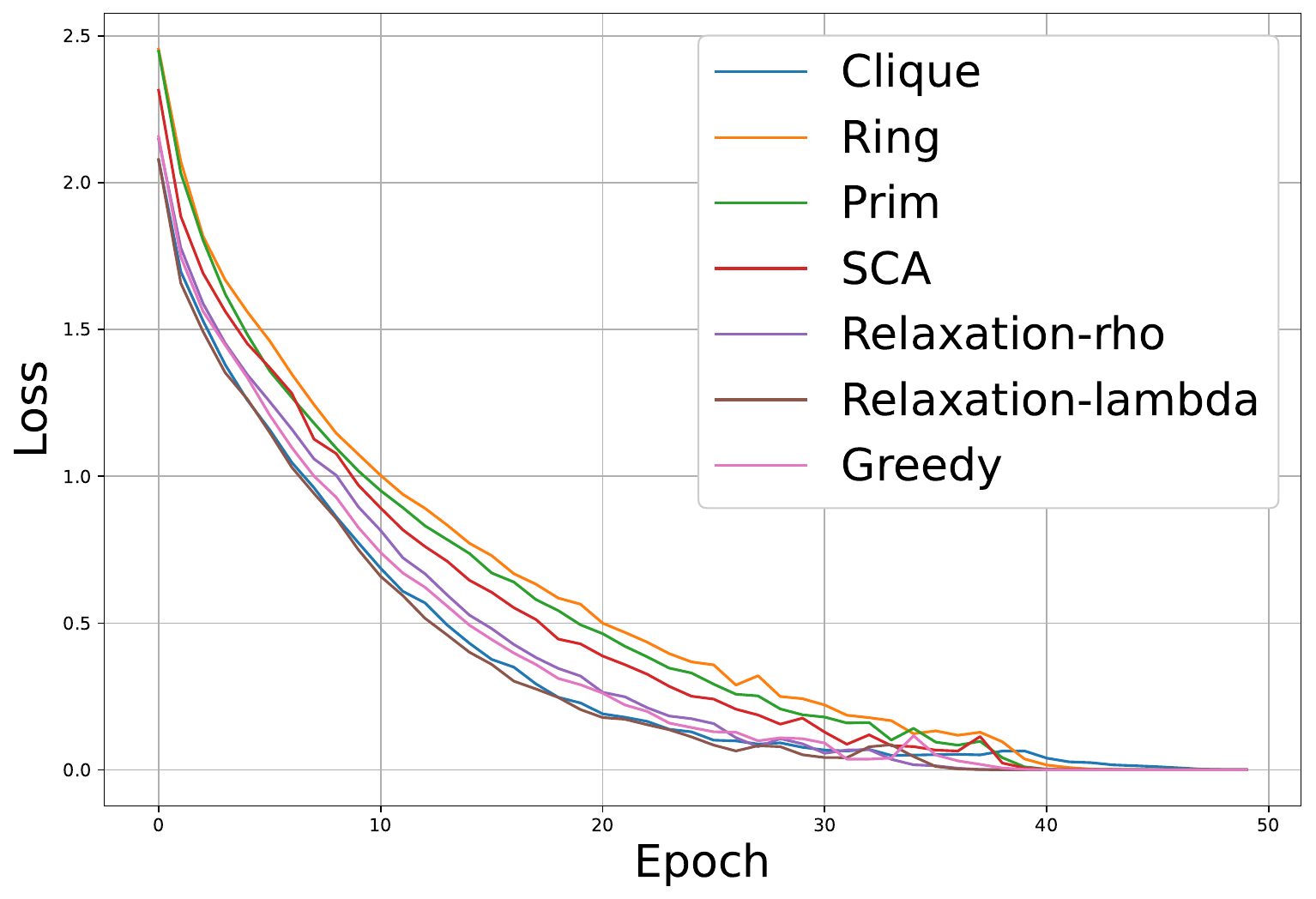}}
\vspace{-.1em}
\end{minipage}
\begin{minipage}{.495\linewidth}
\centerline{
\includegraphics[width=1\linewidth]{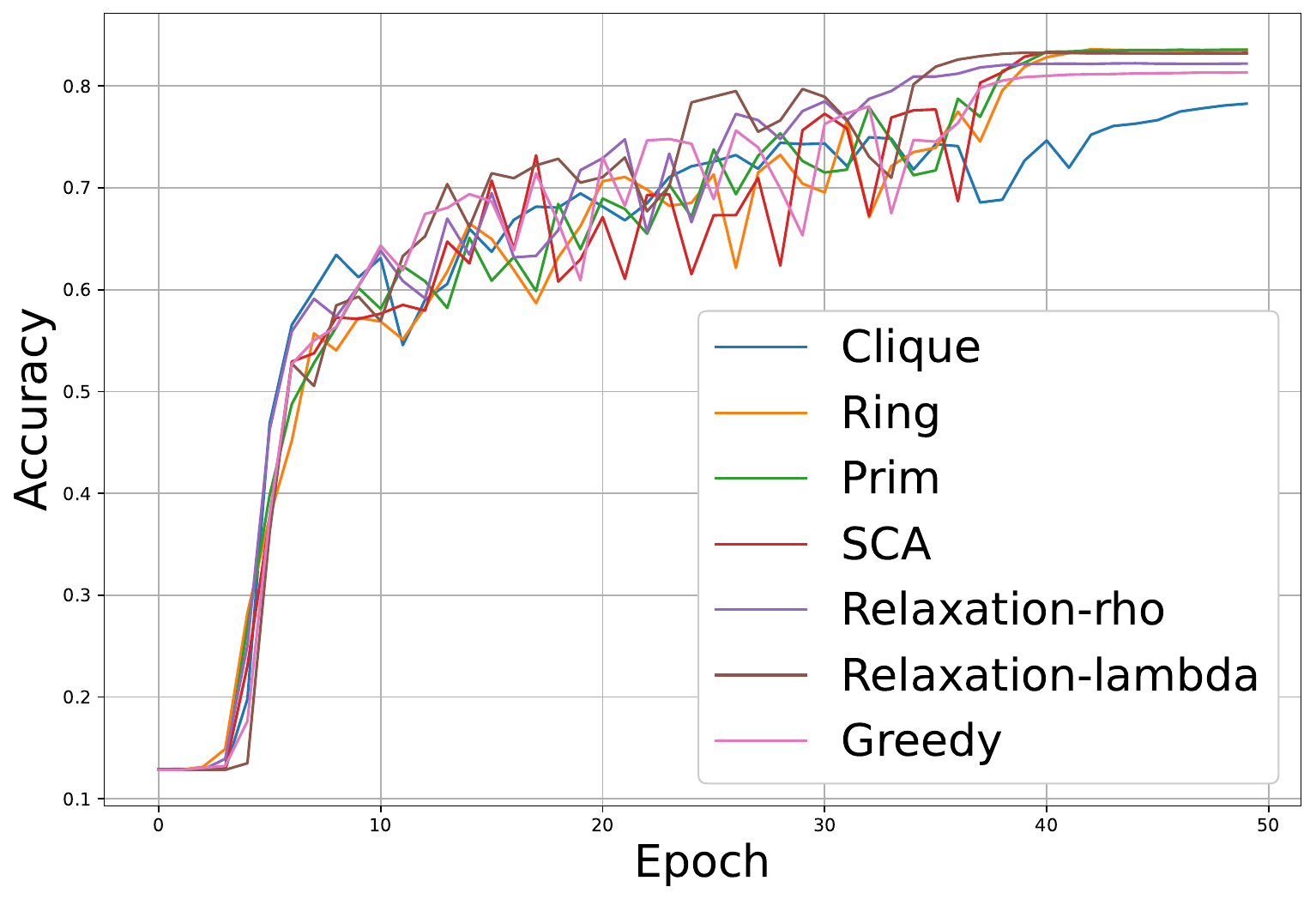}}
\vspace{-.1em}
\end{minipage}
\begin{minipage}{.495\linewidth}
\centerline{
\includegraphics[width=1\linewidth]{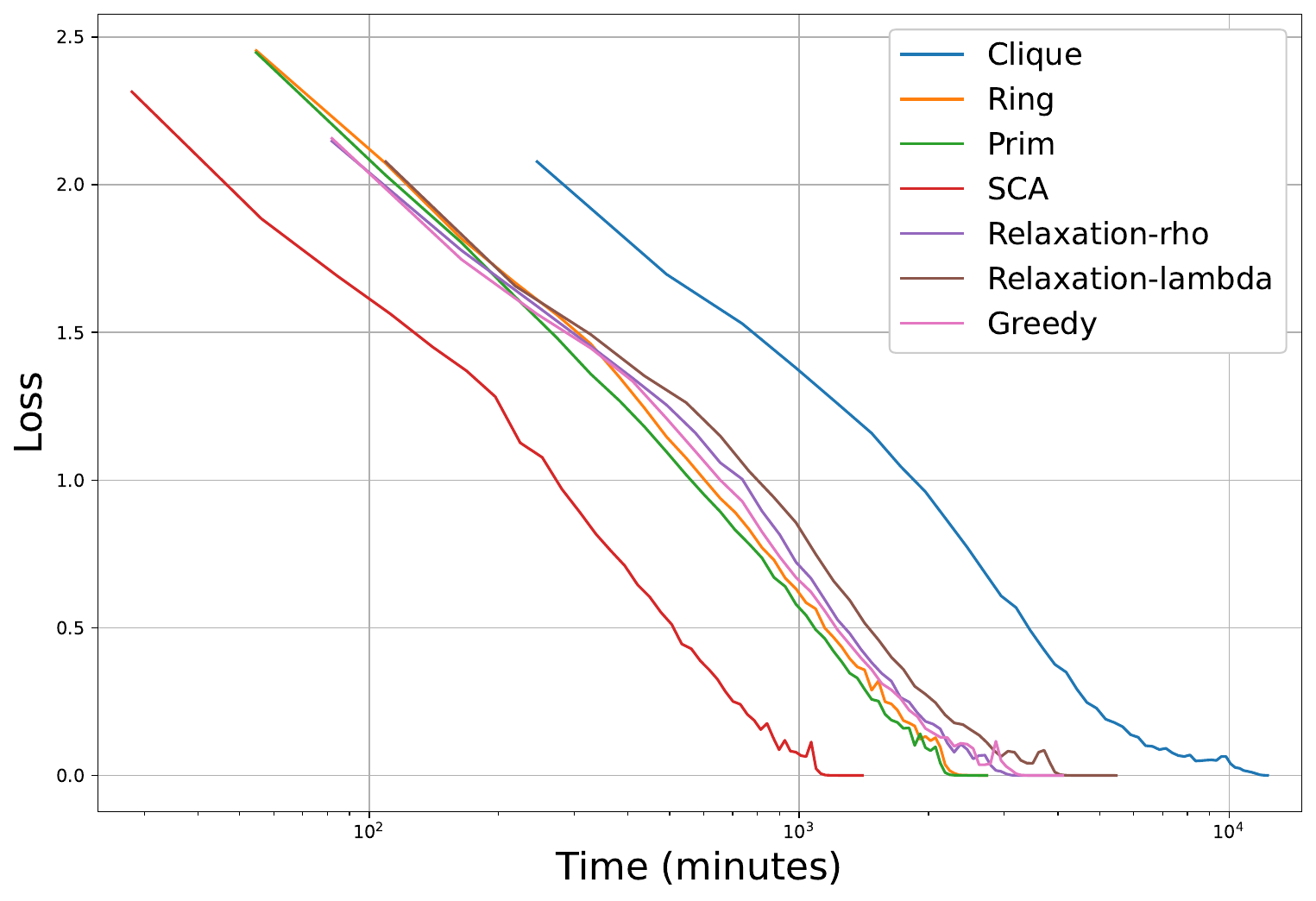}}
\vspace{-.1em}
\end{minipage}
\begin{minipage}{.495\linewidth}
\centerline{
\includegraphics[width=1\linewidth]{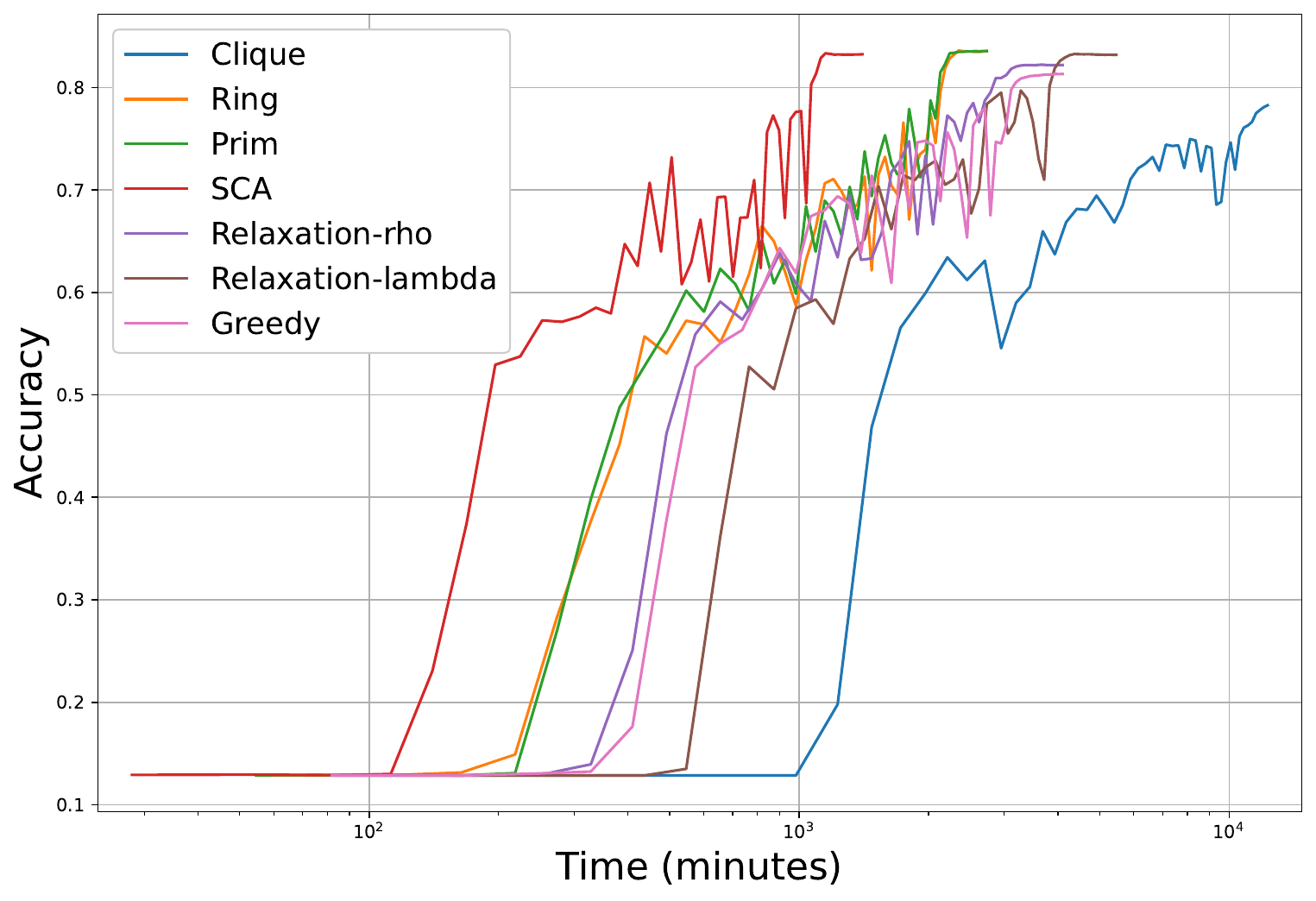}}
\vspace{-.1em}
\end{minipage}
\vspace{-1.25em}
\caption{CIFAR-10 over Roofnet: without overlay routing. 
} \label{fig:roofnet_no_routing}
\vspace{-.5em}
\end{figure}

As the most difficult part of our problem is topology design (i.e., optimization of $E_a$), we first compare the topology design solutions without overlay routing. As shown in Fig.~\ref{fig:roofnet_no_routing}, (i) using sparse topologies rather than the clique can effectively reduce the training time without compromising the performance at convergence, (ii) different topology designs only slightly differ in terms of the convergence rate over epochs, but can differ significantly in terms of the convergence rate over the actual (wall-clock) time, and (iii) the proposed design by `SCA' notably outperforms the others in terms of training time, while achieving the same loss/accuracy at convergence. Note that the time axis is in log scale. A closer examination further shows that the existing topology designs (`Prim', `Ring') work relatively well in that they not only converge much faster than the baseline (`Clique') but also outperform the other heuristics based on the optimizations we formulate (`Relaxation-$\rho$', `Relaxation-$\lambda$', `Greedy'). Nevertheless, `SCA' is able to converge even faster by better approximating the optimal solution to \eqref{eq:lower-level ICP}. 
Note that although we have not optimized overlay routing, the proposed algorithm still benefits from the knowledge of how links are shared by routing paths within the underlay (via constraint \eqref{eq:independence constraint}), which allows it to better balance the convergence rate and the communication time per iteration, while the state-of-the-art design (`Prim') ignores such link sharing. This result highlights the importance of underlay-aware design for overlay-based DFL.

Meanwhile, we note that the simpler heuristics `Relaxation-$\rho$' and `Relaxation-$\lambda$' outperform `SCA' in terms of running time, as shown in Table~\ref{tab: run_time_algs}, and all the algorithms based on our optimizations are slower than `Prim'. This indicates further room for improvement for future work in terms of the tradeoff between the quality of design and the computational efficiency.

\begin{table}[]
\small
    \centering
    \begin{tabular}{c|c|c|c|c}
         &  SCA & Relaxation-$\rho$ & Relaxation-$\lambda$ & Greedy\\
         \hline
MNIST & 21.31  & 5.96 & 6.71 & 55.84  \\       
CIFAR-10 & 19.55  & 6.12 &  6.35 & 51.74
    \end{tabular}
    \caption{Running times of the proposed algorithms relative to `Prim' (as ratios) for Roofnet.
    }
    \label{tab: run_time_algs}
    \vspace{-1.5em}
\end{table}

\subsubsection{Results with Overlay Routing}

\begin{figure}[t!]
\begin{minipage}{.495\linewidth}
\centerline{
\includegraphics[width=1\linewidth]{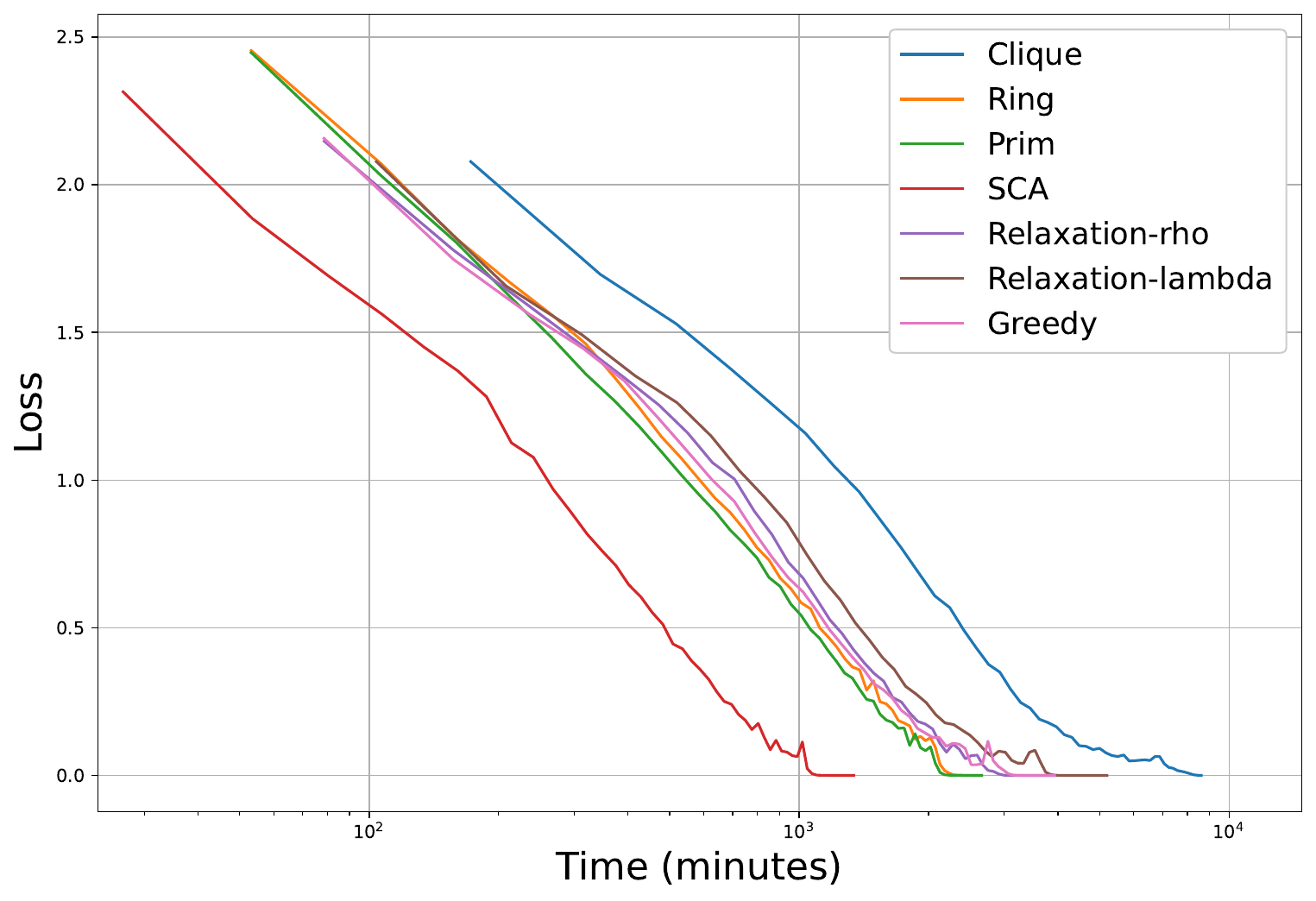}}
\vspace{-.1em}
\end{minipage}
\vspace{-.1em}
\begin{minipage}{.495\linewidth}
\centerline{
\includegraphics[width=1\linewidth]{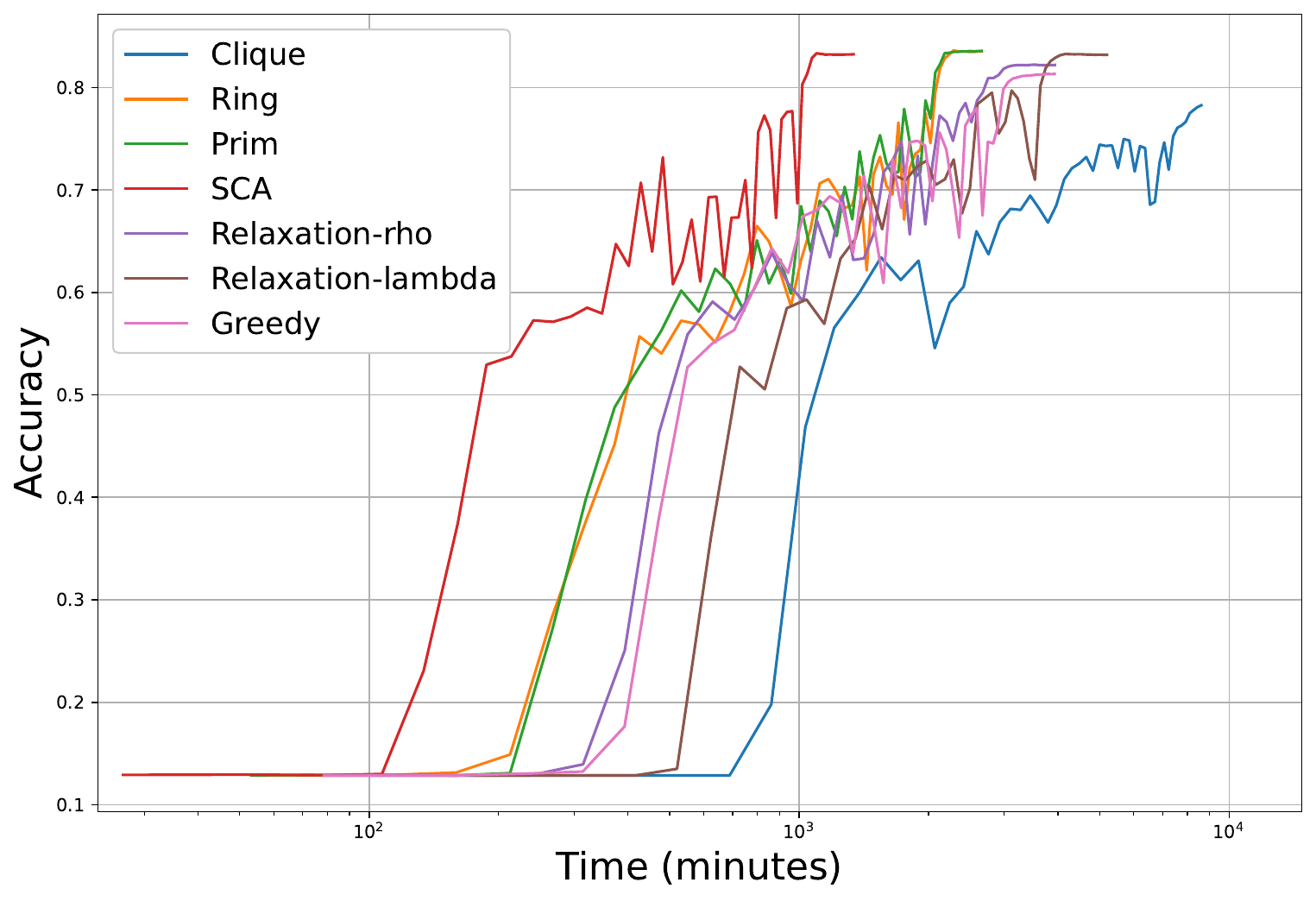}}
\vspace{-.1em}
\end{minipage}
\vspace{-1.25em}
\caption{CIFAR-10 over Roofnet: with overlay routing. 
} \label{fig:roofnet_routing}
\vspace{-.5em}
\end{figure}

Fig.~\ref{fig:roofnet_routing} shows the results after optimizing the communication schedule under each design by the overlay routing optimization \eqref{eq:min-time}. Compared with Fig.~\ref{fig:roofnet_no_routing} (second row), we see that the training time is further reduced for all the topology designs. However, the improvement is only prominent for the dense topology (`Clique') with a reduction of $28\%$, while the improvement for the other topologies is  incremental ($2$--$4\%$). Intuitively, this is because these sparse topologies generate much less load on the underlay network, and hence leave less room for improvement for overlay routing.

\subsubsection{Results with Inference Errors}

While the above results are obtained under the perfect knowledge of the nonempty categories $\mathcal{F}$ and the category capacities $(C_F)_{F\in \mathcal{F}}$, the observations therein remain valid under the inferred values 
of these parameters, 
\if\thisismainpaper1
as our inference algorithms proposed in \cite{Huang24TONsub} are able to infer these parameters with sufficient accuracy. We thus defer the detailed results to \cite{Huang24:report} due to space limitation. \looseness=0 
\else
as shown in Fig.~\ref{fig:Roofnet_CIFAR10_inference}, where the inference is performed by the T-COIN algorithm in \cite{Huang24TONsub} based on packet-level simulations in NS3. Compared with the results under perfect knowledge in Fig.~\ref{fig:roofnet_no_routing}--\ref{fig:roofnet_routing}, designing the communication demands and schedule based on the inferred information only slightly perturbs the training performance, and the comparison between different designs remains roughly the same. 

\begin{figure}[t!]
\begin{minipage}{.495\linewidth}
\centerline{
\includegraphics[width=1\linewidth]{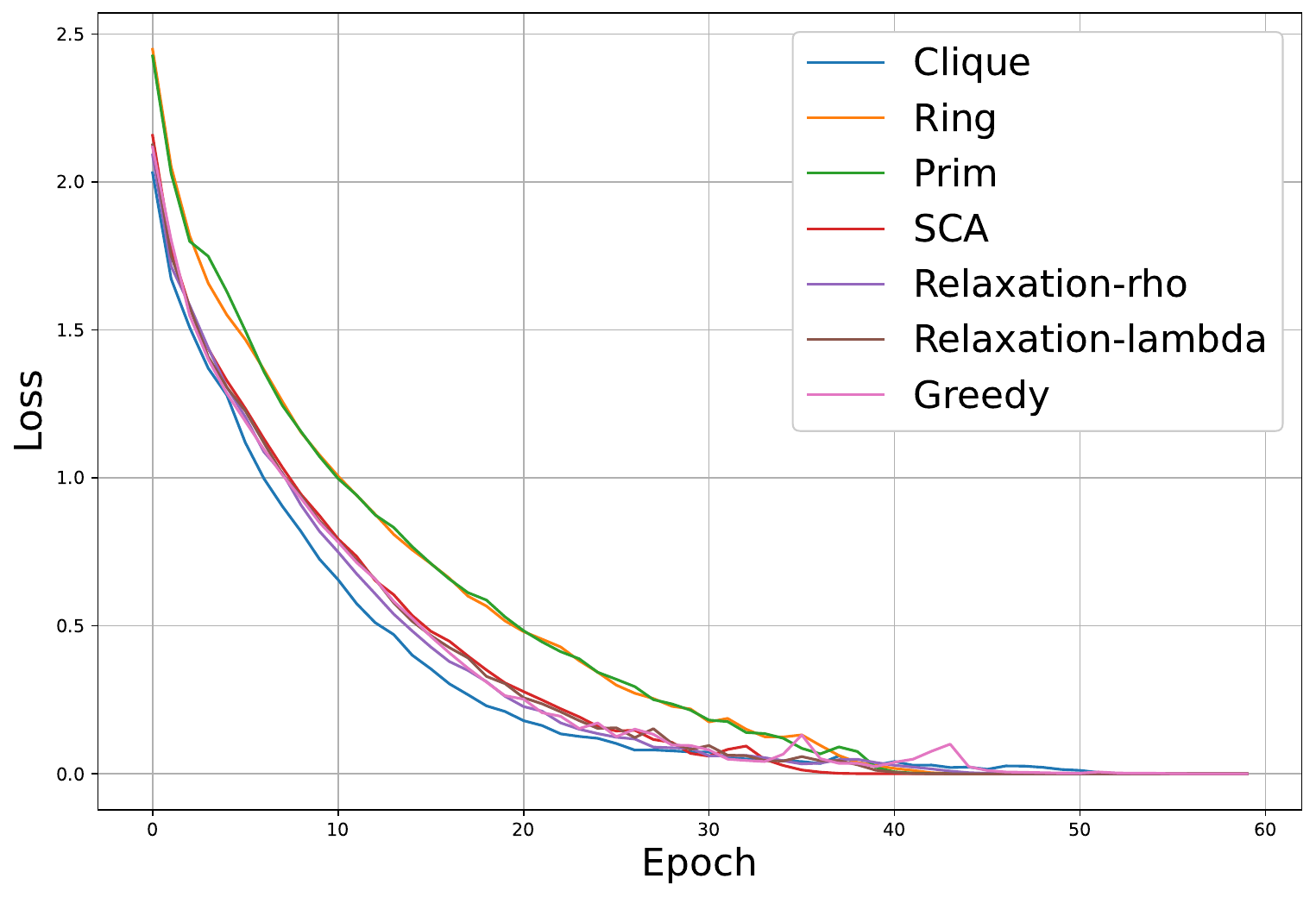}}
\vspace{-.1em}
\end{minipage}
\begin{minipage}{.495\linewidth}
\centerline{
\includegraphics[width=1\linewidth]{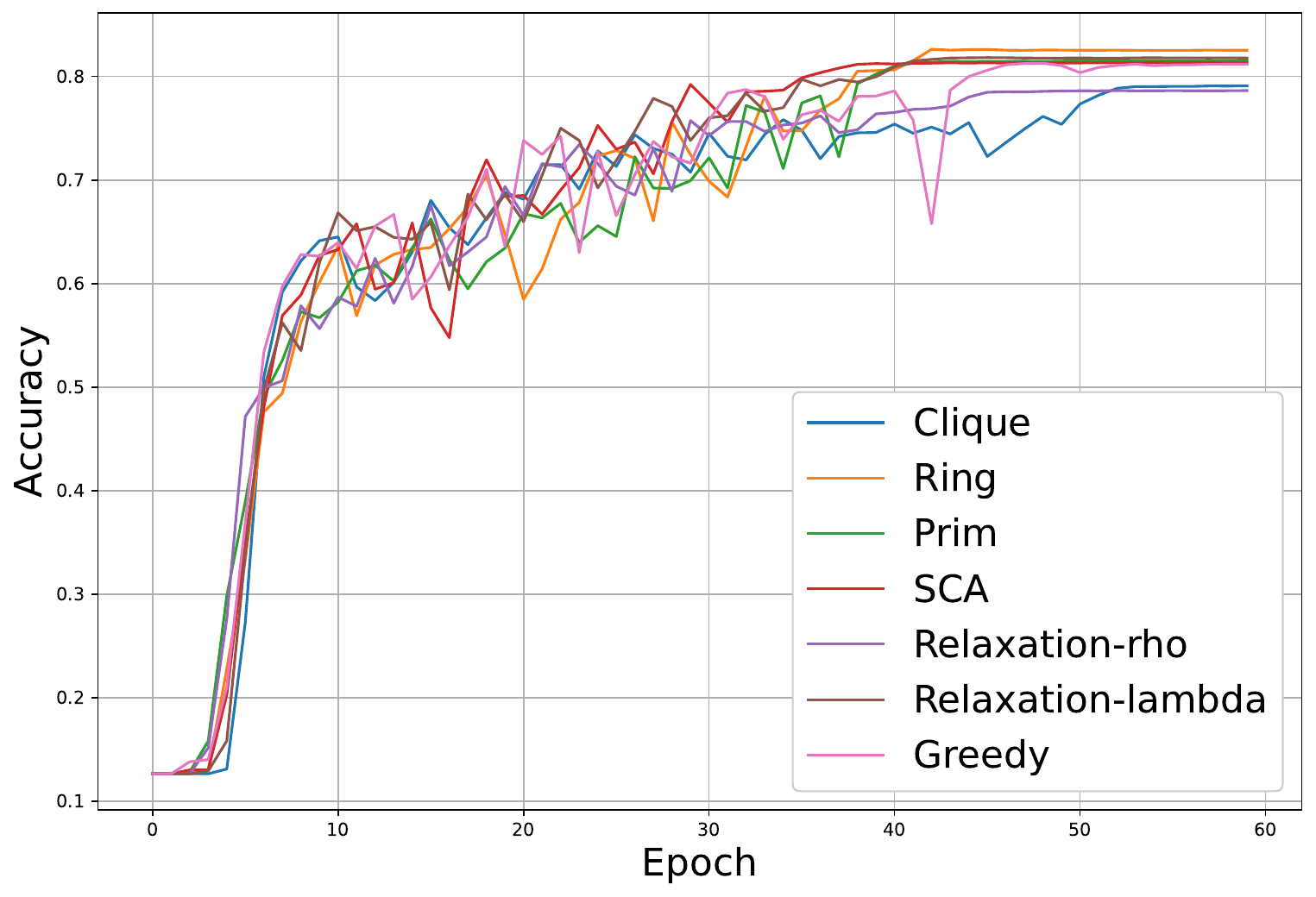}}
\vspace{-.1em}
\end{minipage}
\begin{minipage}{.495\linewidth}
\centerline{
\includegraphics[width=1\linewidth]{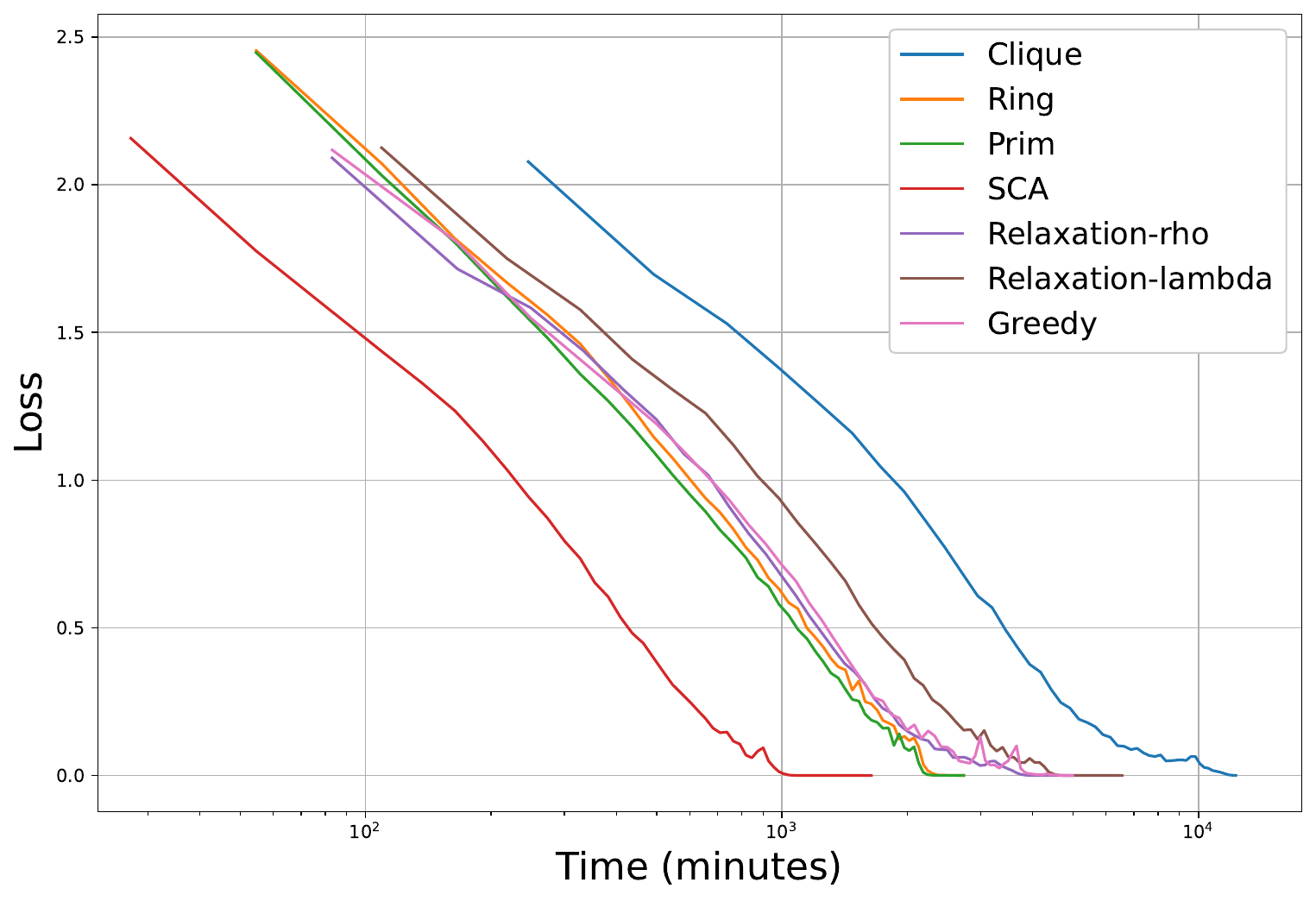}}
\vspace{-.1em}
\end{minipage}
\begin{minipage}{.495\linewidth}
\centerline{
\includegraphics[width=1\linewidth]{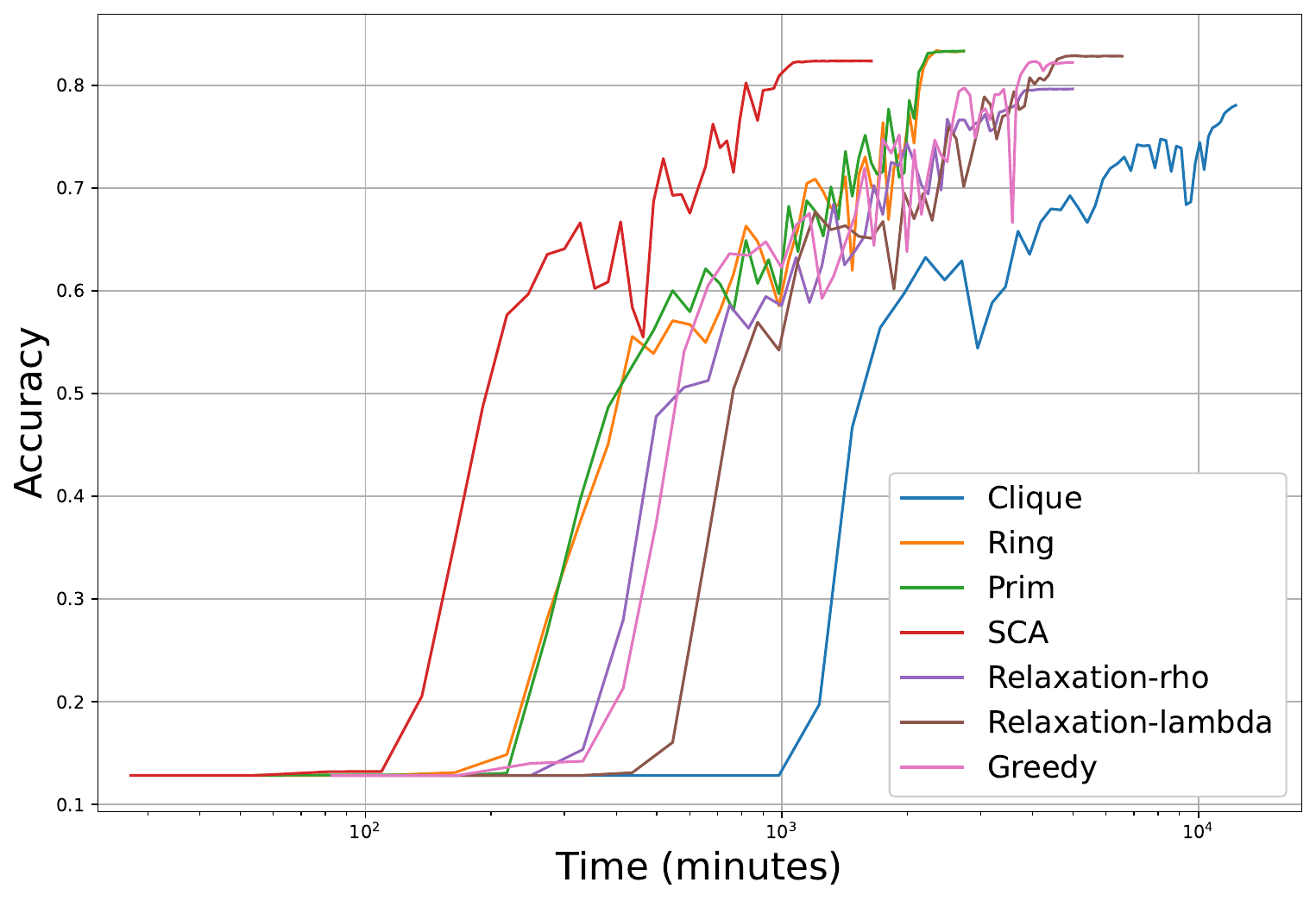}}
\vspace{-.1em}
\end{minipage}
\begin{minipage}{.495\linewidth}
\centerline{
\includegraphics[width=1\linewidth]{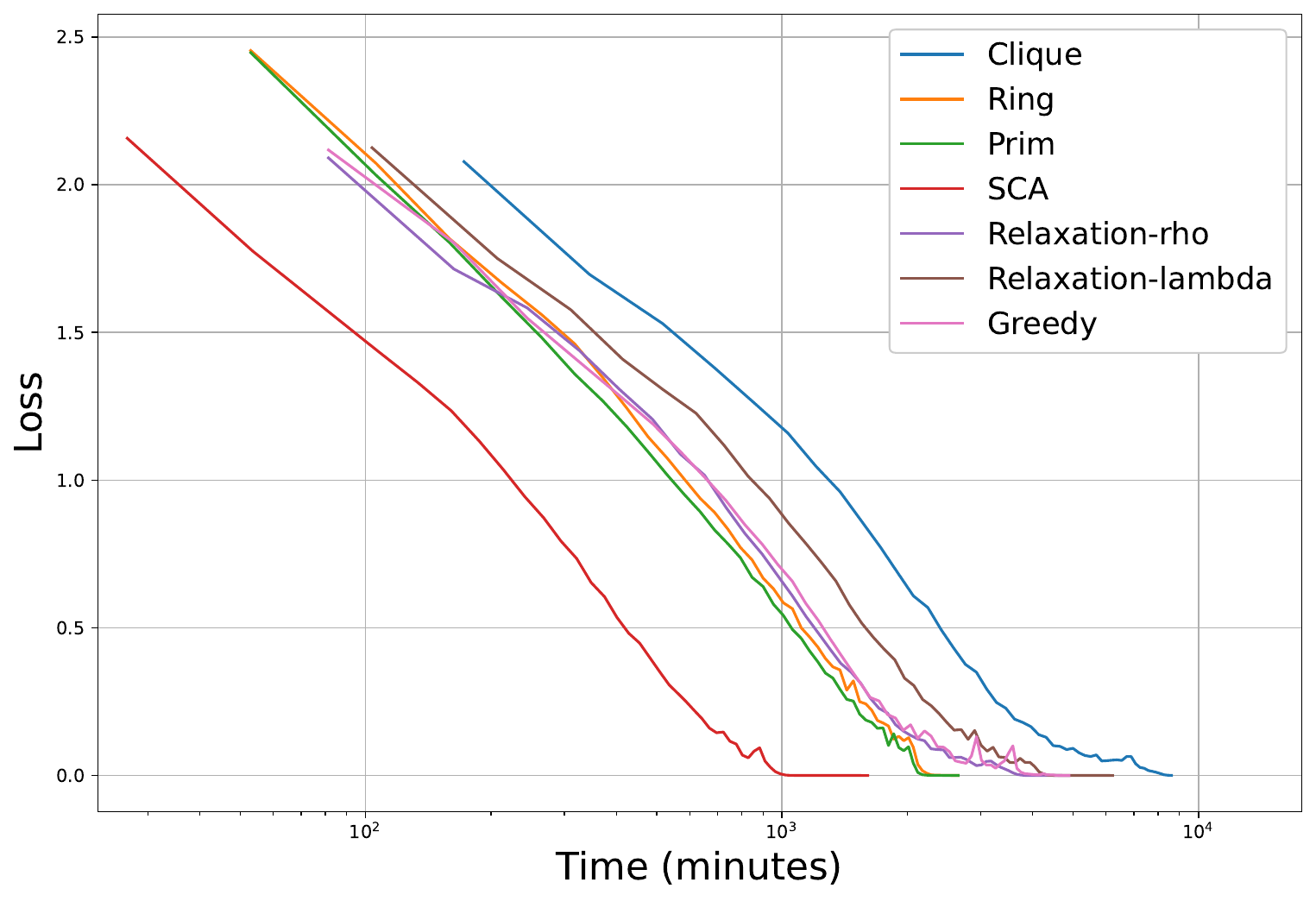}}
\vspace{-.1em}
\end{minipage}
\begin{minipage}{.495\linewidth}
\centerline{
\includegraphics[width=1\linewidth]{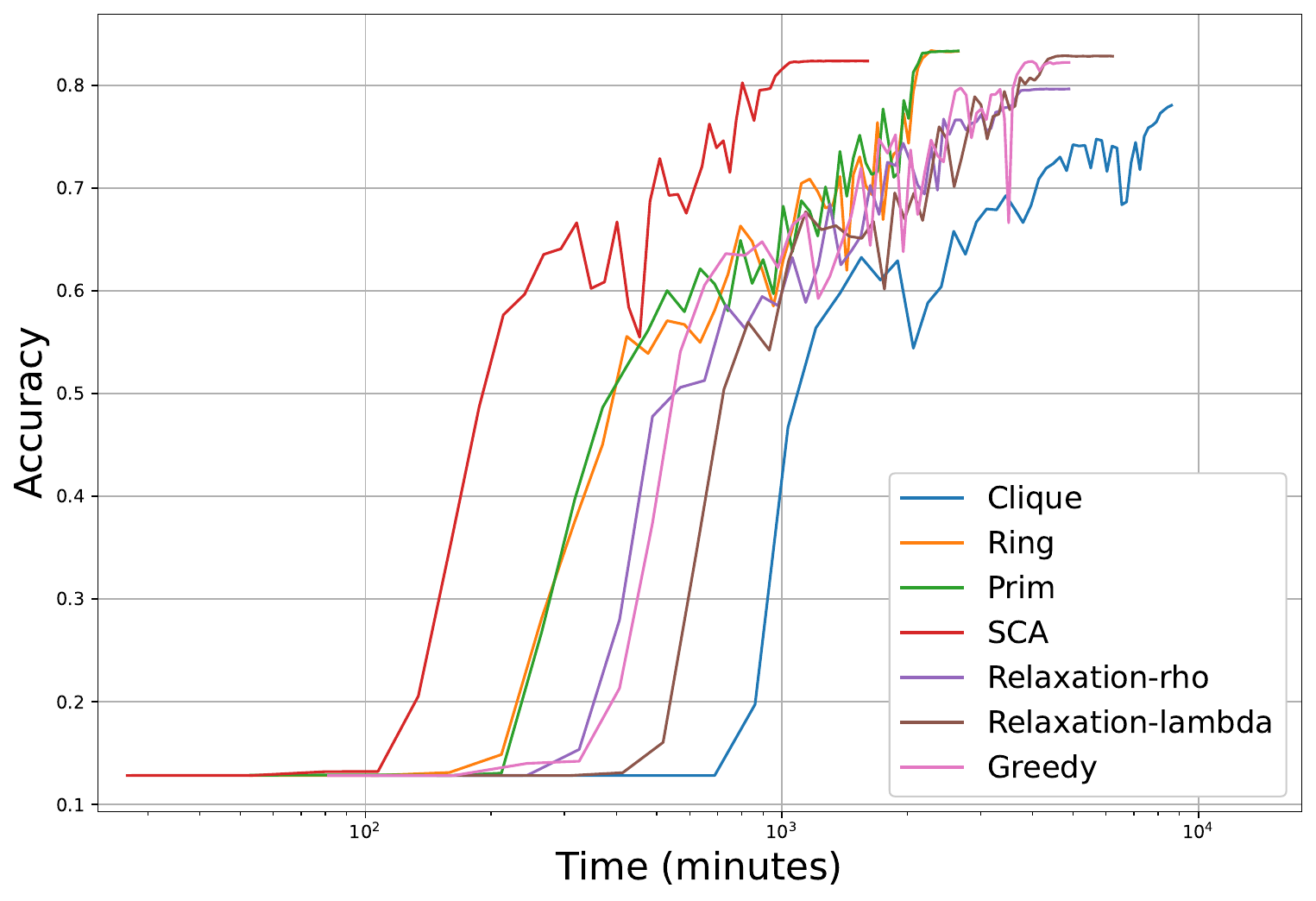}}
\vspace{-.1em}
\end{minipage}
\vspace{-1.25em}
\caption{CIFAR-10 over Roofnet with inference errors (second row: time without overlay routing; third row: time with overlay routing). 
} \label{fig:Roofnet_CIFAR10_inference}
\vspace{-.5em}
\end{figure}

\fi

\subsubsection{Results under Other Weight Design}

\if\thisismainpaper1
\rev{Instead of solving the SDP \eqref{eq:min rho wo cost}, one could use alternative designs for link weights. To assess the impact of different weight designs, we conducted additional simulations for the case of Fig.~\ref{fig:roofnet_no_routing}--\ref{fig:roofnet_routing} using the widely-adopted Metropolis-Hasting weights. The results, provided in \cite{Huang24:report}, show that the Metropolis-Hasting weights introduce a noticeable delay in convergence compared to our proposed weight design. 
}
\else
Instead of solving the SDP \eqref{eq:min rho wo cost}, one could use other designs of link weights. To see the impact of weight design, we have repeated the simulations in Fig.~\ref{fig:roofnet_no_routing}--\ref{fig:roofnet_routing} under a common design called Metropolis-Hasting weights~\cite{Xiao2006DistributedAC}. The results in Fig.~\ref{fig:roofnet_CIFAR10_metropolis} show that the Metropolis-Hasting weights introduce a noticeable delay in convergence compared to our proposed weight design, as evidenced by comparing each curve in Fig.~\ref{fig:roofnet_CIFAR10_metropolis} with its counterpart in Fig.~\ref{fig:roofnet_no_routing}--\ref{fig:roofnet_routing}.

\begin{figure}[t!]
\begin{minipage}{.495\linewidth}
\centerline{
\includegraphics[width=1\linewidth]{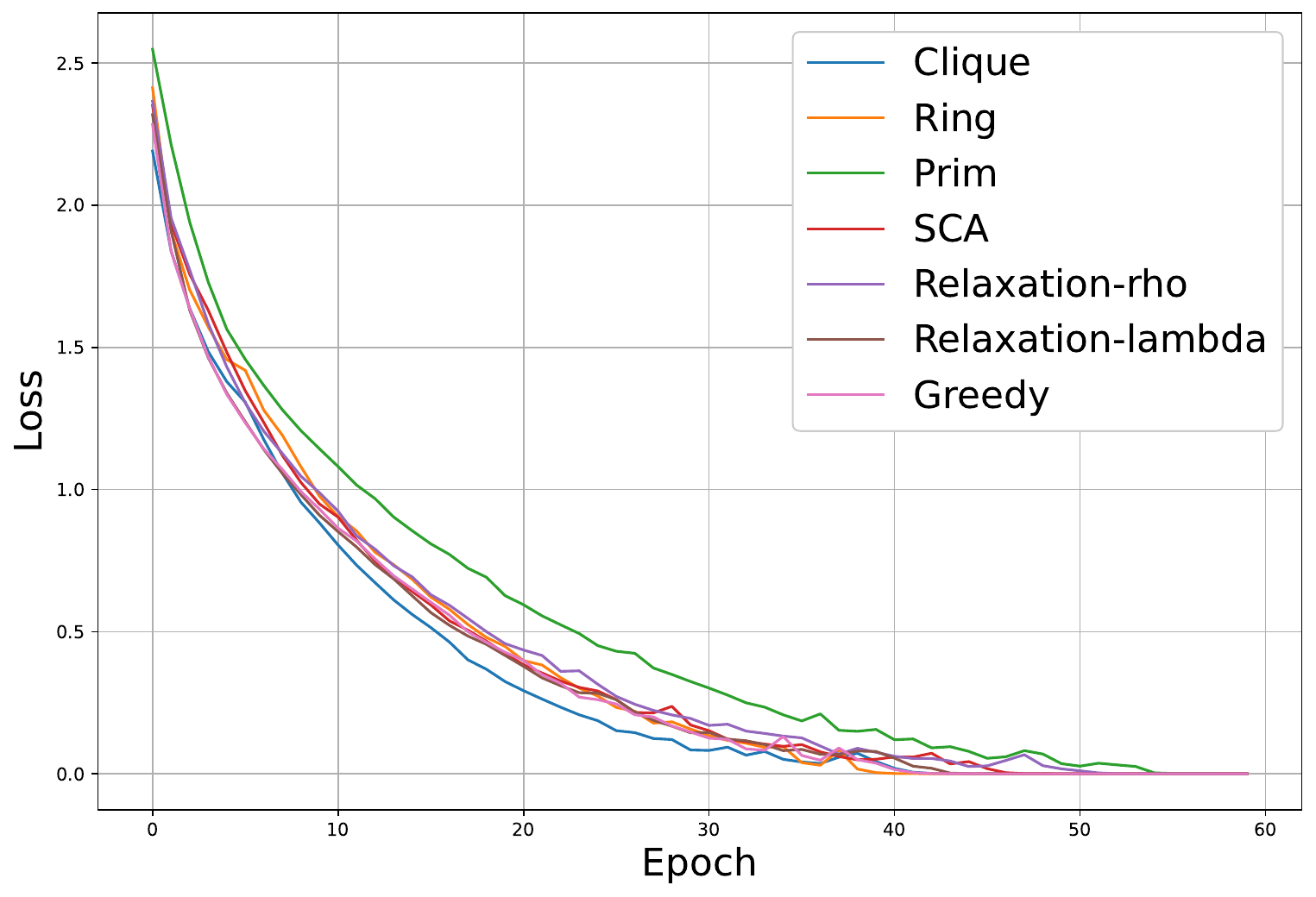}}
\vspace{-.1em}
\end{minipage}
\begin{minipage}{.495\linewidth}
\centerline{
\includegraphics[width=1\linewidth]{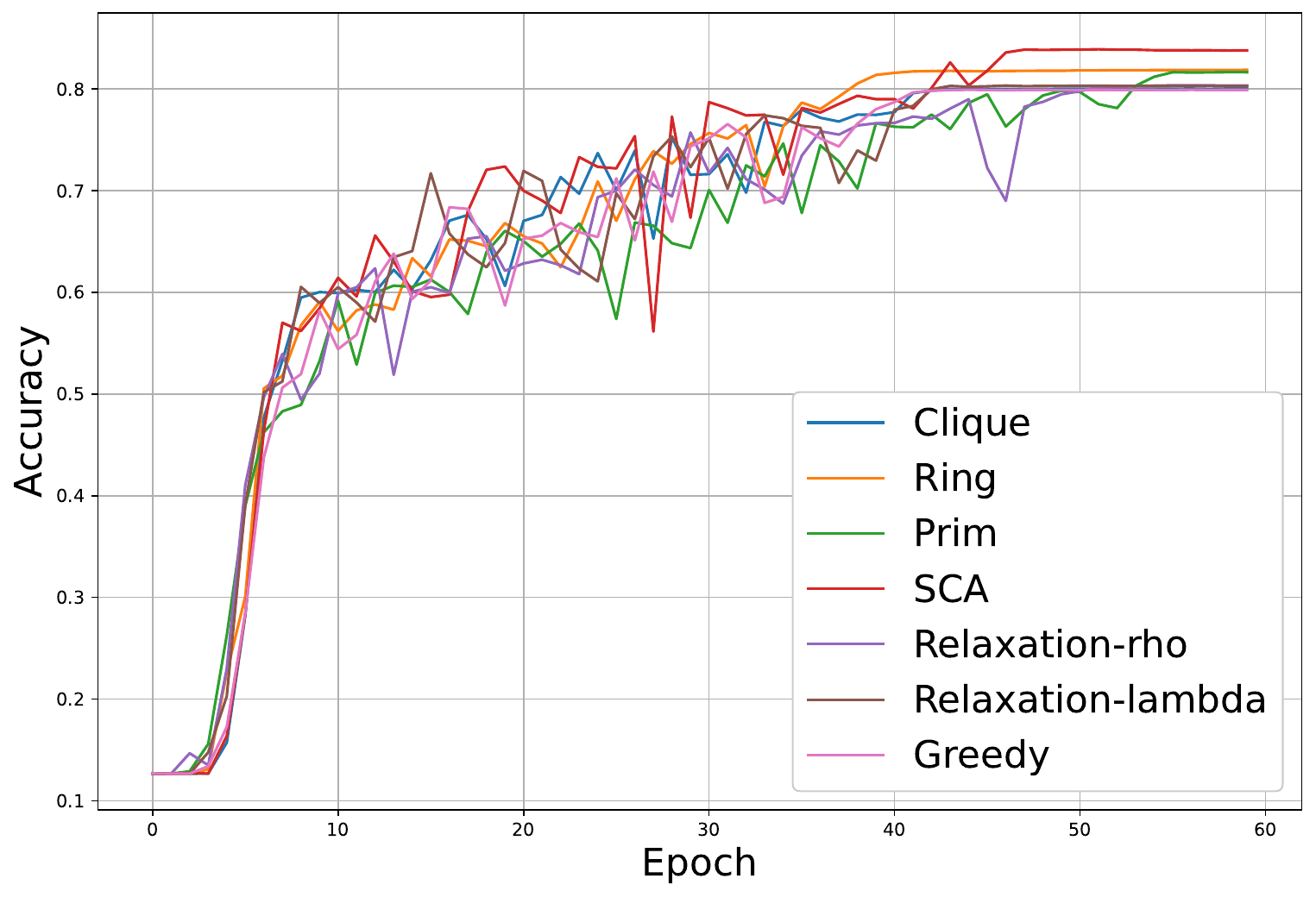}}
\vspace{-.1em}
\end{minipage}
\begin{minipage}{.495\linewidth}
\centerline{
\includegraphics[width=1\linewidth]{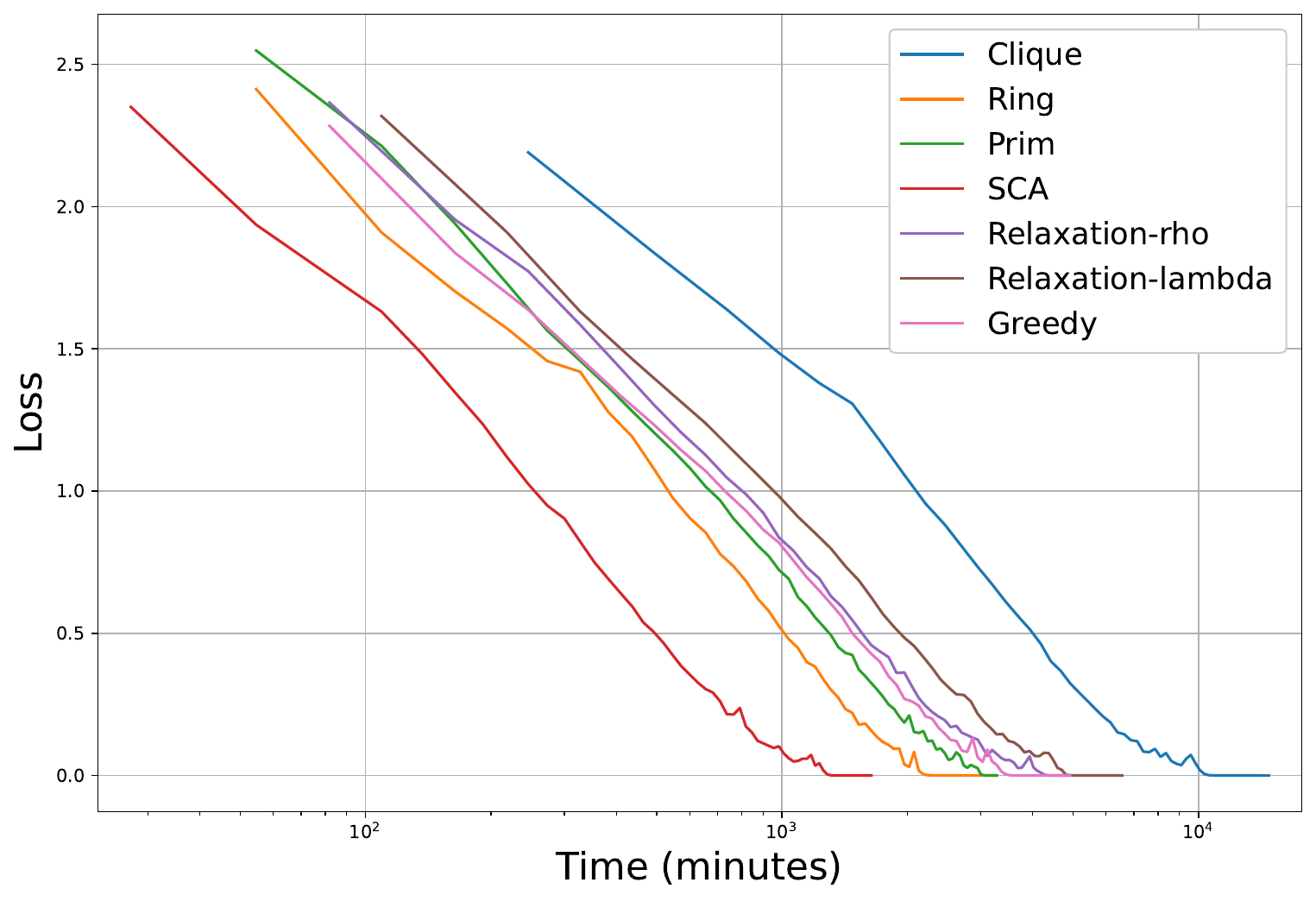}}
\vspace{-.1em}
\end{minipage}
\begin{minipage}{.495\linewidth}
\centerline{
\includegraphics[width=1\linewidth]{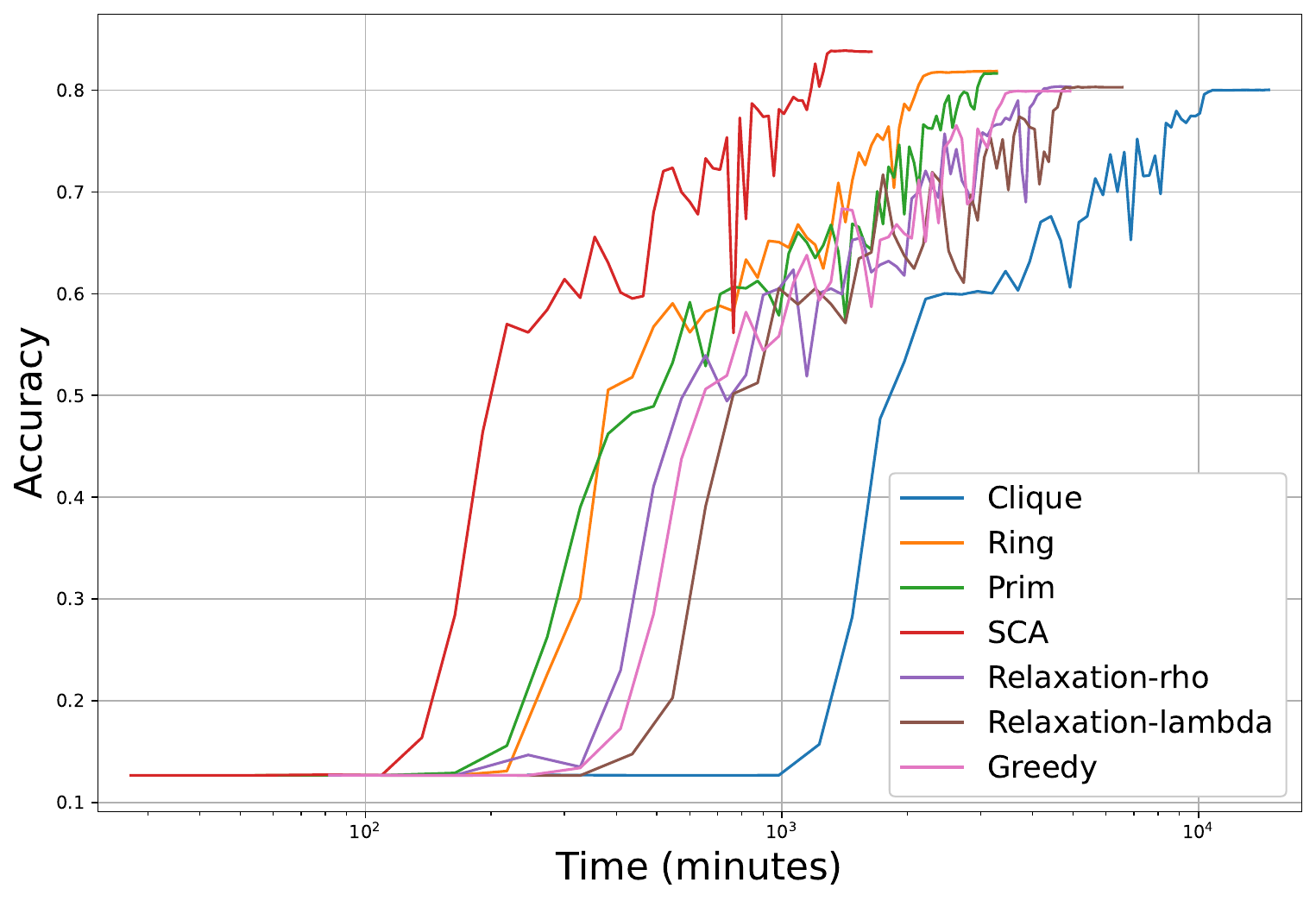}}
\vspace{-.1em}
\end{minipage}
\begin{minipage}{.495\linewidth}
\centerline{
\includegraphics[width=1\linewidth]{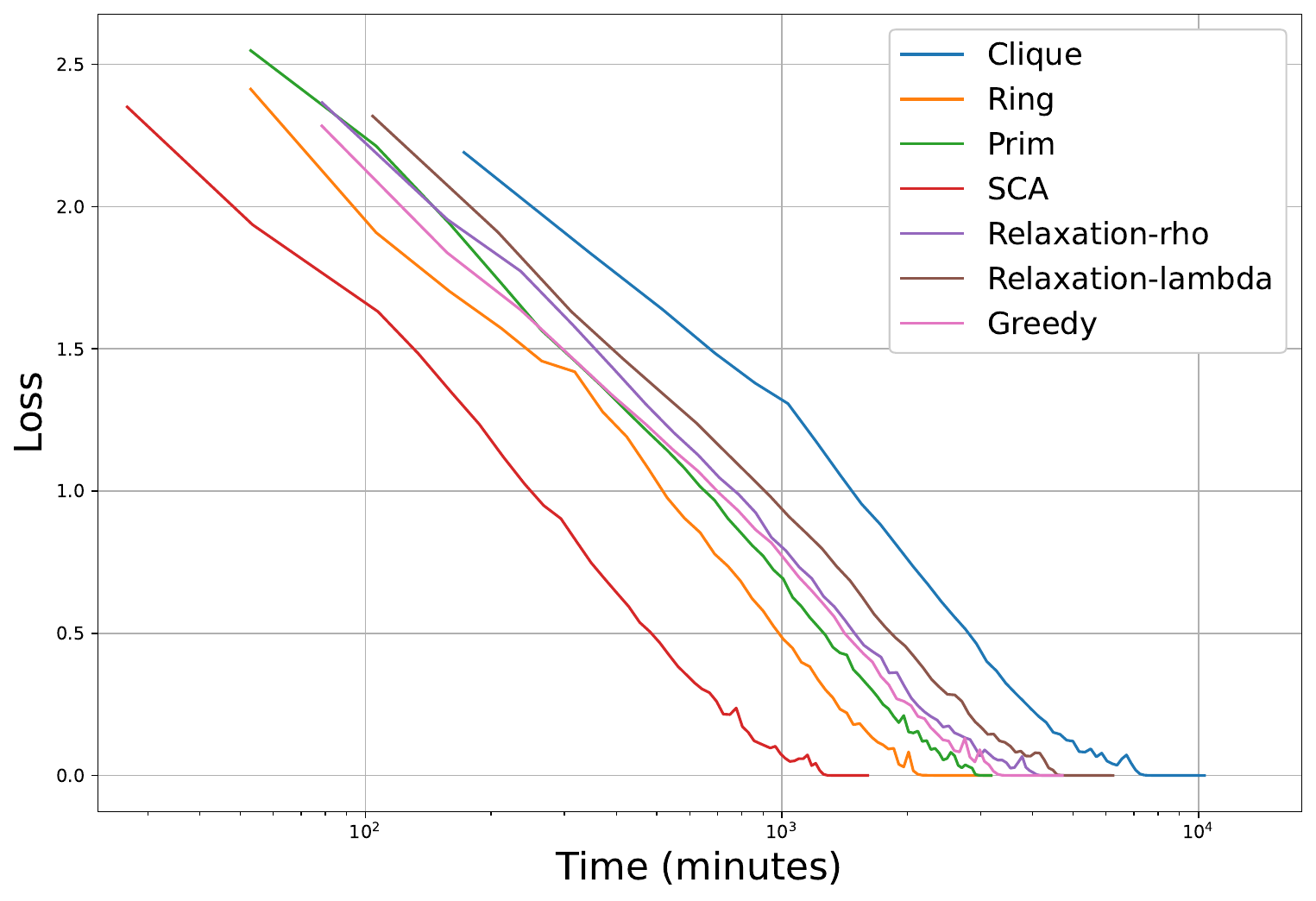}}
\vspace{-.1em}
\end{minipage}
\begin{minipage}{.495\linewidth}
\centerline{
\includegraphics[width=1\linewidth]{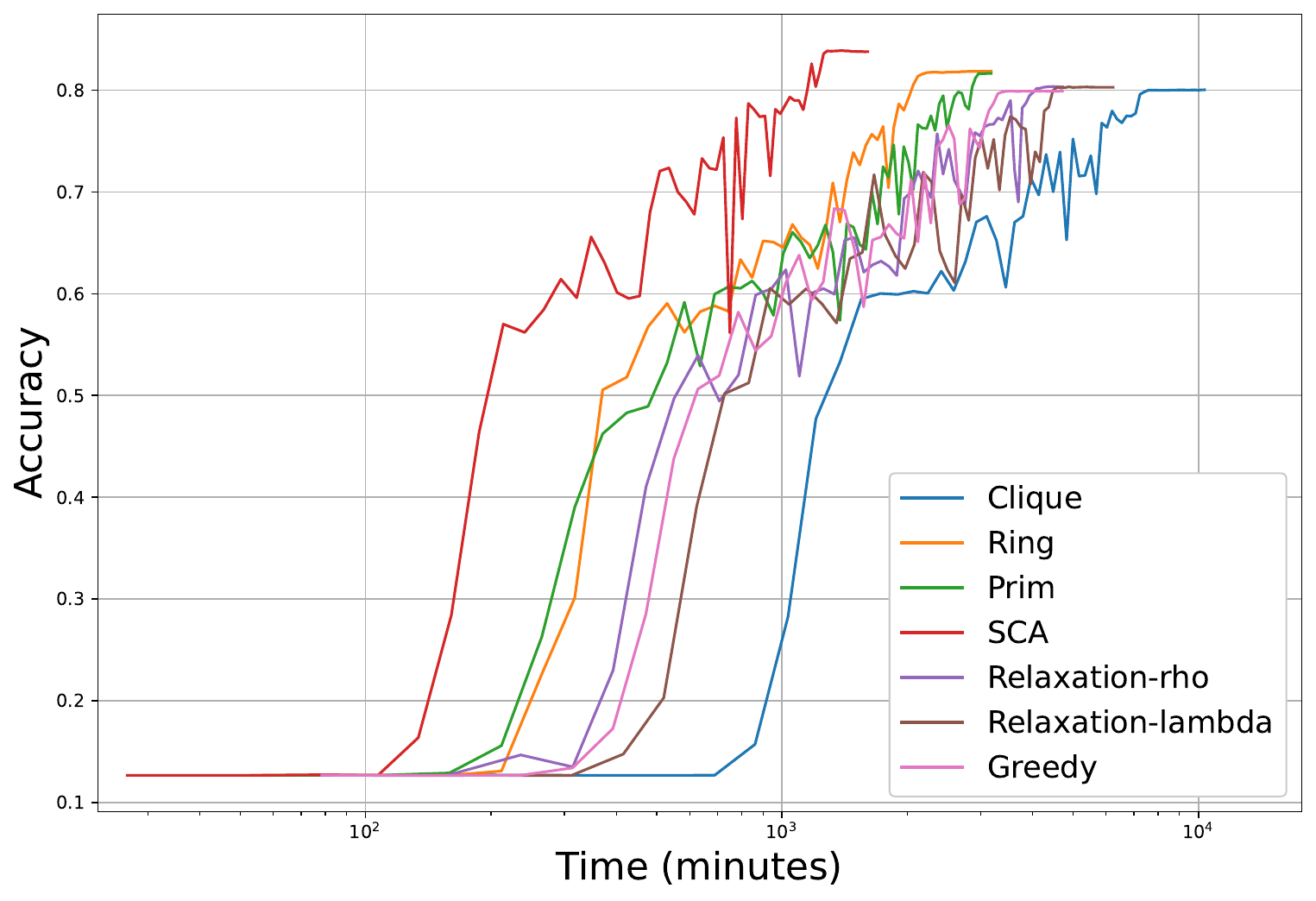}}
\vspace{-.1em}
\end{minipage}
\vspace{-1.25em}
\caption{CIFAR10 over Roofnet using Metropolis-Hasting weights (second row: time without overlay routing; third row: time with overlay routing). 
} \label{fig:roofnet_CIFAR10_metropolis}
\vspace{-.5em}
\end{figure}
\fi

\if\thisismainpaper0

\subsubsection{Results on MNIST}

\begin{figure}[t!]
\begin{minipage}{.495\linewidth}
\centerline{
\includegraphics[width=1\linewidth]{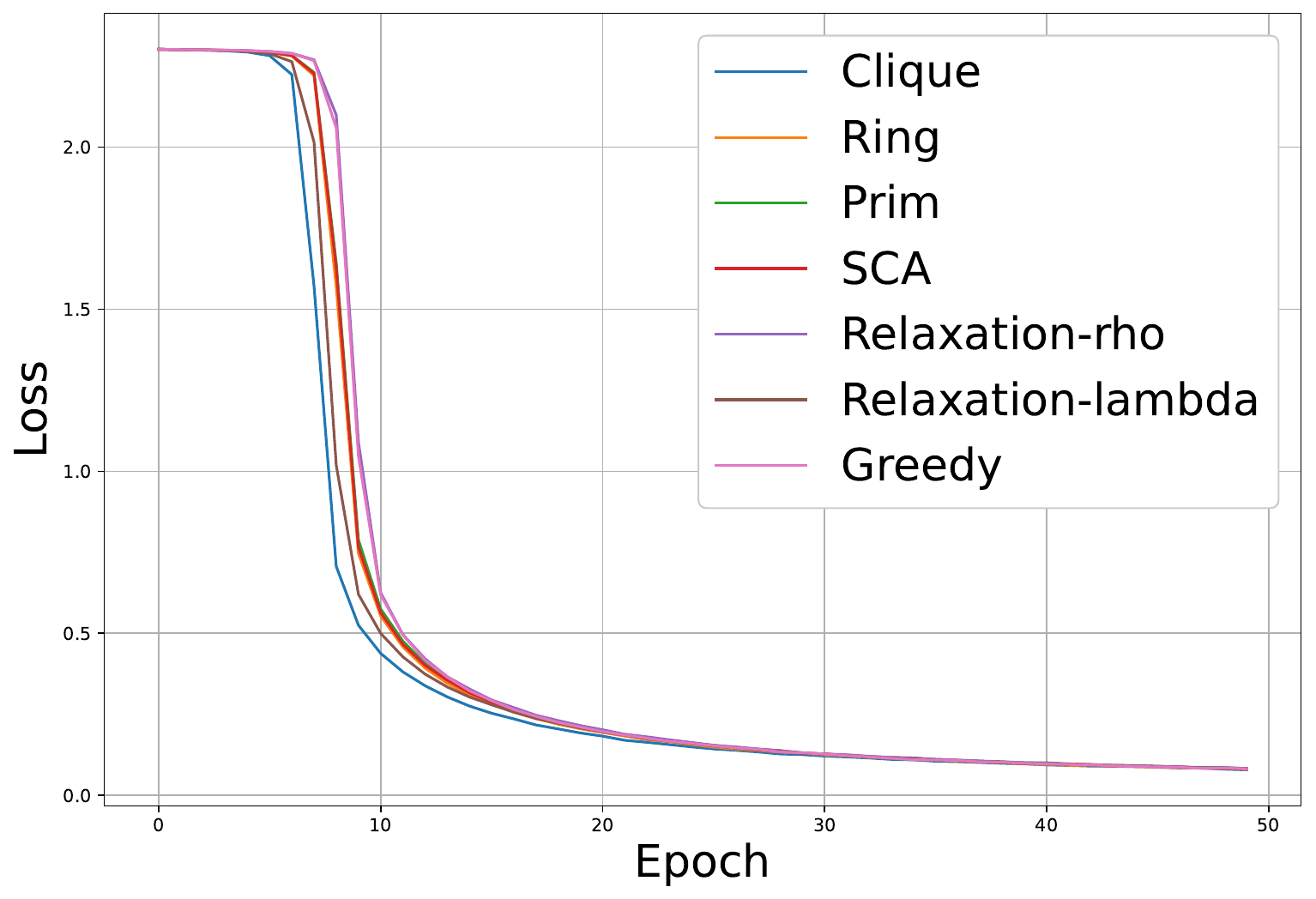}}
\vspace{-.1em}
\end{minipage}
\begin{minipage}{.495\linewidth}
\centerline{
\includegraphics[width=1\linewidth]{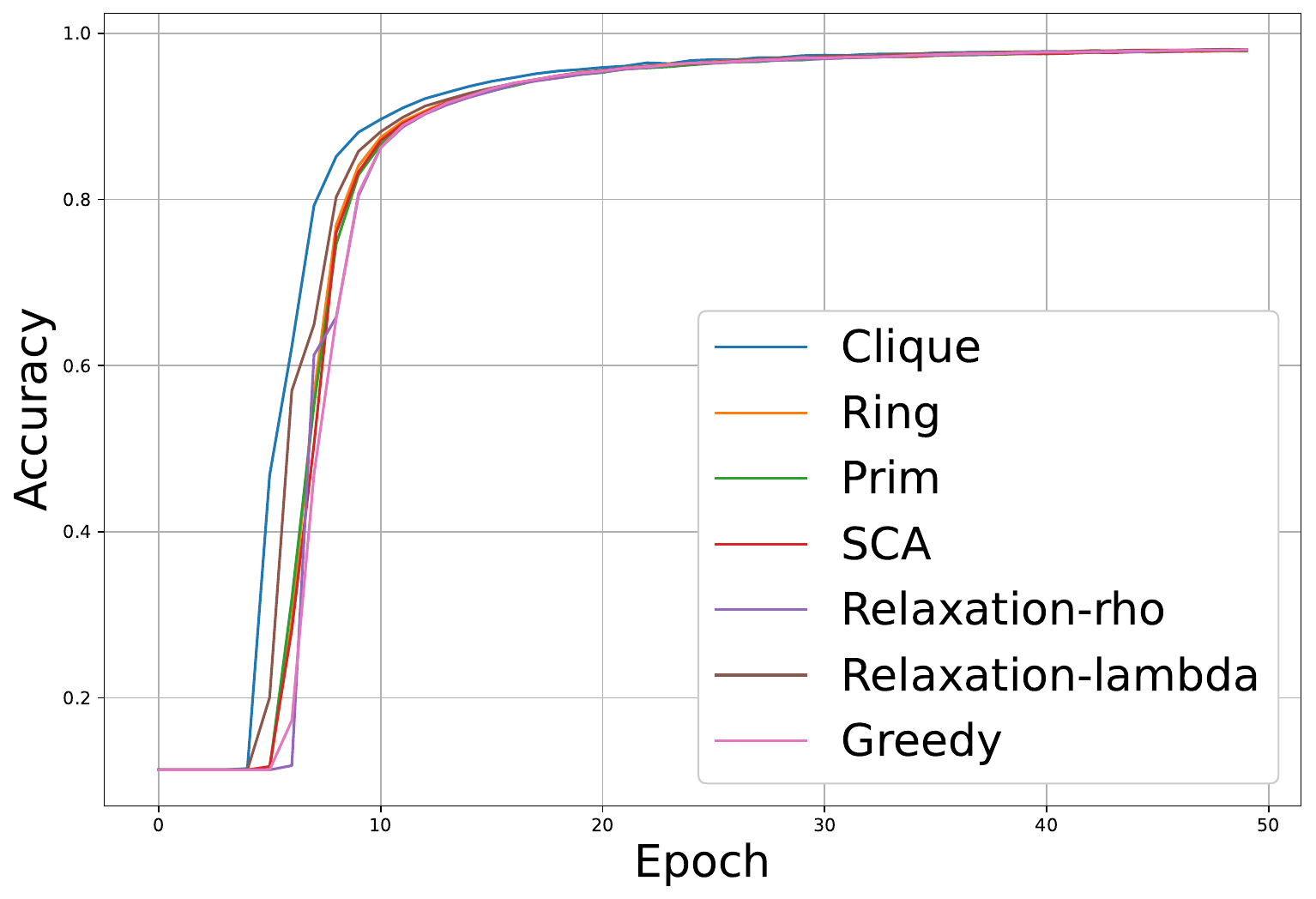}}
\vspace{-.1em}
\end{minipage}
\begin{minipage}{.495\linewidth}
\centerline{
\includegraphics[width=1\linewidth]{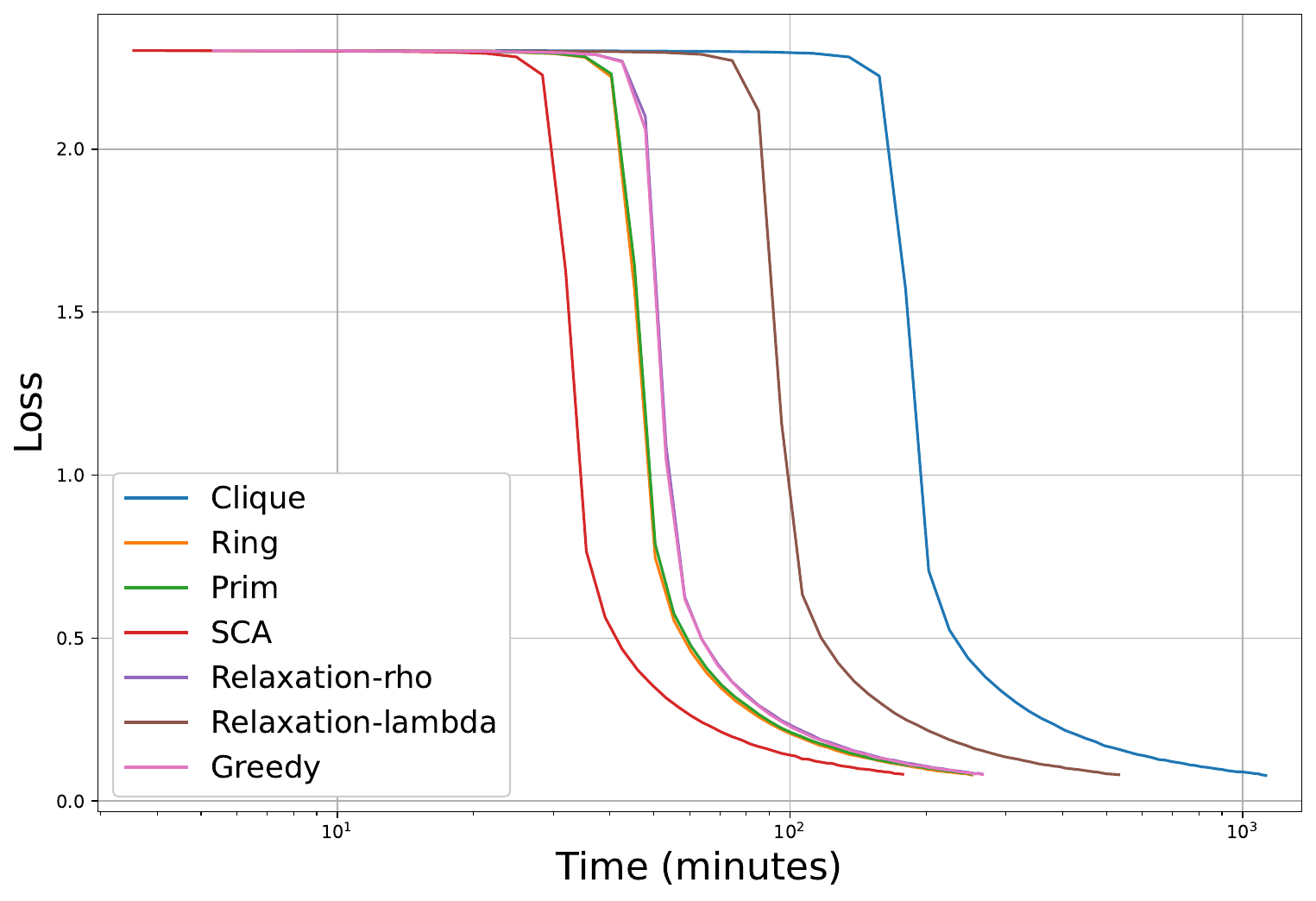}}
\vspace{-.1em}
\end{minipage}
\begin{minipage}{.495\linewidth}
\centerline{
\includegraphics[width=1\linewidth]{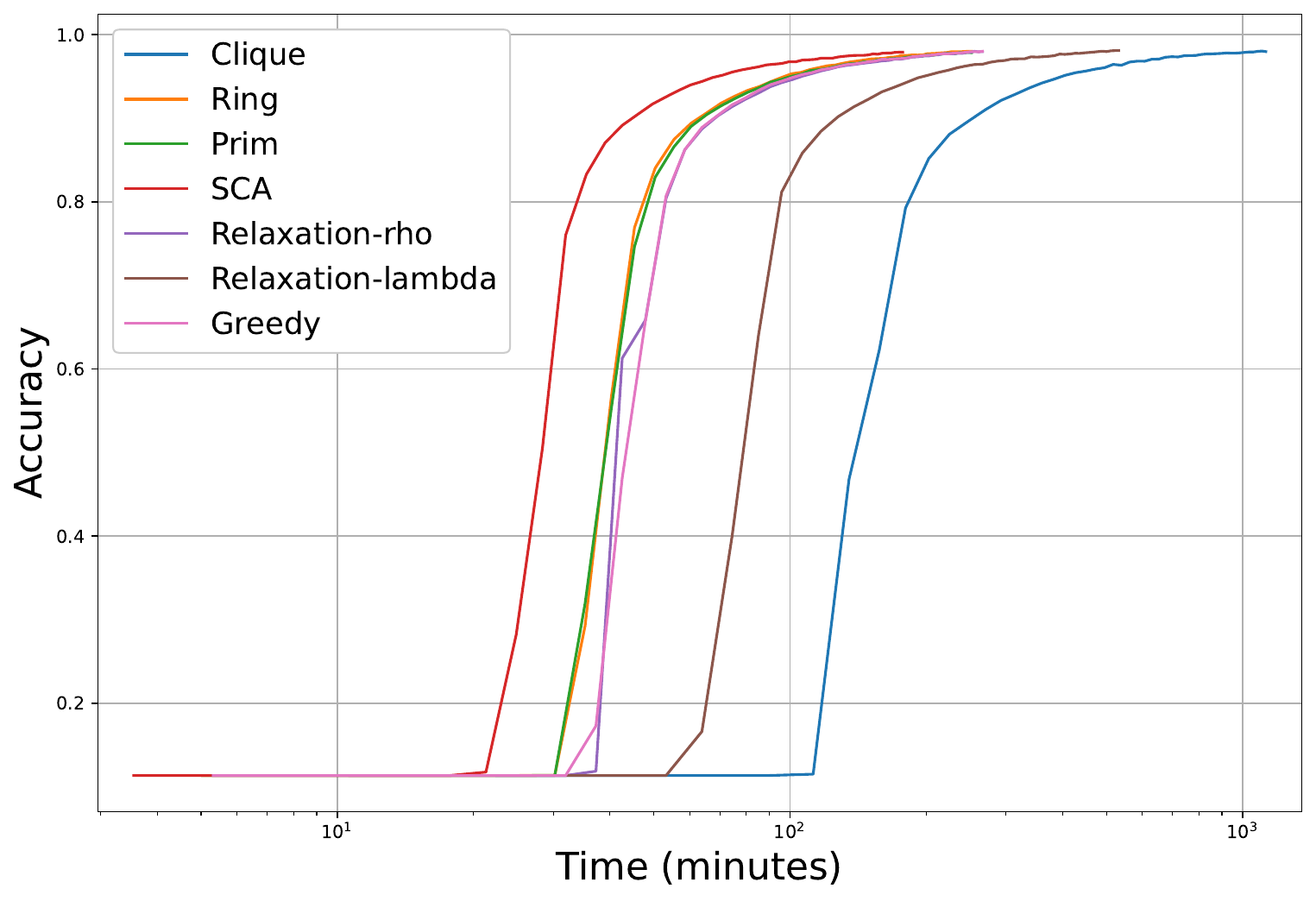}}
\vspace{-.1em}
\end{minipage}
\begin{minipage}{.495\linewidth}
\centerline{
\includegraphics[width=1\linewidth]{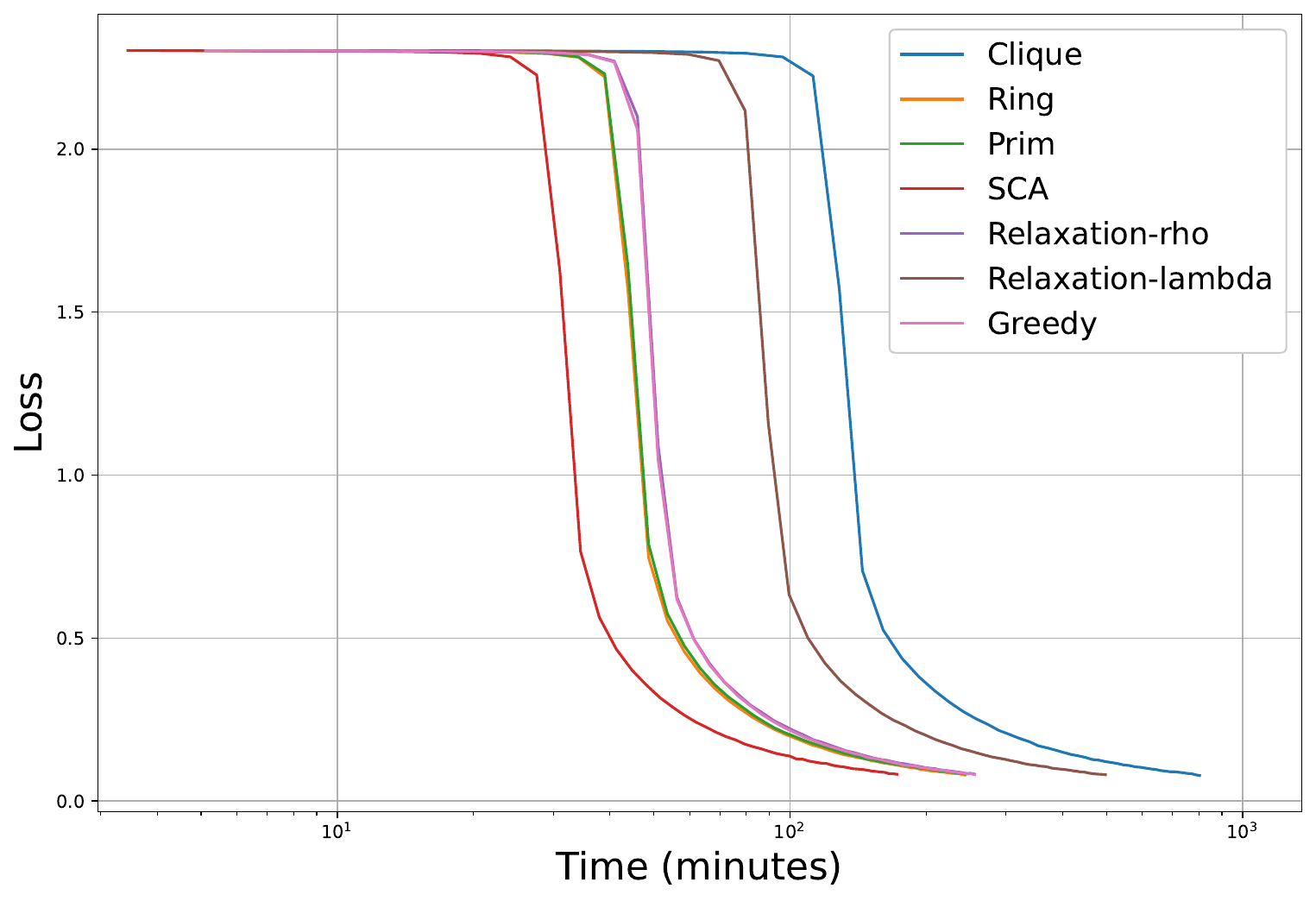}}
\vspace{-.1em}
\end{minipage}
\begin{minipage}{.495\linewidth}
\centerline{
\includegraphics[width=1\linewidth]{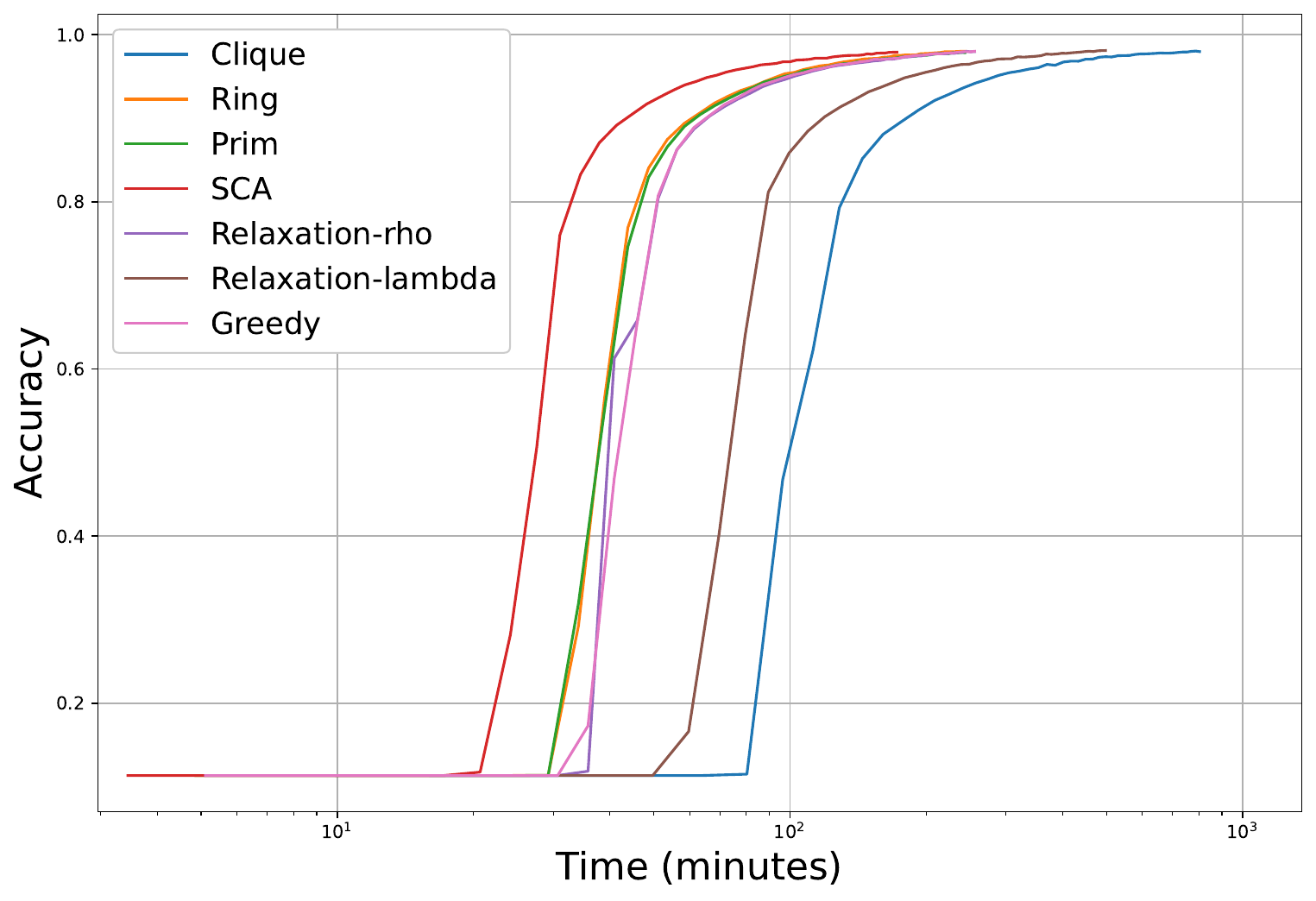}}
\vspace{-.1em}
\end{minipage}
\vspace{-1.25em}
\caption{MNIST over Roofnet without inference errors (second row: time without overlay routing; third row: time with overlay routing). 
} \label{fig:roofnet_MNIST}
\vspace{-.5em}
\end{figure}

\begin{figure}[t!]
\begin{minipage}{.495\linewidth}
\centerline{
\includegraphics[width=1\linewidth]{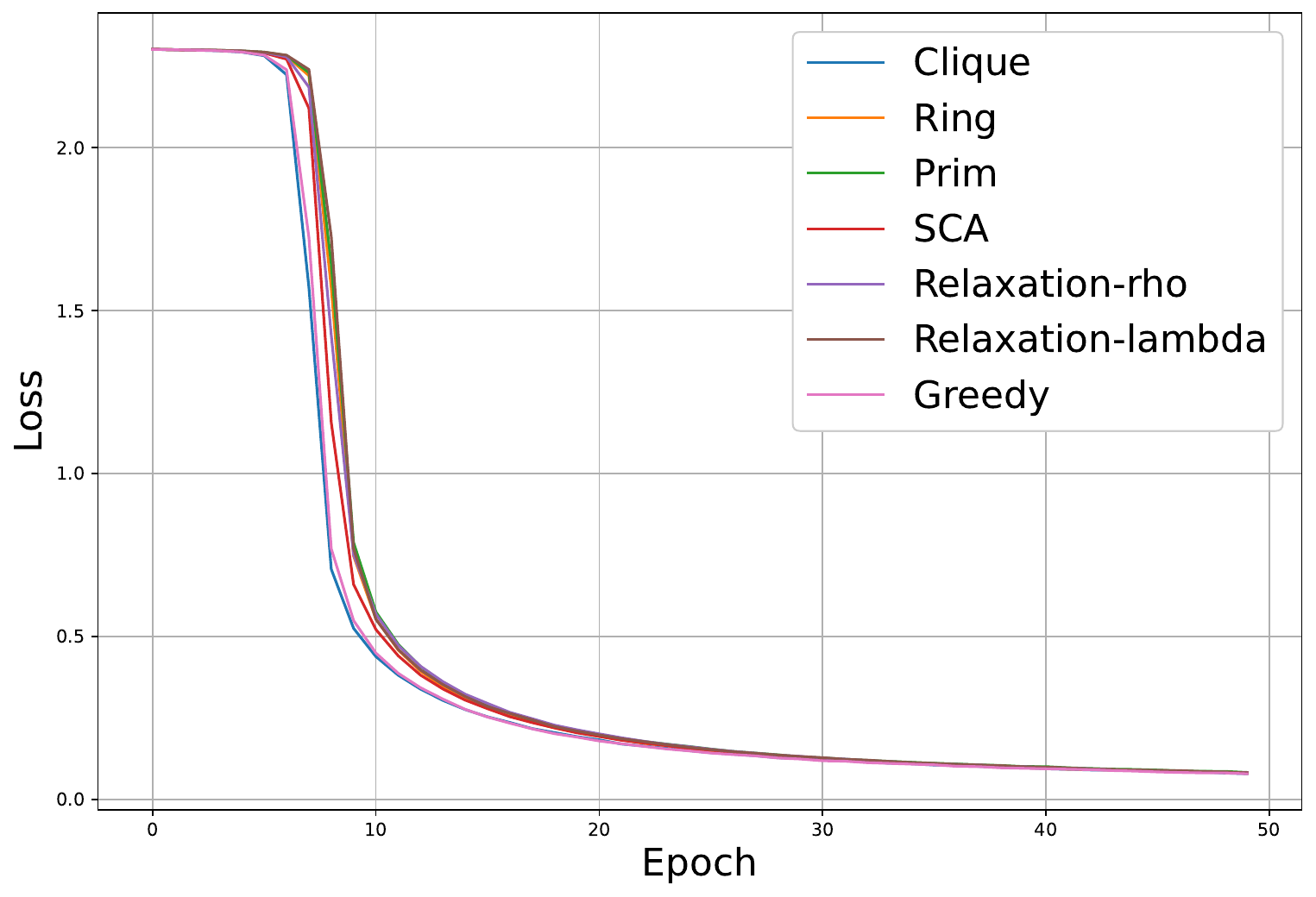}}
\vspace{-.1em}
\end{minipage}
\begin{minipage}{.495\linewidth}
\centerline{
\includegraphics[width=1\linewidth]{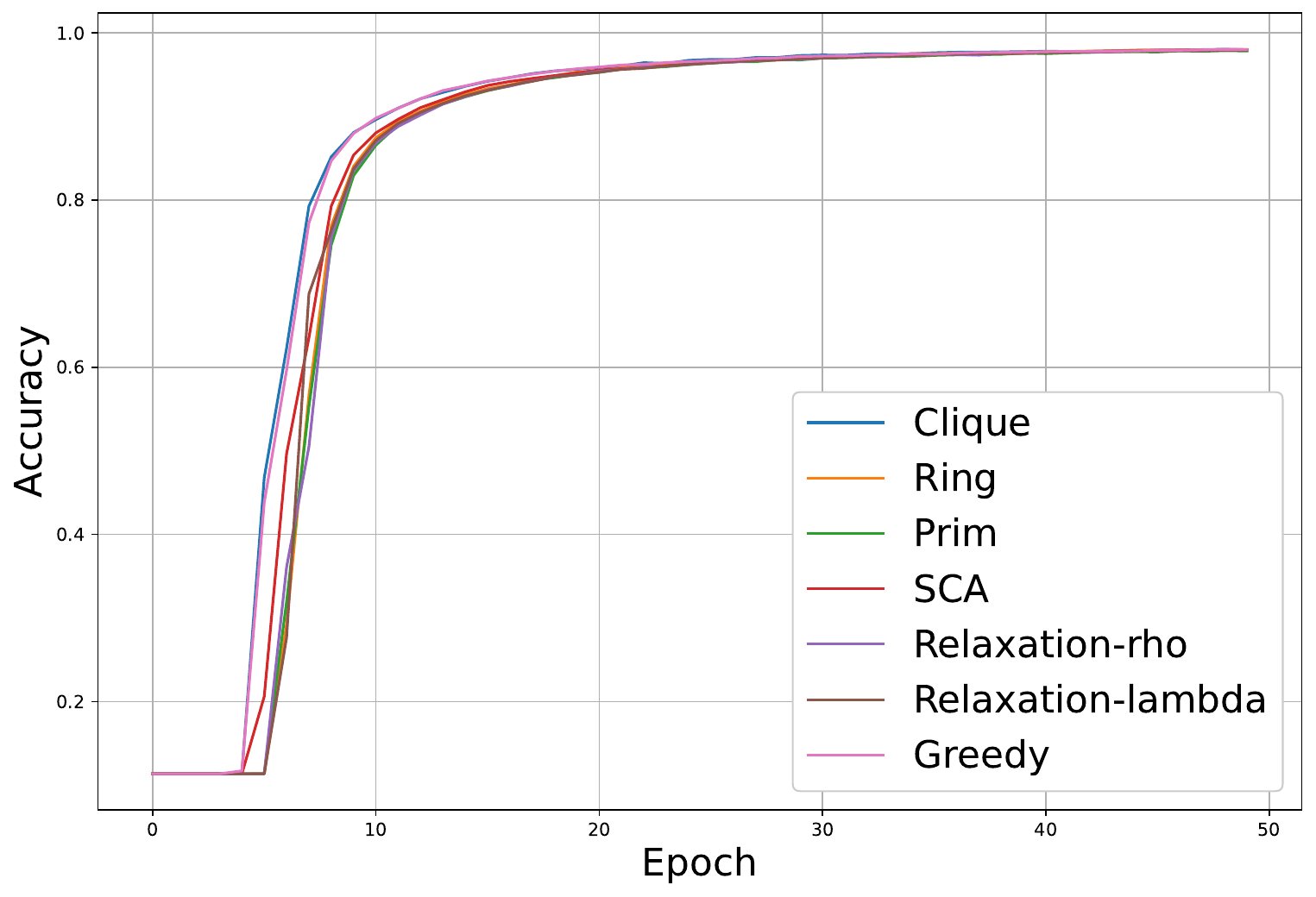}}
\vspace{-.1em}
\end{minipage}
\begin{minipage}{.495\linewidth}
\centerline{
\includegraphics[width=1\linewidth]{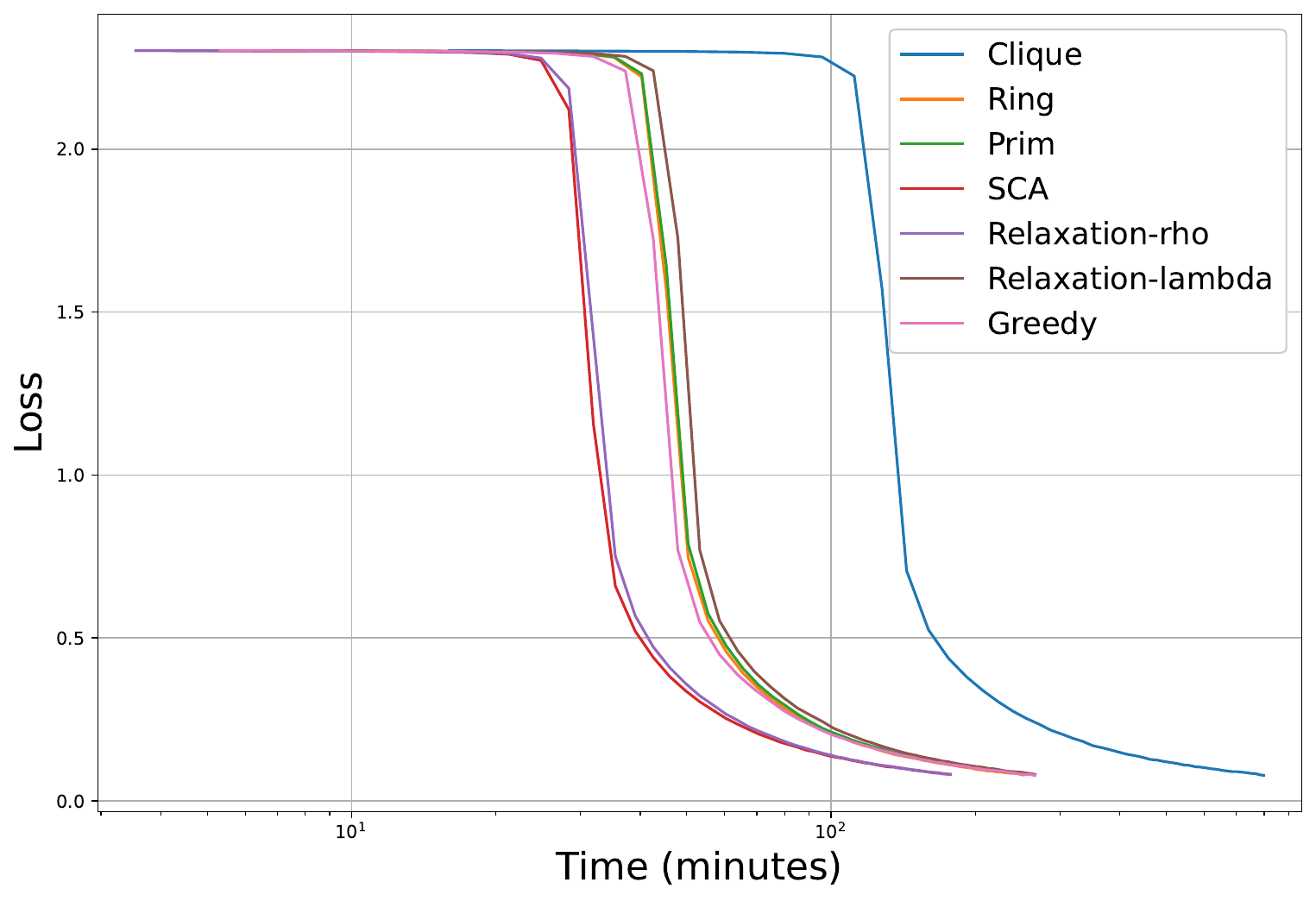}}
\vspace{-.1em}
\end{minipage}
\begin{minipage}{.495\linewidth}
\centerline{
\includegraphics[width=1\linewidth]{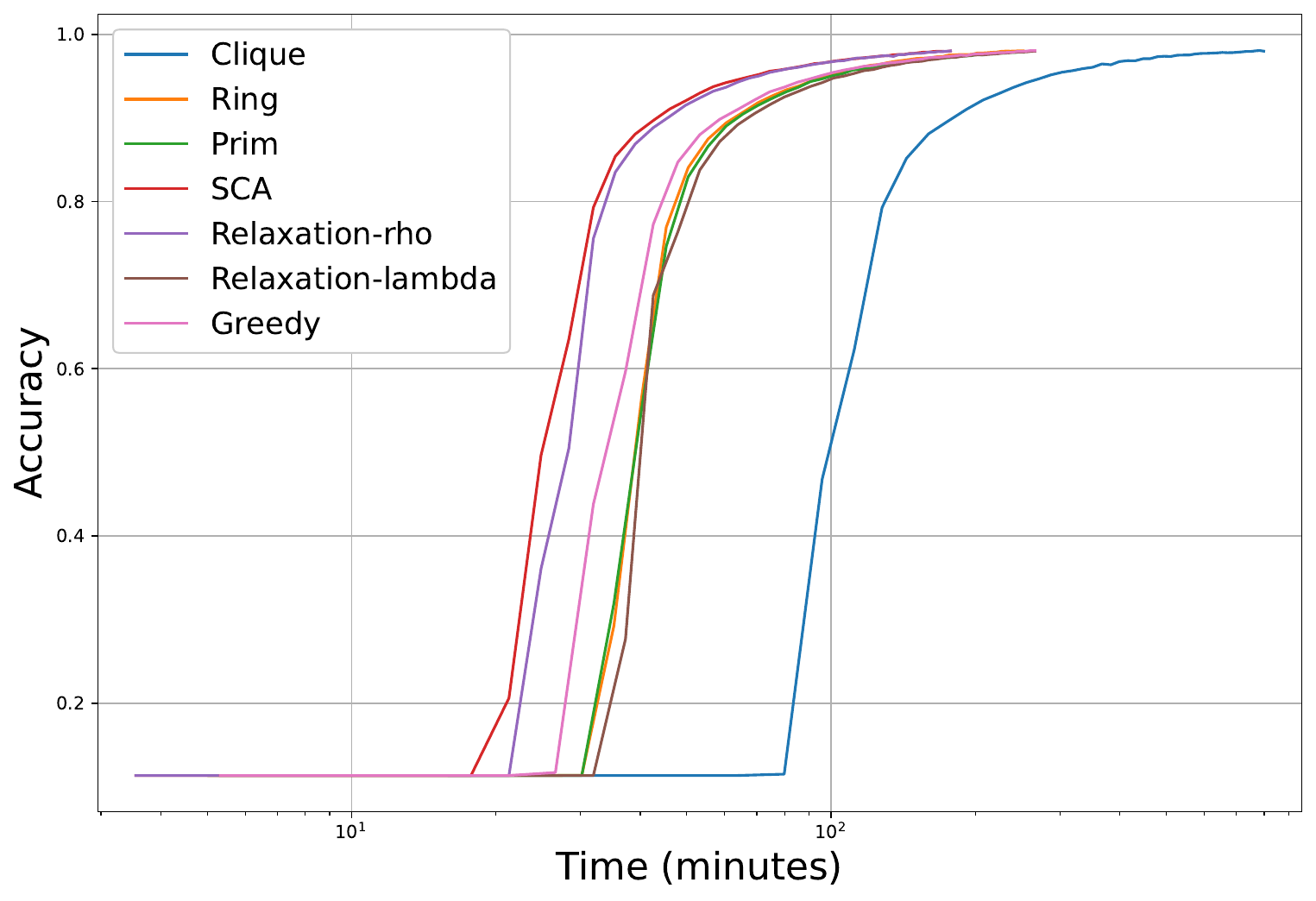}}
\vspace{-.1em}
\end{minipage}
\begin{minipage}{.495\linewidth}
\centerline{
\includegraphics[width=1\linewidth]{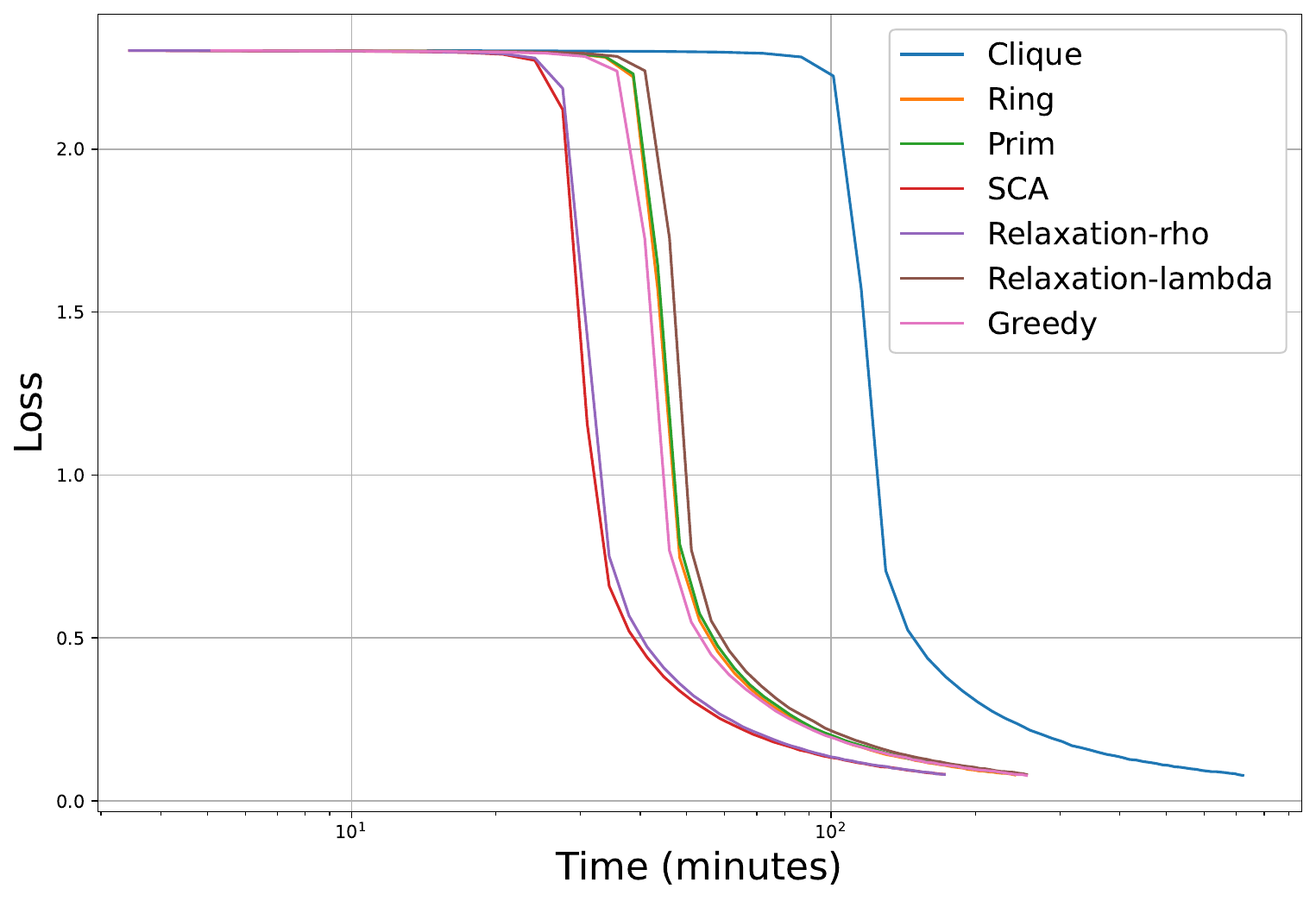}}
\vspace{-.1em}
\end{minipage}
\begin{minipage}{.495\linewidth}
\centerline{
\includegraphics[width=1\linewidth]{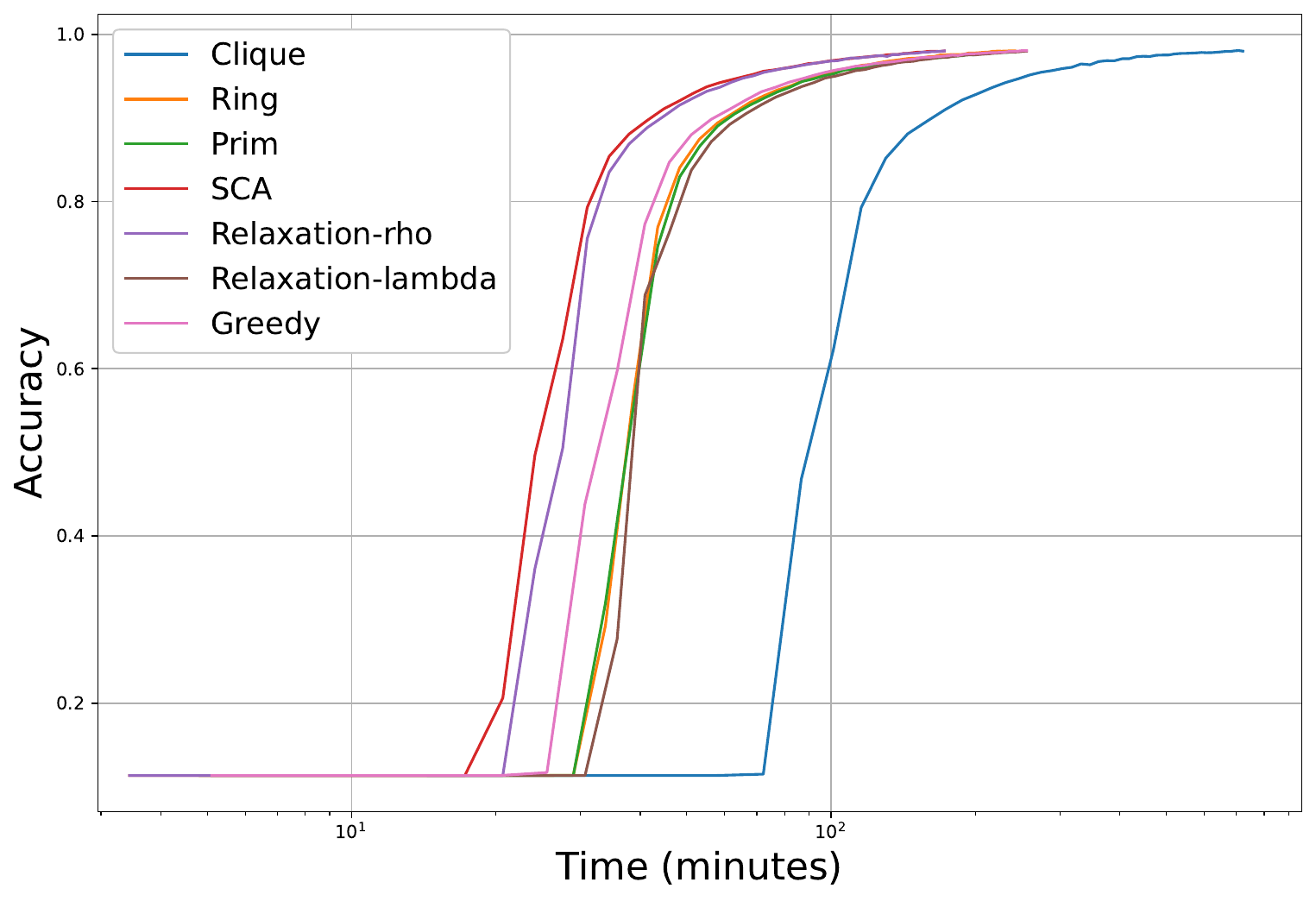}}
\vspace{-.1em}
\end{minipage}
\vspace{-1.25em}
\caption{MNIST over Roofnet with inference errors (second row: time without overlay routing; third row: time with overlay routing). 
} \label{fig:Roofnet_MNIST_inference}
\vspace{-.5em}
\end{figure}

To check the generalizability of our observations wrt the learning task, we repeat the above tests on MNIST based on the same network topology. The results in Fig.~\ref{fig:roofnet_MNIST}--\ref{fig:Roofnet_MNIST_inference} show the following ranking of the algorithms: `SCA' performs the best, `Clique' performs the worst, and the rest (`Prim', `Ring', `Relaxation-$\rho$', `Greedy', and `Relaxation-$\lambda$') fall in between. The inference errors may change the exact performance of each design, but their comparison remains largely the same. 
Compared with CIFAR-10, we see that although the small size of MNIST causes the gaps between algorithms to shrink, their comparison remains the same. \looseness=0

\subsection{Simulation Results on IAB Network}

To further validate our observations in different network scenarios, we repeat the previous tests on the IAB network, as shown in Fig.~\ref{fig:IAB_CIFAR10}--\ref{fig:IAB_MNIST_inference}. The results exhibit similar trends as those on Roofnet (Fig.~\ref{fig:roofnet_no_routing}--\ref{fig:Roofnet_MNIST_inference}), in that the sparse topologies significantly reduce the time in achieving the same level of convergence as the clique, and overlay routing brings more improvement to dense topologies like the clique. Meanwhile, there is also a notable difference: except for the clique, all the topologies perform similarly on the IAB network. Further examination shows that this is because the IAB network is smaller, with fewer hops between the learning agents and less link sharing, and thus leaves less room for improvement. Nevertheless, the proposed algorithm `SCA' is still among the best-performing solutions. This result demonstrates the generalizability of our previous observations. 

\begin{figure}[t!]
\begin{minipage}{.495\linewidth}
\centerline{
\includegraphics[width=1\linewidth]{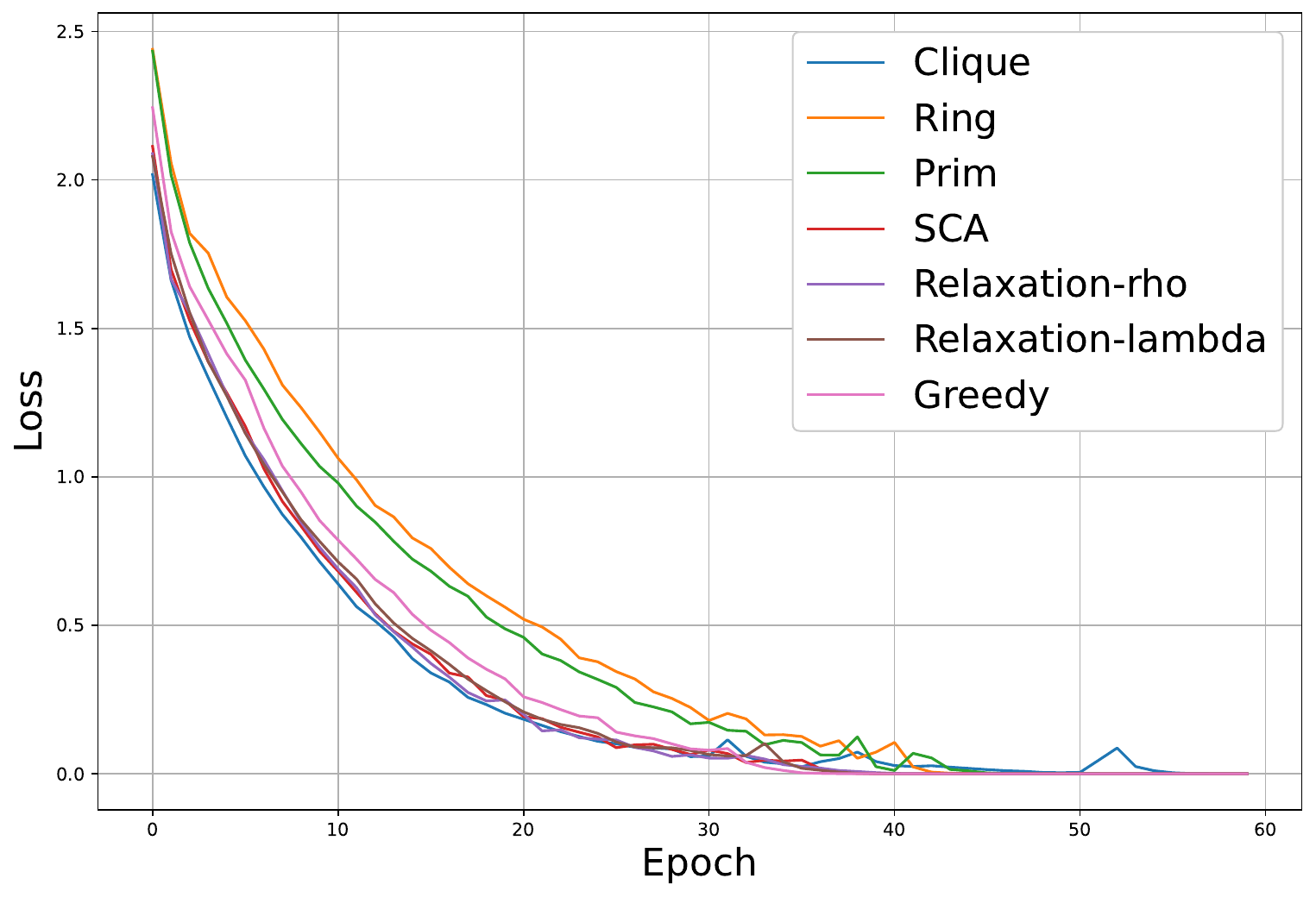}}
\vspace{-.1em}
\end{minipage}
\begin{minipage}{.495\linewidth}
\centerline{
\includegraphics[width=1\linewidth]{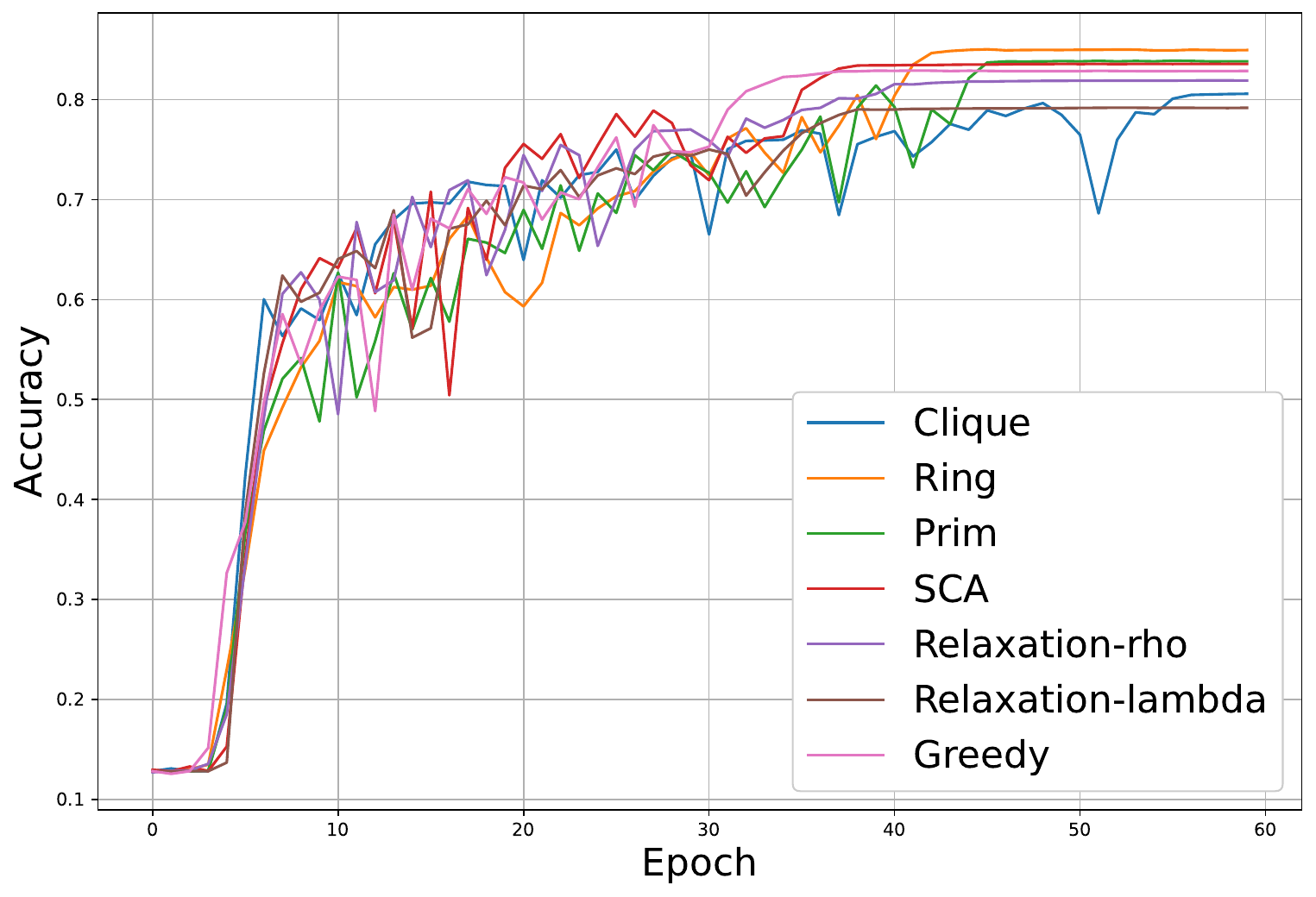}}
\vspace{-.1em}
\end{minipage}
\begin{minipage}{.495\linewidth}
\centerline{
\includegraphics[width=1\linewidth]{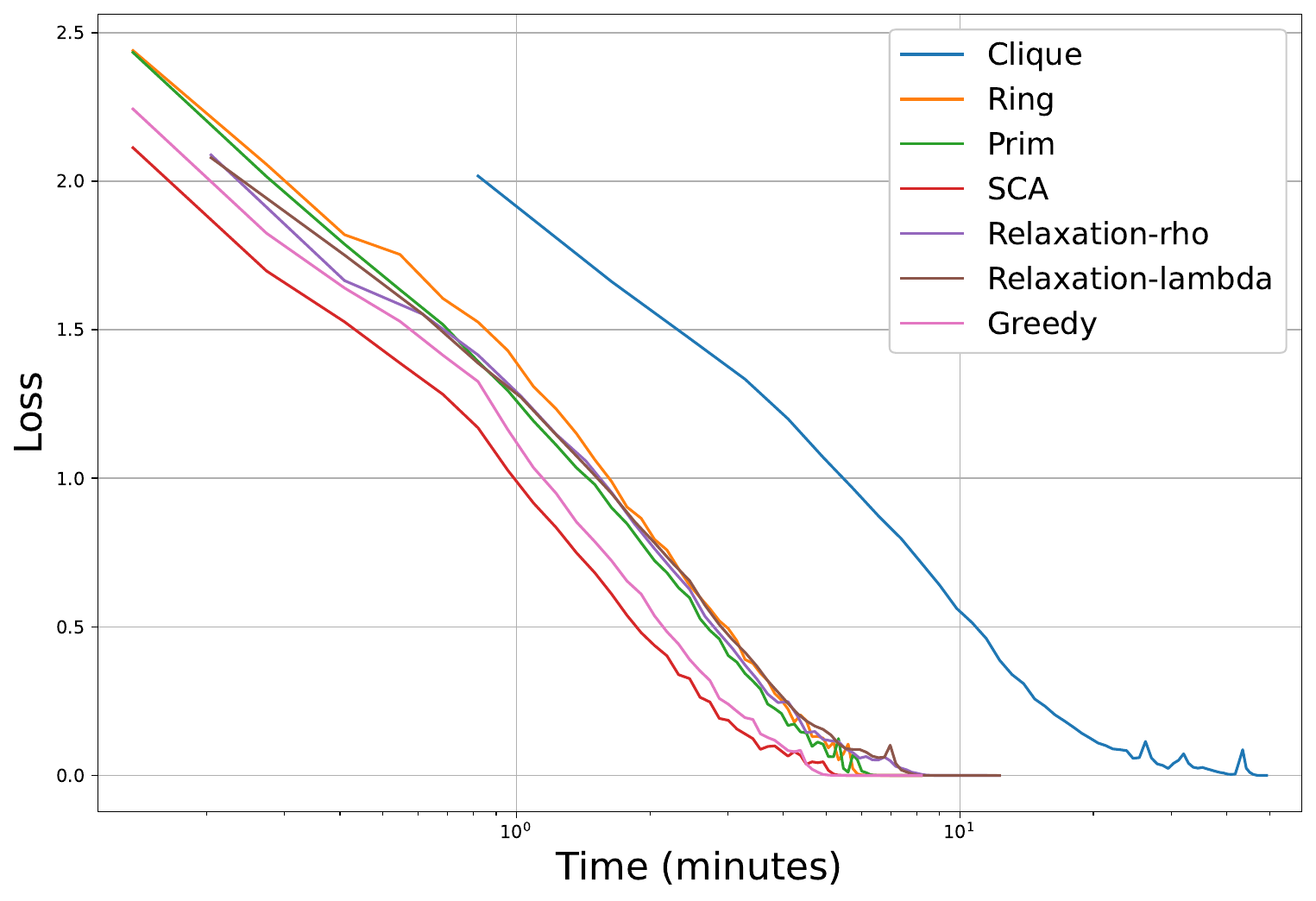}}
\vspace{-.1em}
\end{minipage}
\begin{minipage}{.495\linewidth}
\centerline{
\includegraphics[width=1\linewidth]{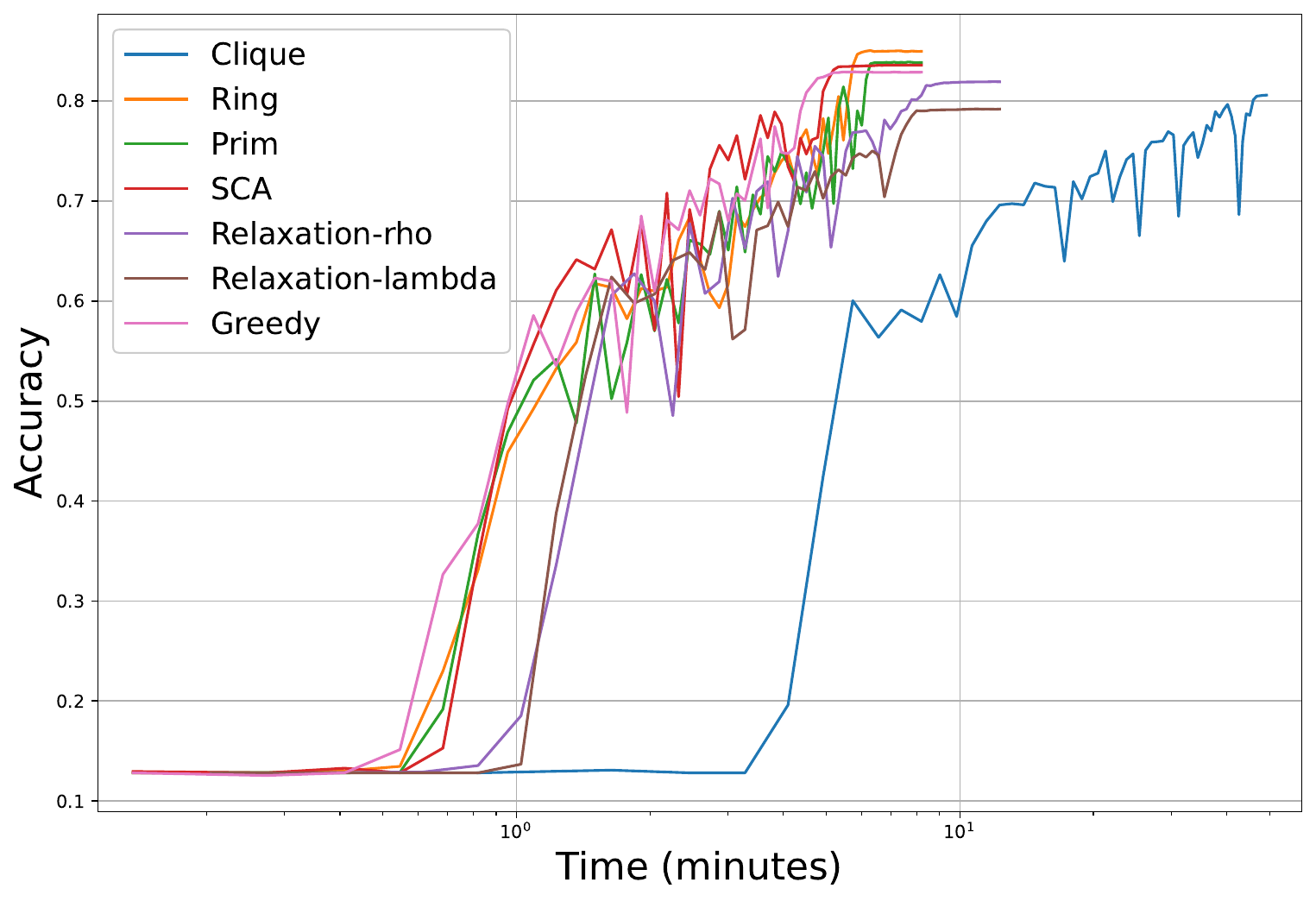}}
\vspace{-.1em}
\end{minipage}
\begin{minipage}{.495\linewidth}
\centerline{
\includegraphics[width=1\linewidth]{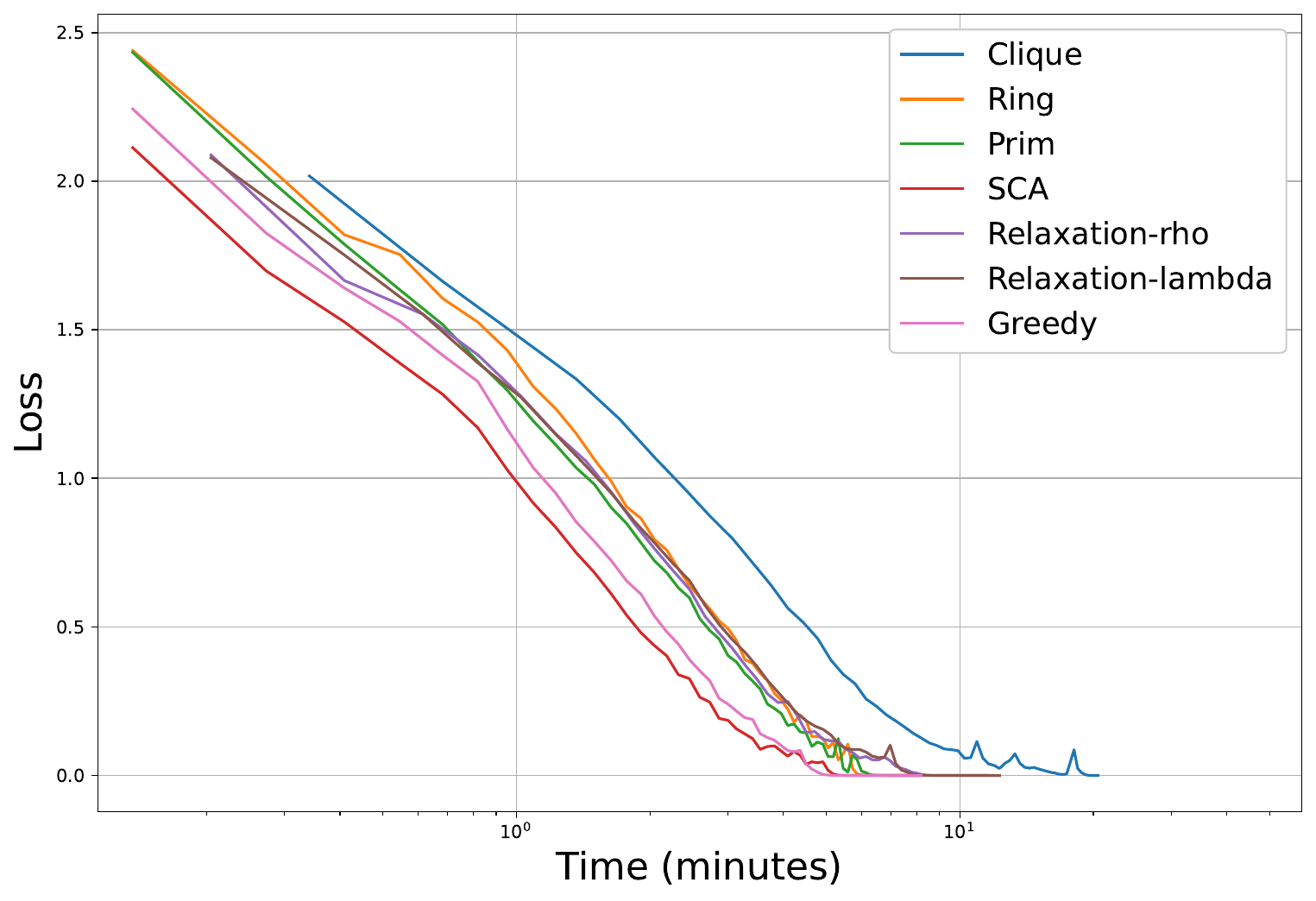}}
\vspace{-.1em}
\end{minipage}
\begin{minipage}{.495\linewidth}
\centerline{
\includegraphics[width=1\linewidth]{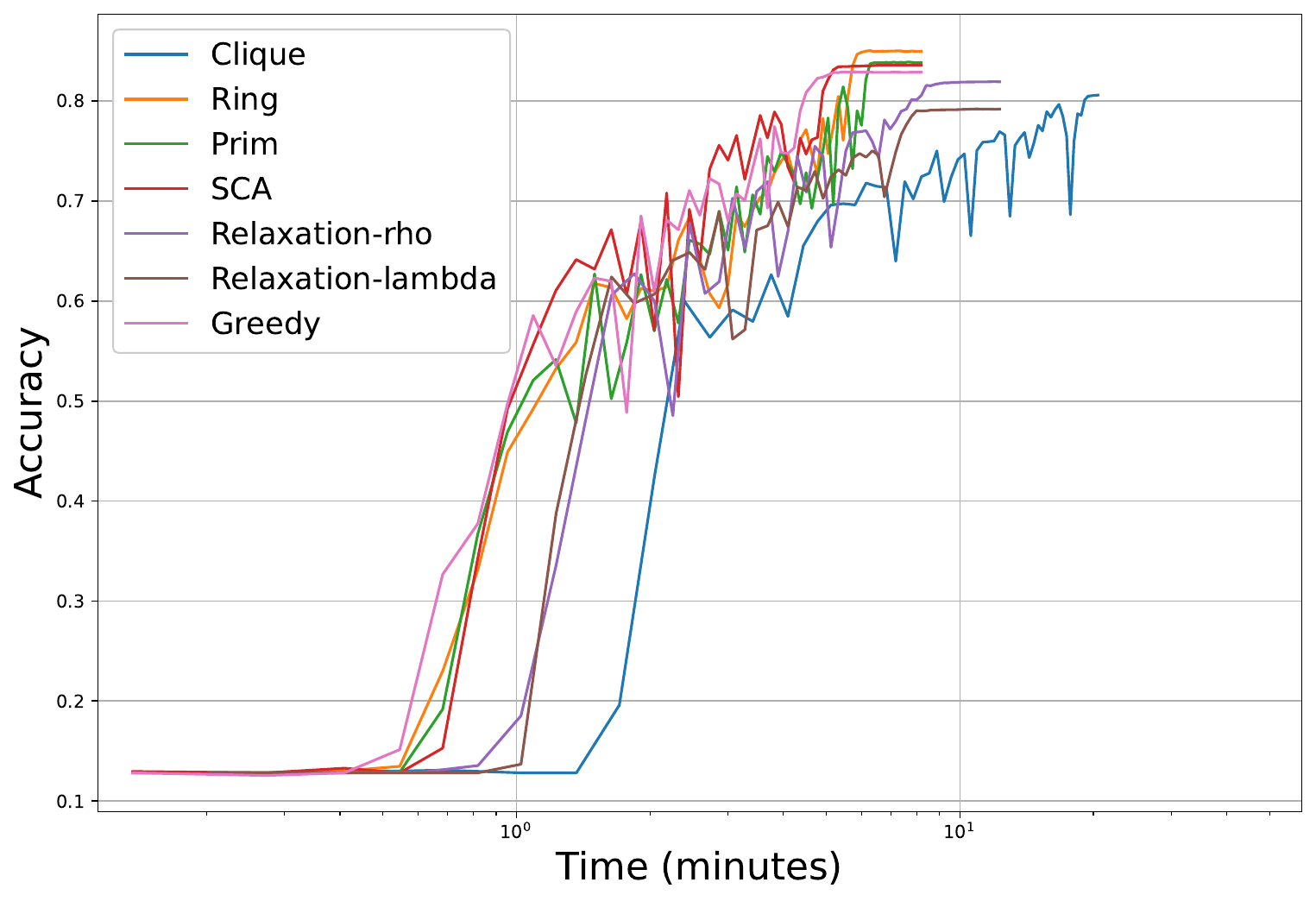}}
\vspace{-.1em}
\end{minipage}
\vspace{-1.25em}
\caption{CIFAR-10 over IAB (second row: time without overlay routing; third row: time with overlay routing). 
} \label{fig:IAB_CIFAR10}
\end{figure}

\begin{figure}[t!]
\begin{minipage}{.495\linewidth}
\centerline{
\includegraphics[width=1\linewidth]{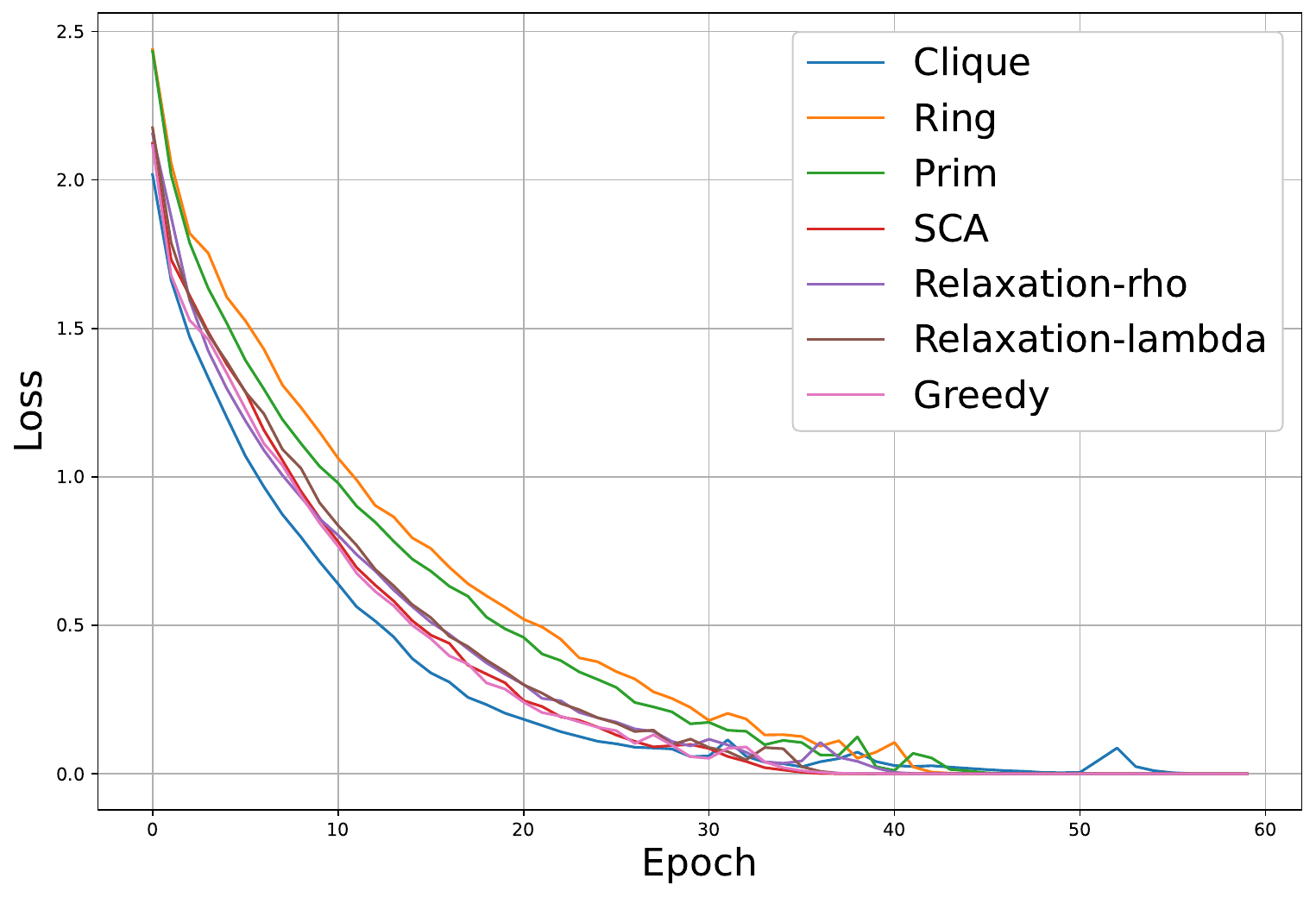}}
\vspace{-.1em}
\end{minipage}
\begin{minipage}{.495\linewidth}
\centerline{
\includegraphics[width=1\linewidth]{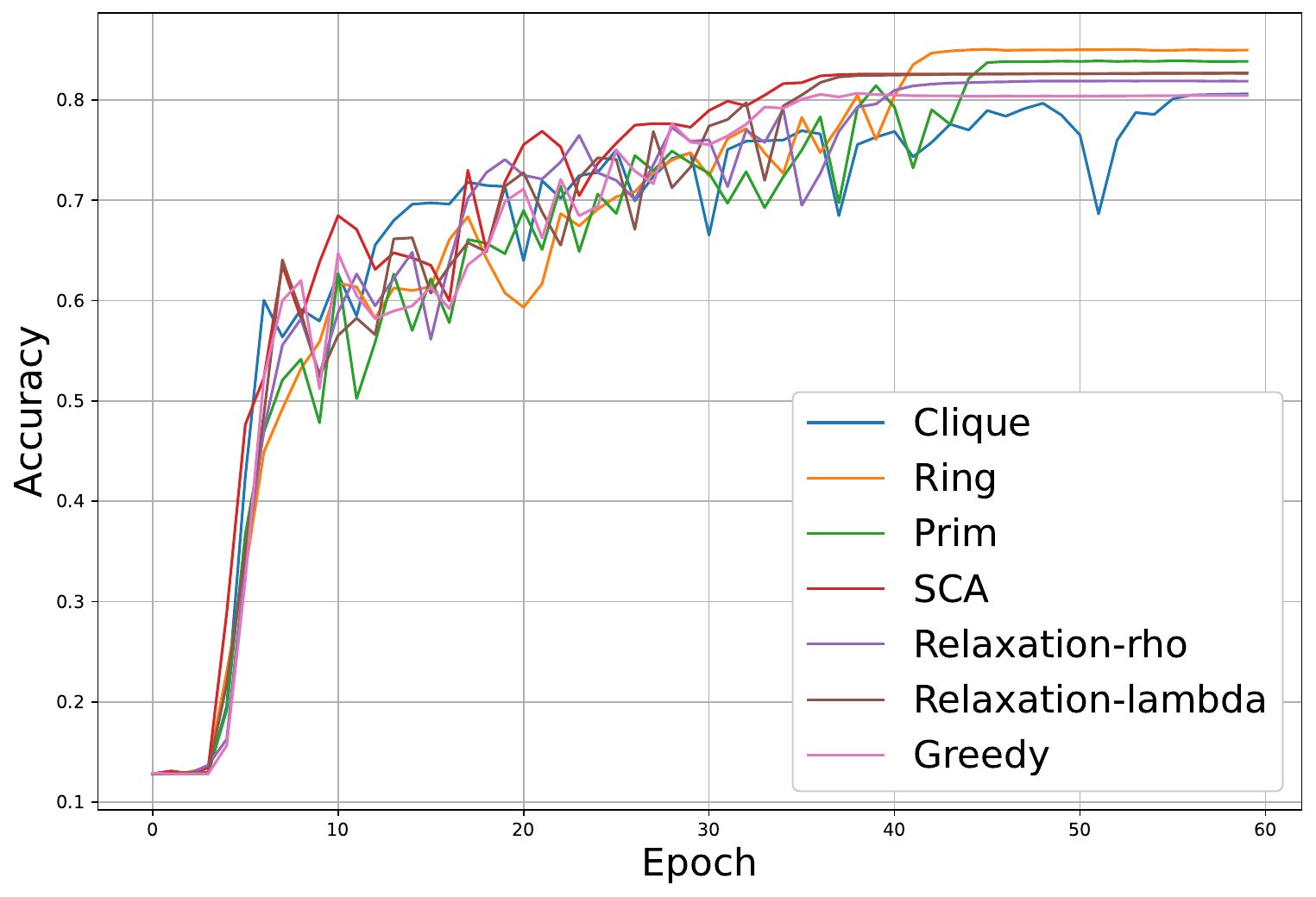}}
\vspace{-.1em}
\end{minipage}
\begin{minipage}{.495\linewidth}
\centerline{
\includegraphics[width=1\linewidth]{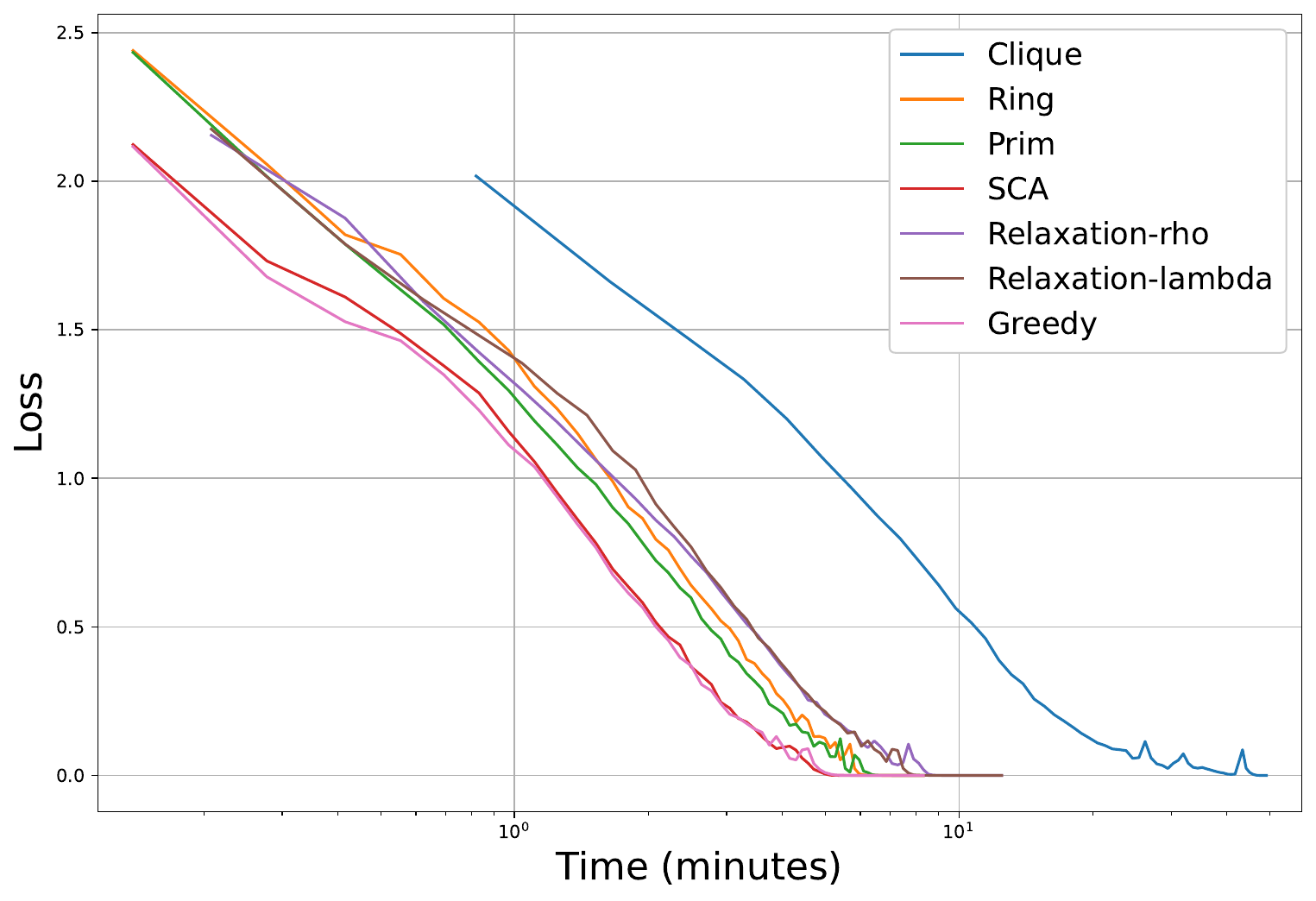}}
\vspace{-.1em}
\end{minipage}
\begin{minipage}{.495\linewidth}
\centerline{
\includegraphics[width=1\linewidth]{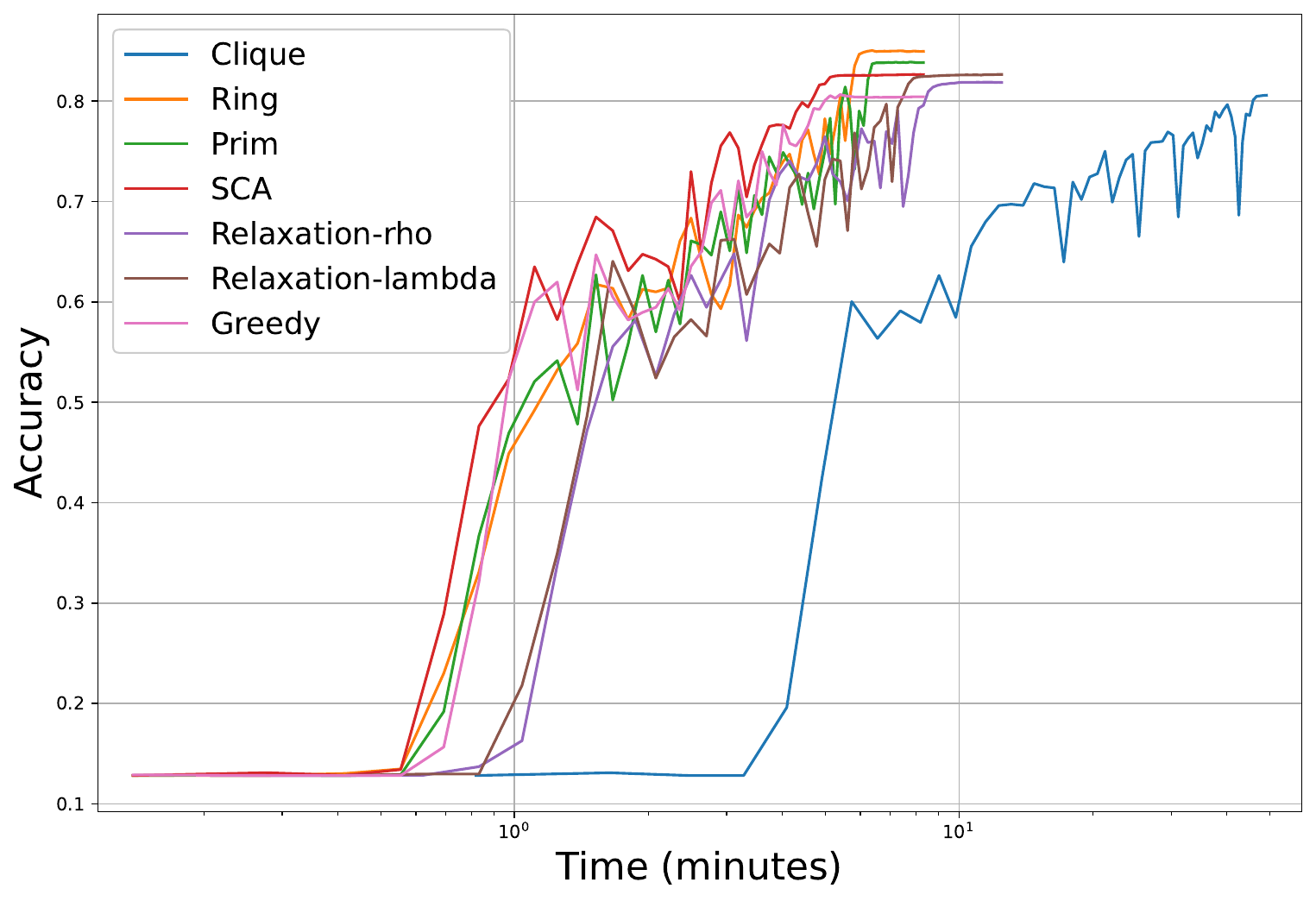}}
\vspace{-.1em}
\end{minipage}
\begin{minipage}{.495\linewidth}
\centerline{
\includegraphics[width=1\linewidth]{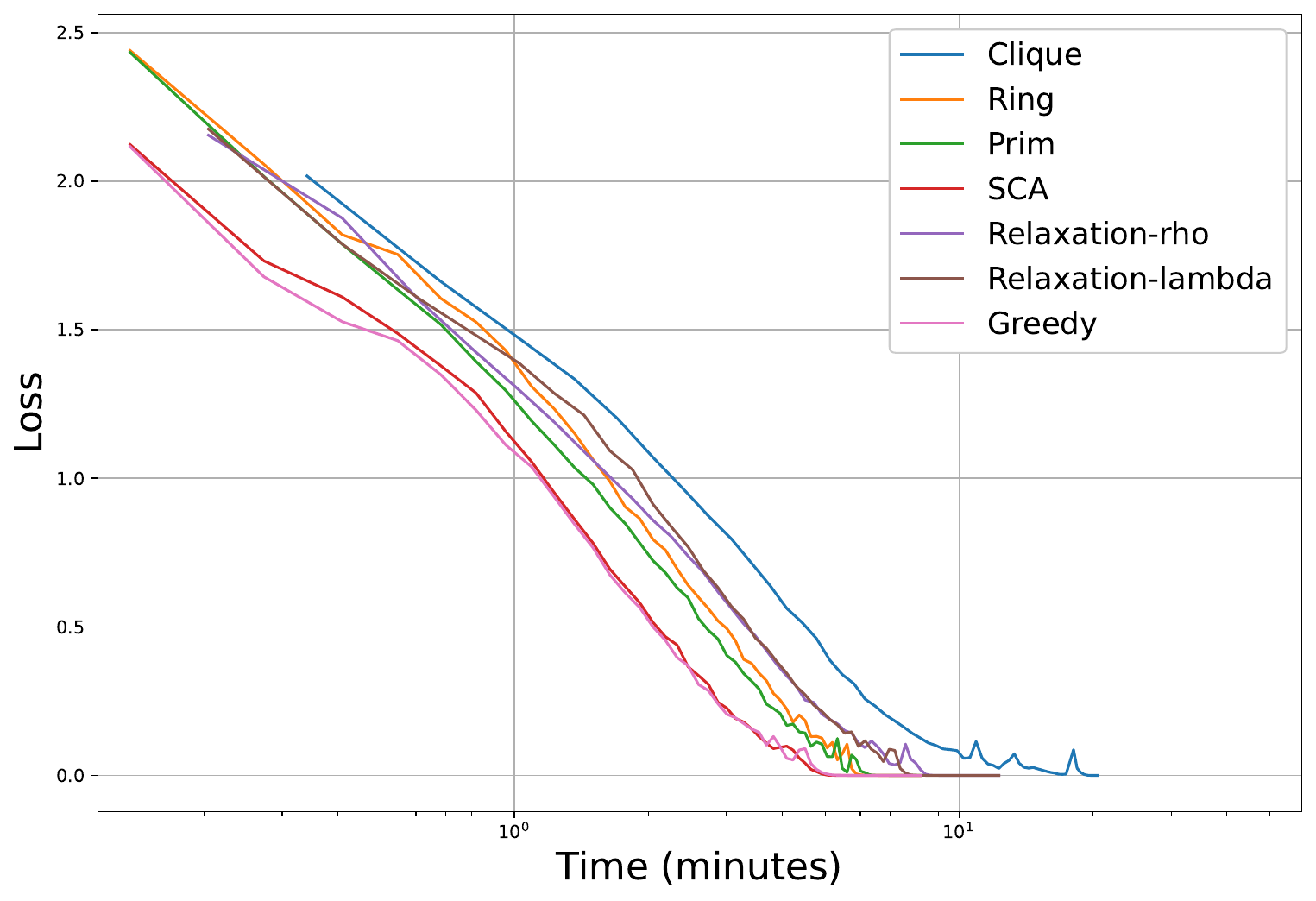}}
\vspace{-.1em}
\end{minipage}
\begin{minipage}{.495\linewidth}
\centerline{
\includegraphics[width=1\linewidth]{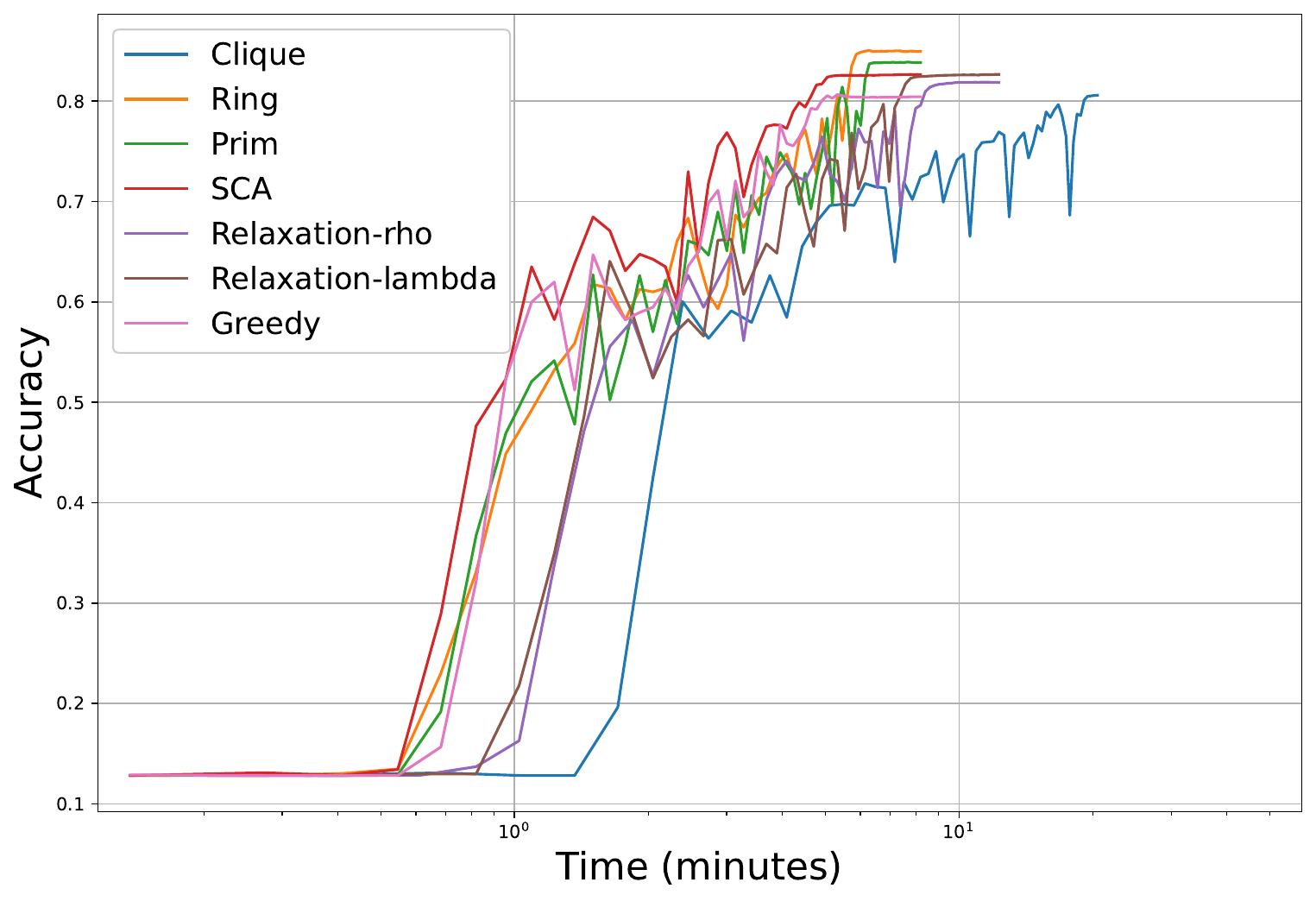}}
\vspace{-.1em}
\end{minipage}
\vspace{-1.25em}
\caption{CIFAR-10 over IAB with inference errors (second row: time without overlay routing; third row: time with overlay routing). 
} \label{fig:IAB_CIFAR10_inference}
\end{figure}

\begin{figure}[t!]
\begin{minipage}{.495\linewidth}
\centerline{
\includegraphics[width=1\linewidth]{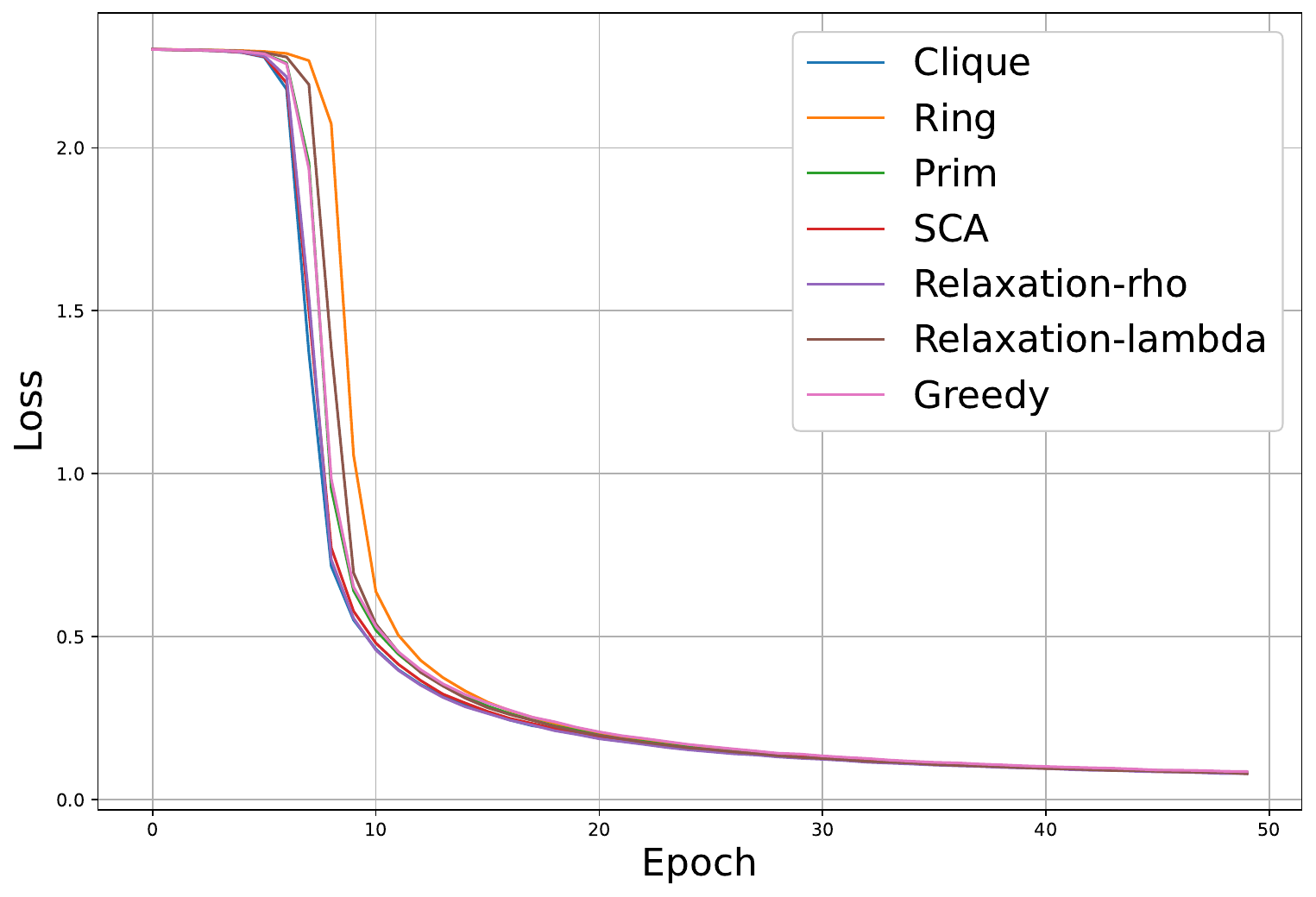}}
\vspace{-.1em}
\end{minipage}
\begin{minipage}{.495\linewidth}
\centerline{
\includegraphics[width=1\linewidth]{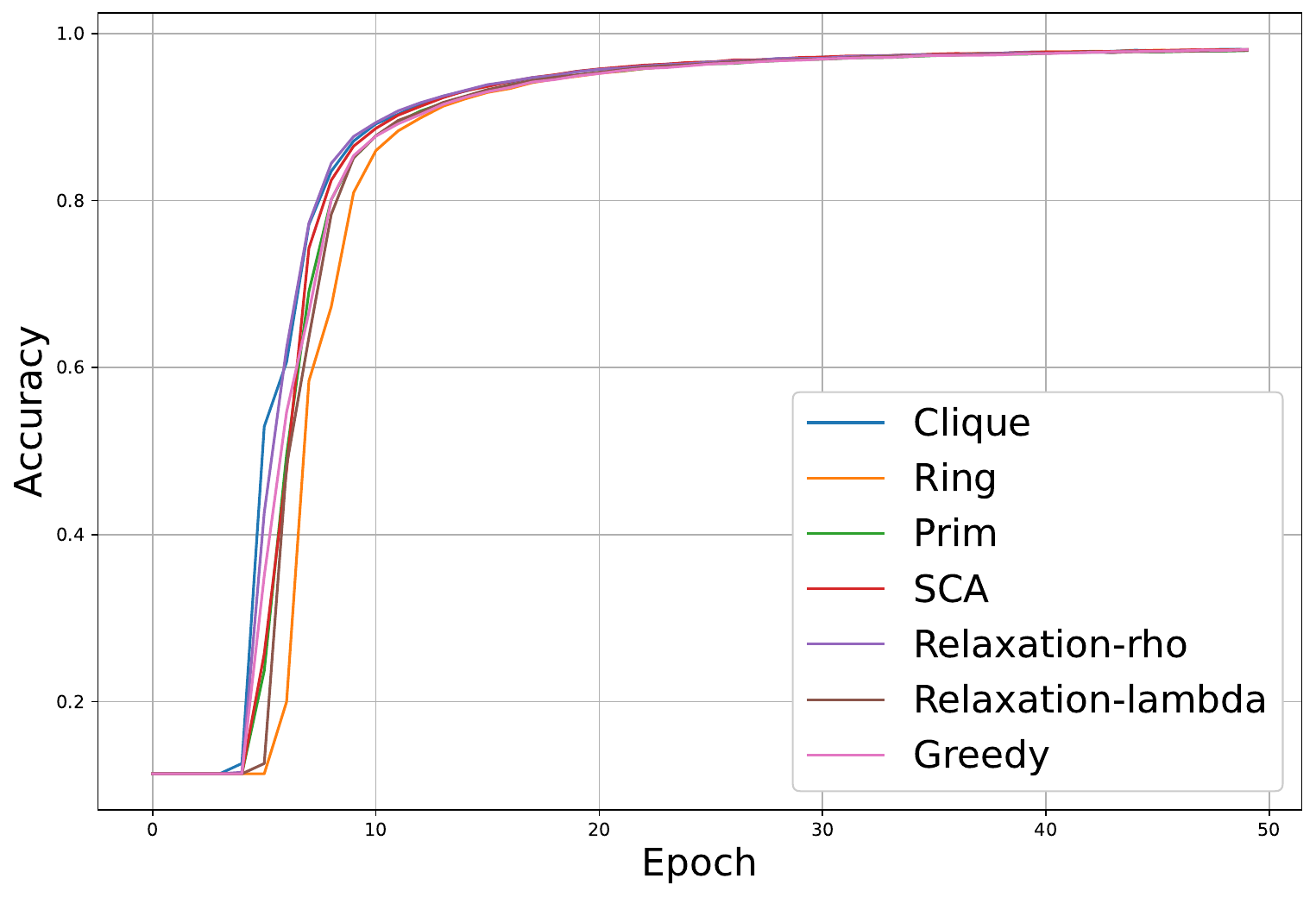}}
\vspace{-.1em}
\end{minipage}
\begin{minipage}{.495\linewidth}
\centerline{
\includegraphics[width=1\linewidth]{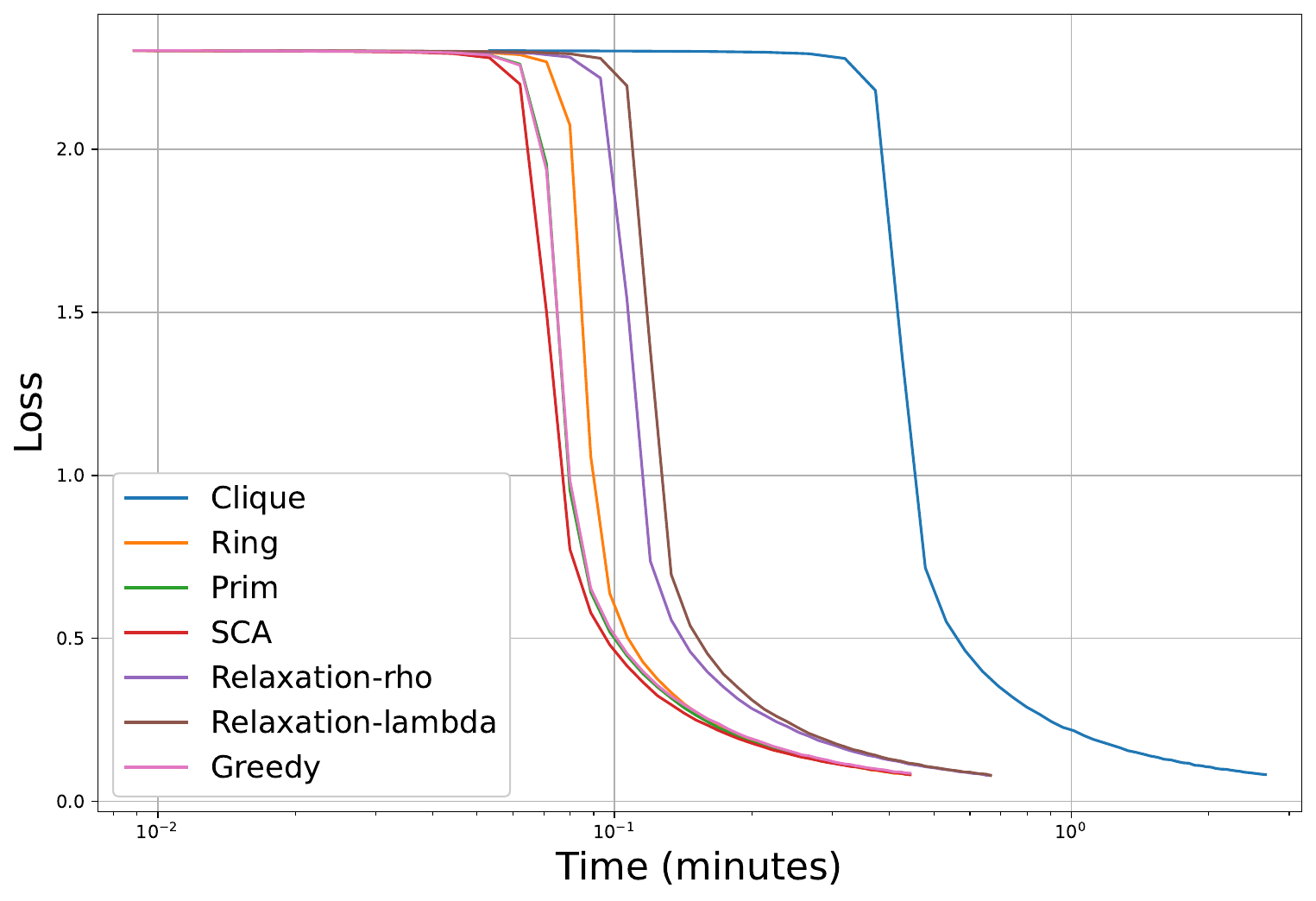}}
\vspace{-.1em}
\end{minipage}
\begin{minipage}{.495\linewidth}
\centerline{
\includegraphics[width=1\linewidth]{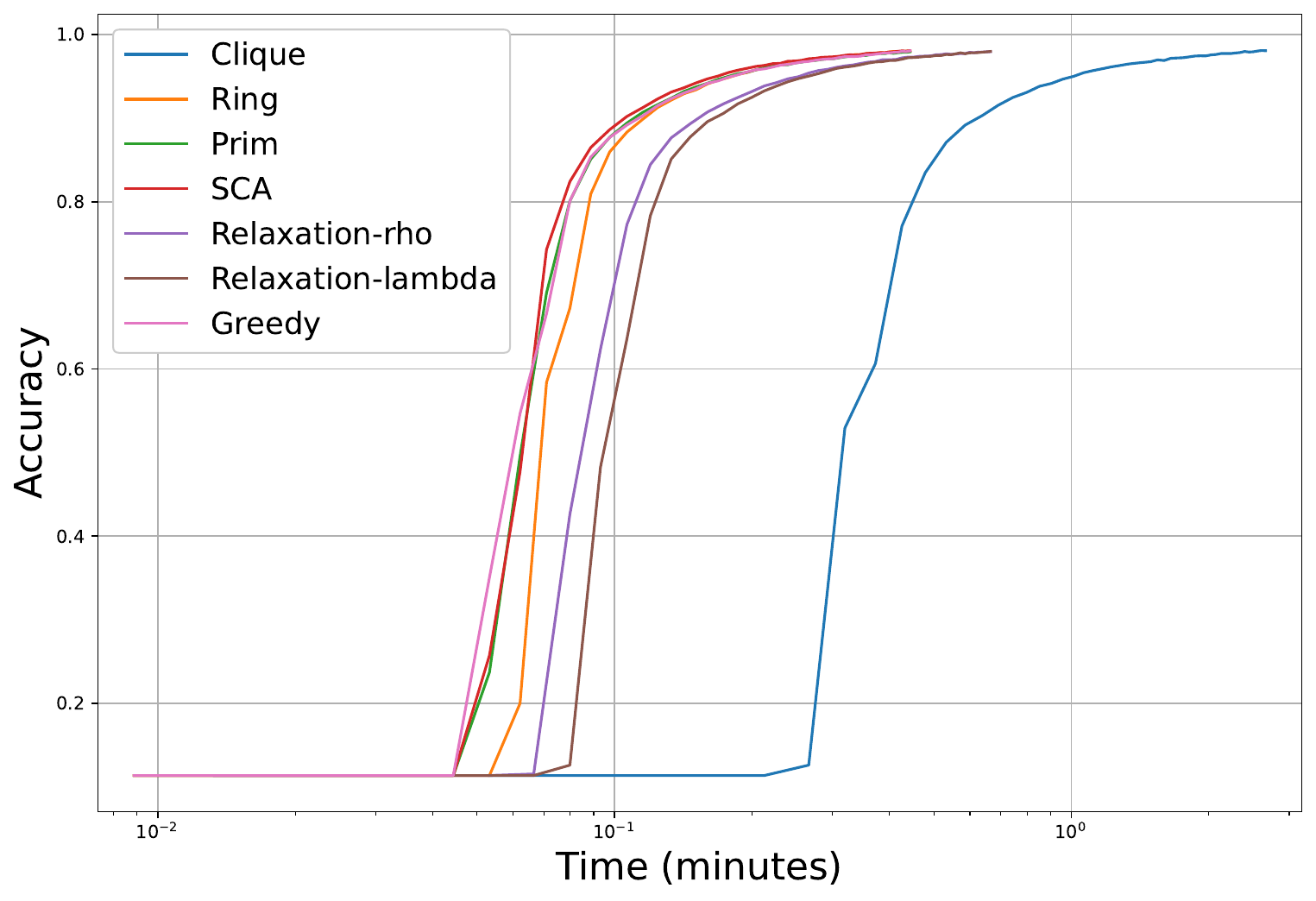}}
\vspace{-.1em}
\end{minipage}
\begin{minipage}{.495\linewidth}
\centerline{
\includegraphics[width=1\linewidth]{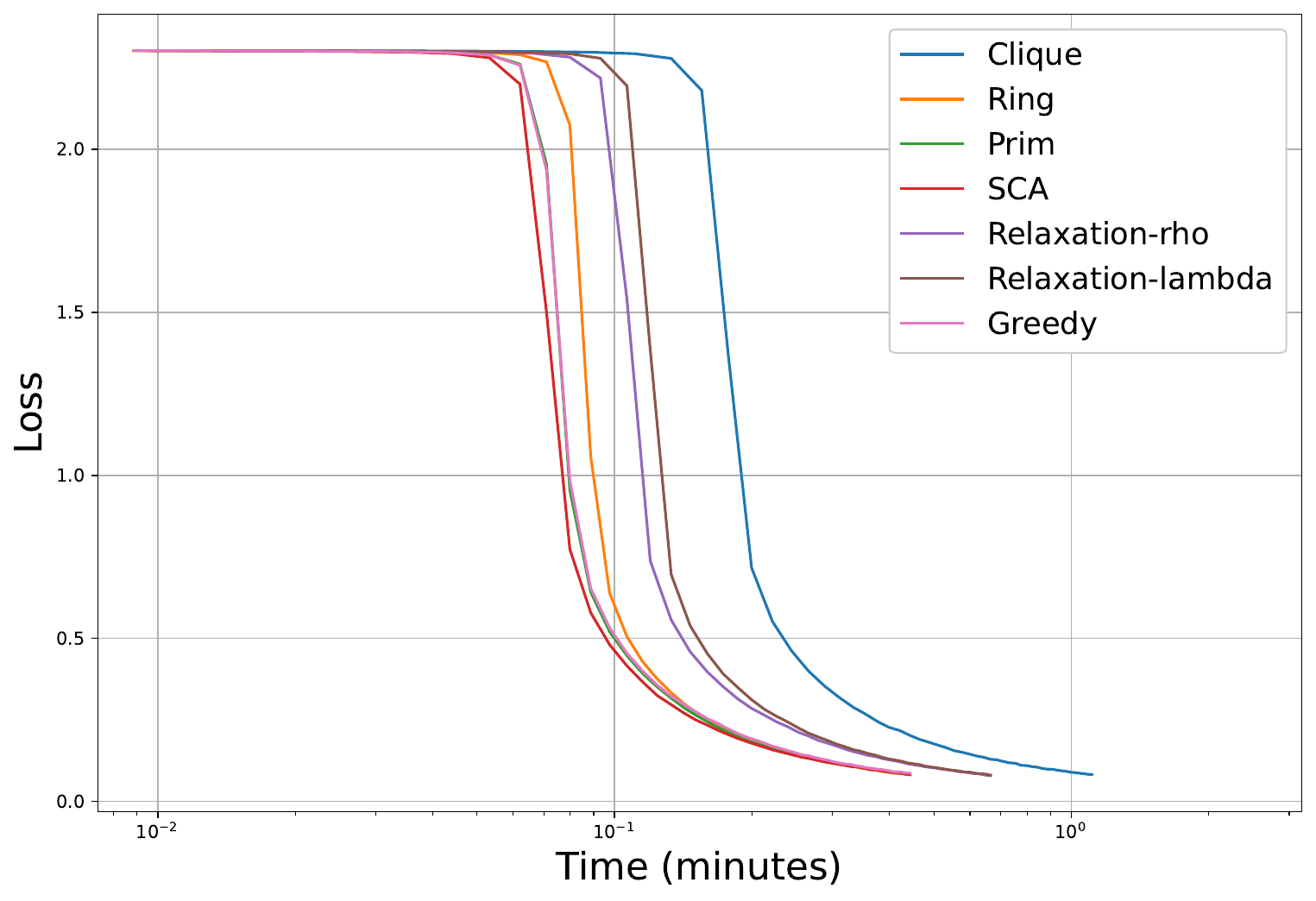}}
\vspace{-.1em}
\end{minipage}
\begin{minipage}{.495\linewidth}
\centerline{
\includegraphics[width=1\linewidth]{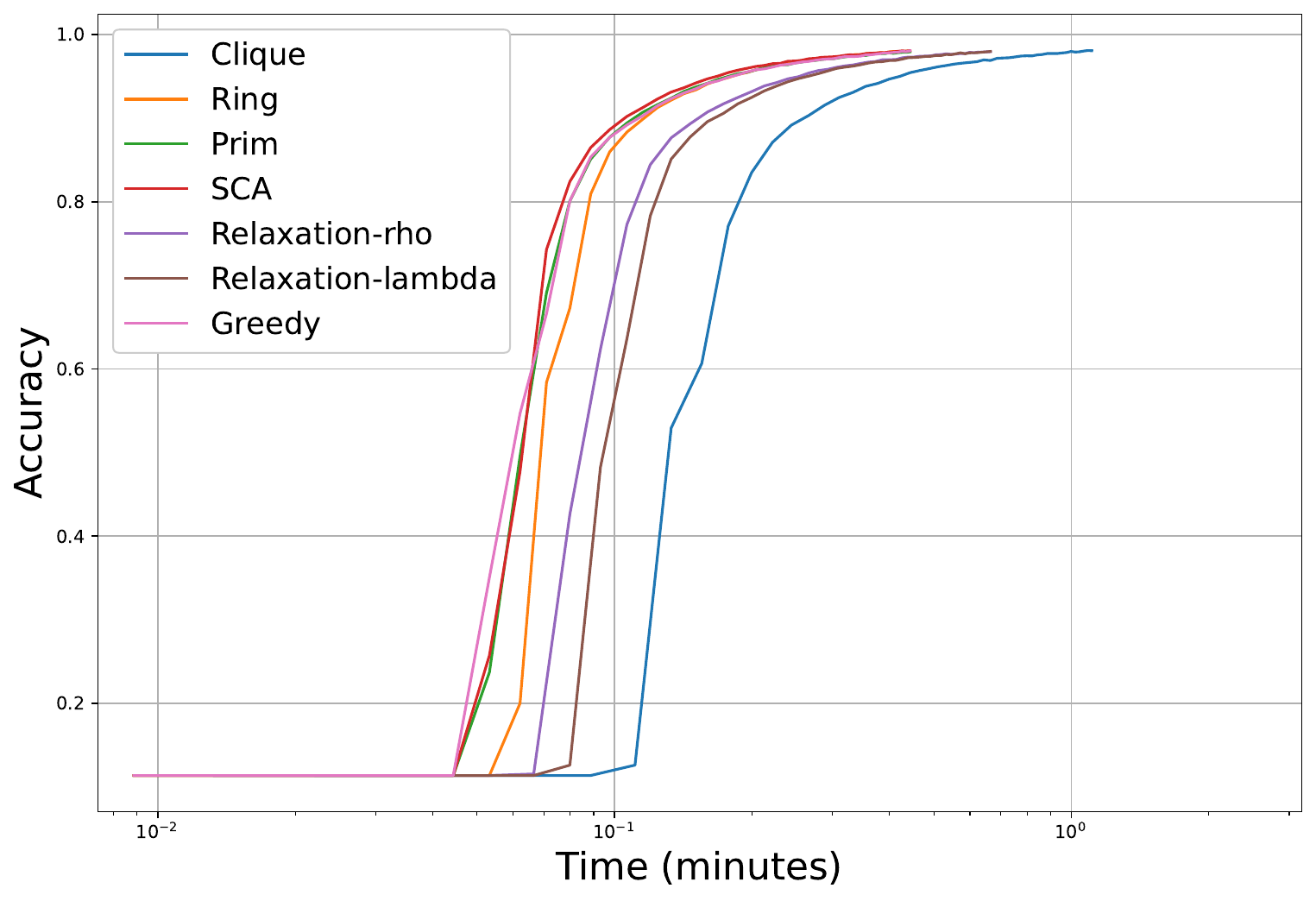}}
\vspace{-.1em}
\end{minipage}
\vspace{-1.25em}
\caption{MNIST over IAB (second row: time without overlay routing; third row: time with overlay routing). 
} \label{fig:IAB_MNIST}
\vspace{-.5em}
\end{figure}

\begin{figure}[t!]
\begin{minipage}{.495\linewidth}
\centerline{
\includegraphics[width=1\linewidth]{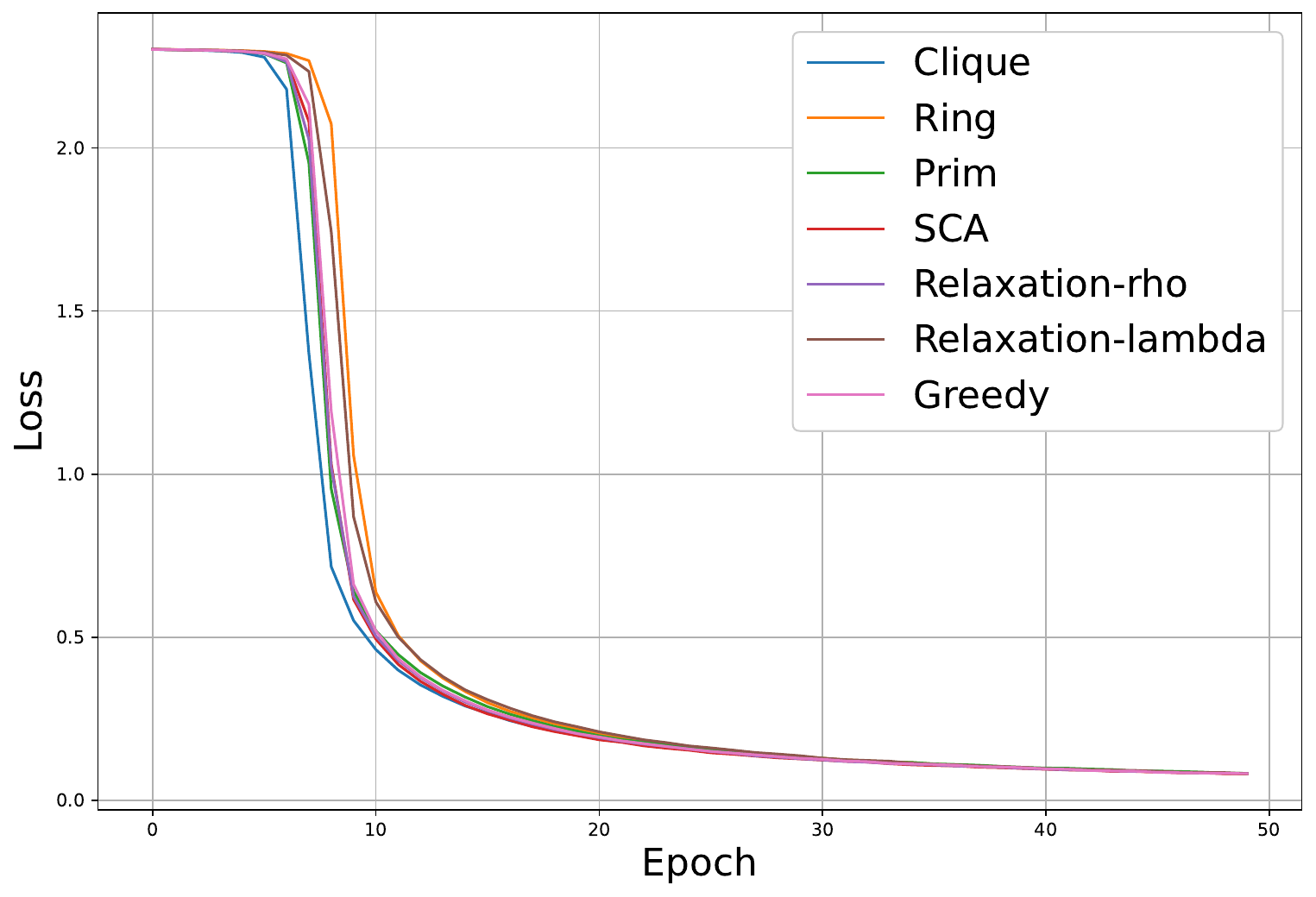}}
\vspace{-.1em}
\end{minipage}
\begin{minipage}{.495\linewidth}
\centerline{
\includegraphics[width=1\linewidth]{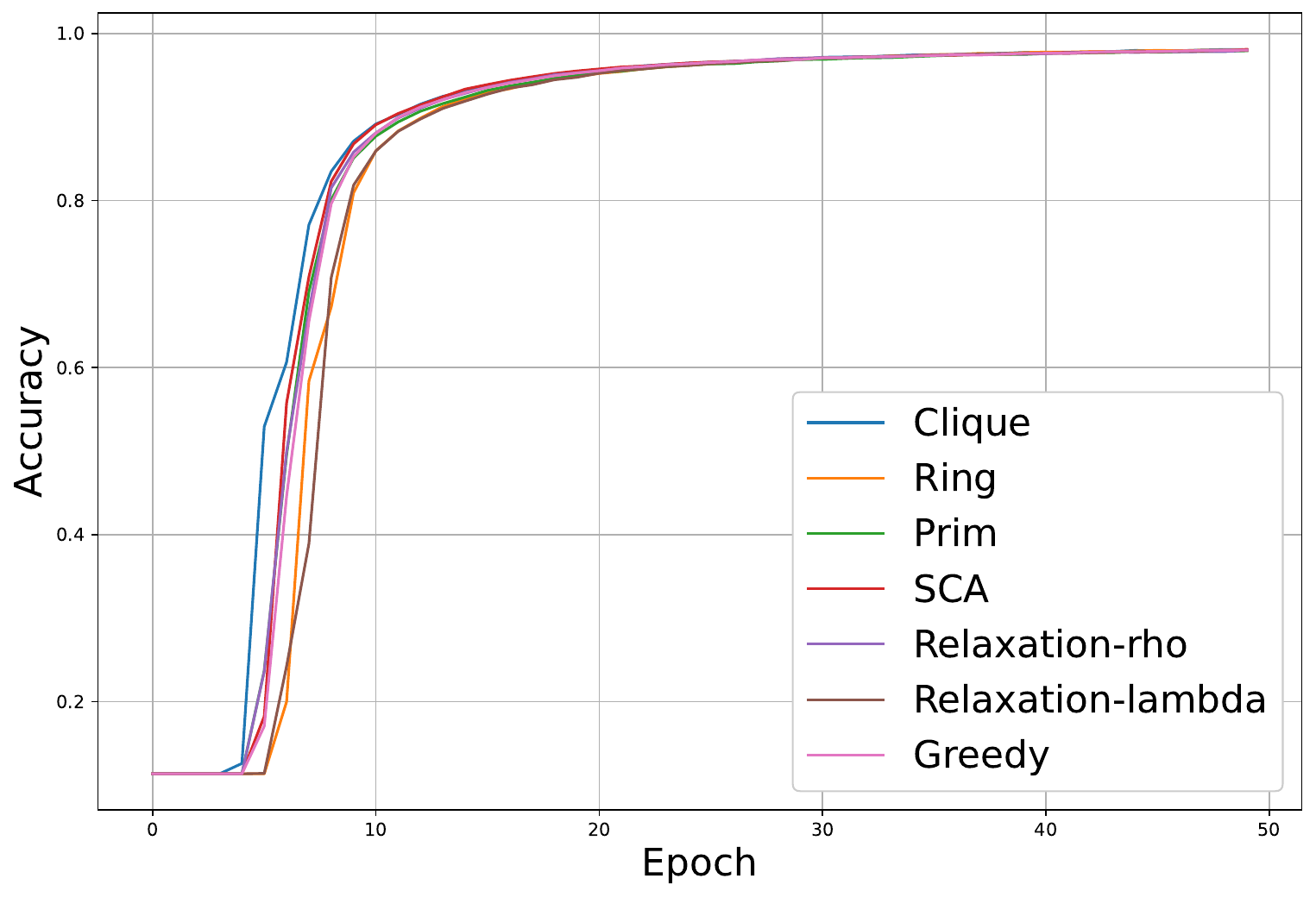}}
\vspace{-.1em}
\end{minipage}
\begin{minipage}{.495\linewidth}
\centerline{
\includegraphics[width=1\linewidth]{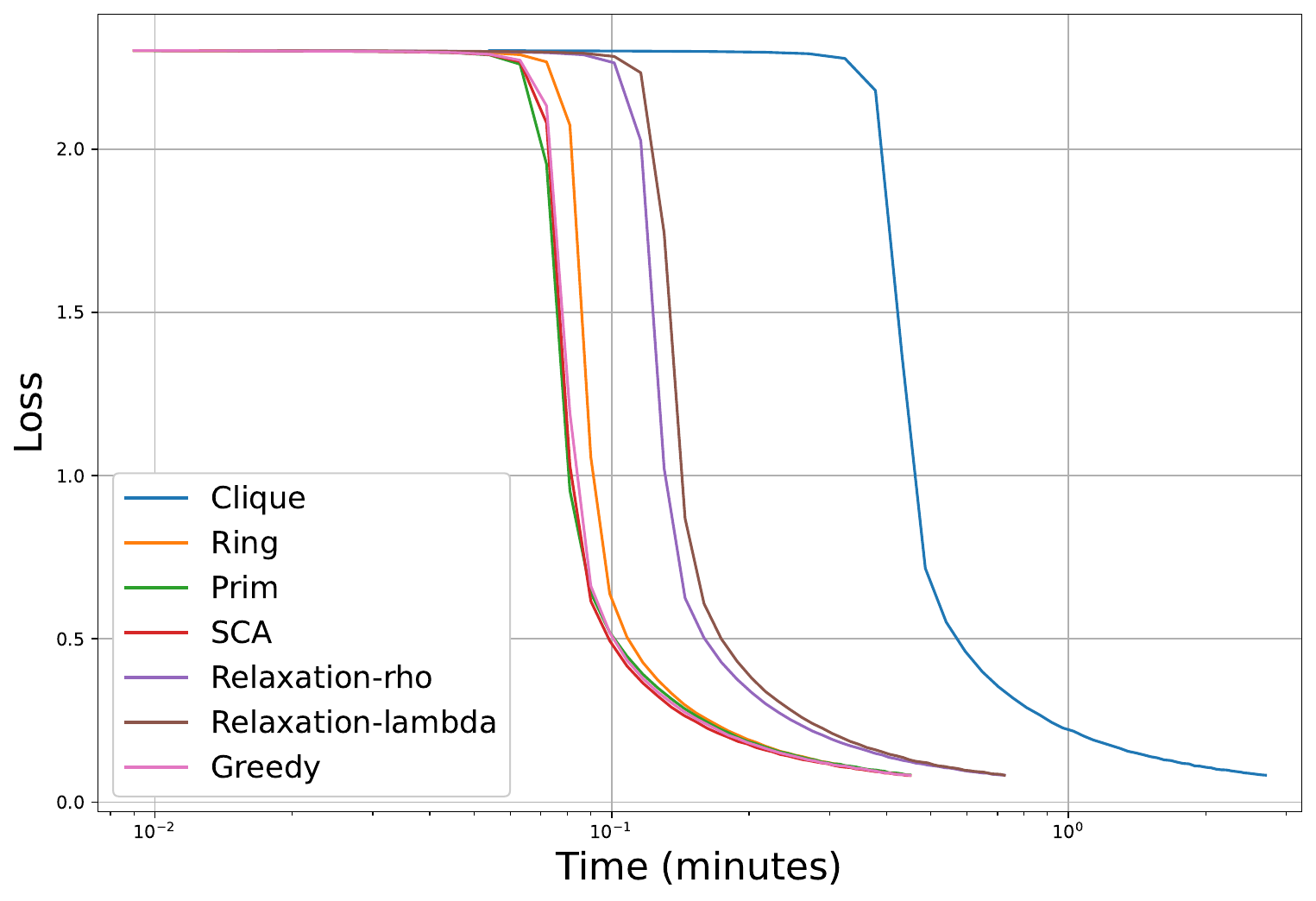}}
\vspace{-.1em}
\end{minipage}
\begin{minipage}{.495\linewidth}
\centerline{
\includegraphics[width=1\linewidth]{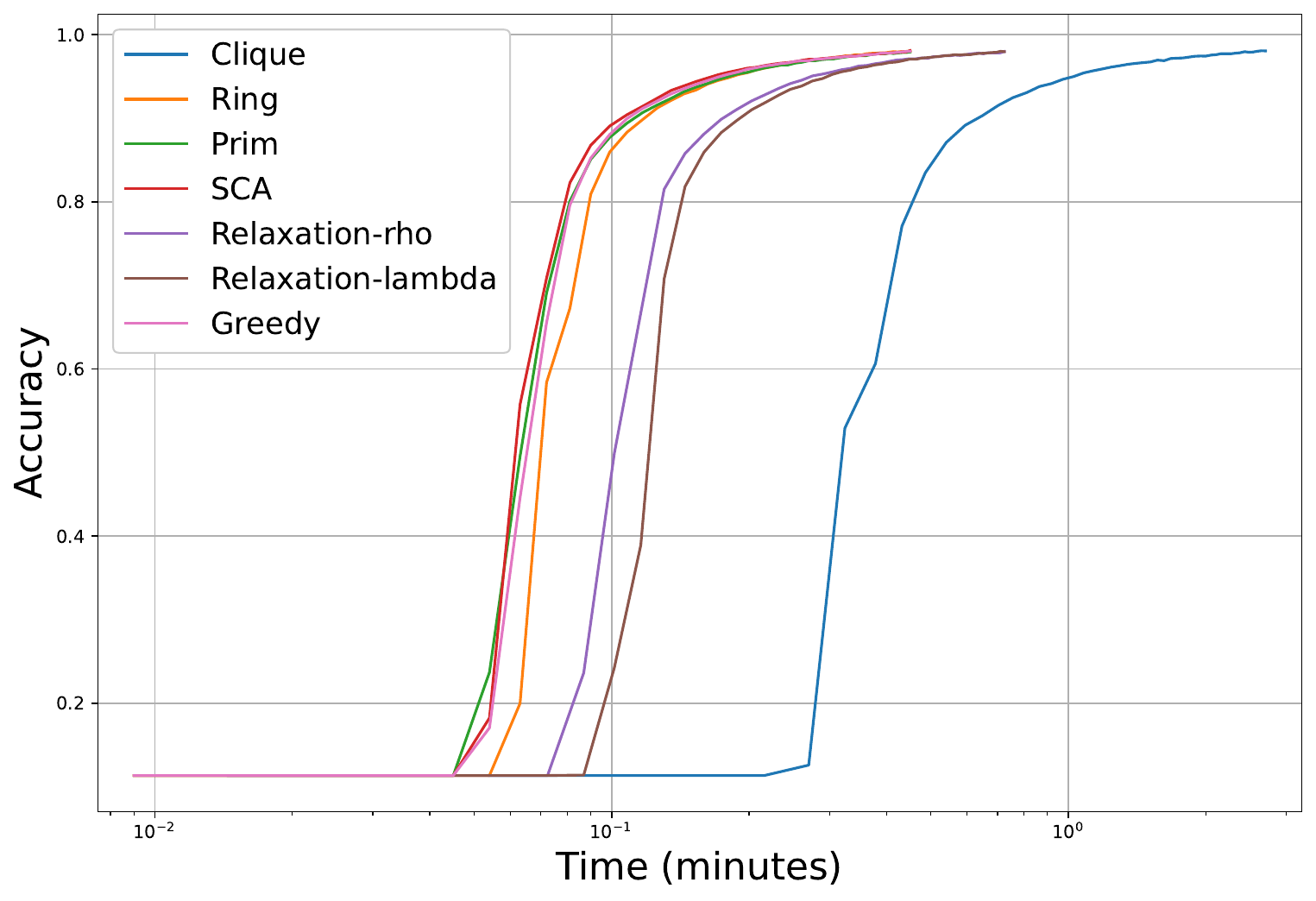}}
\vspace{-.1em}
\end{minipage}
\begin{minipage}{.495\linewidth}
\centerline{
\includegraphics[width=1\linewidth]{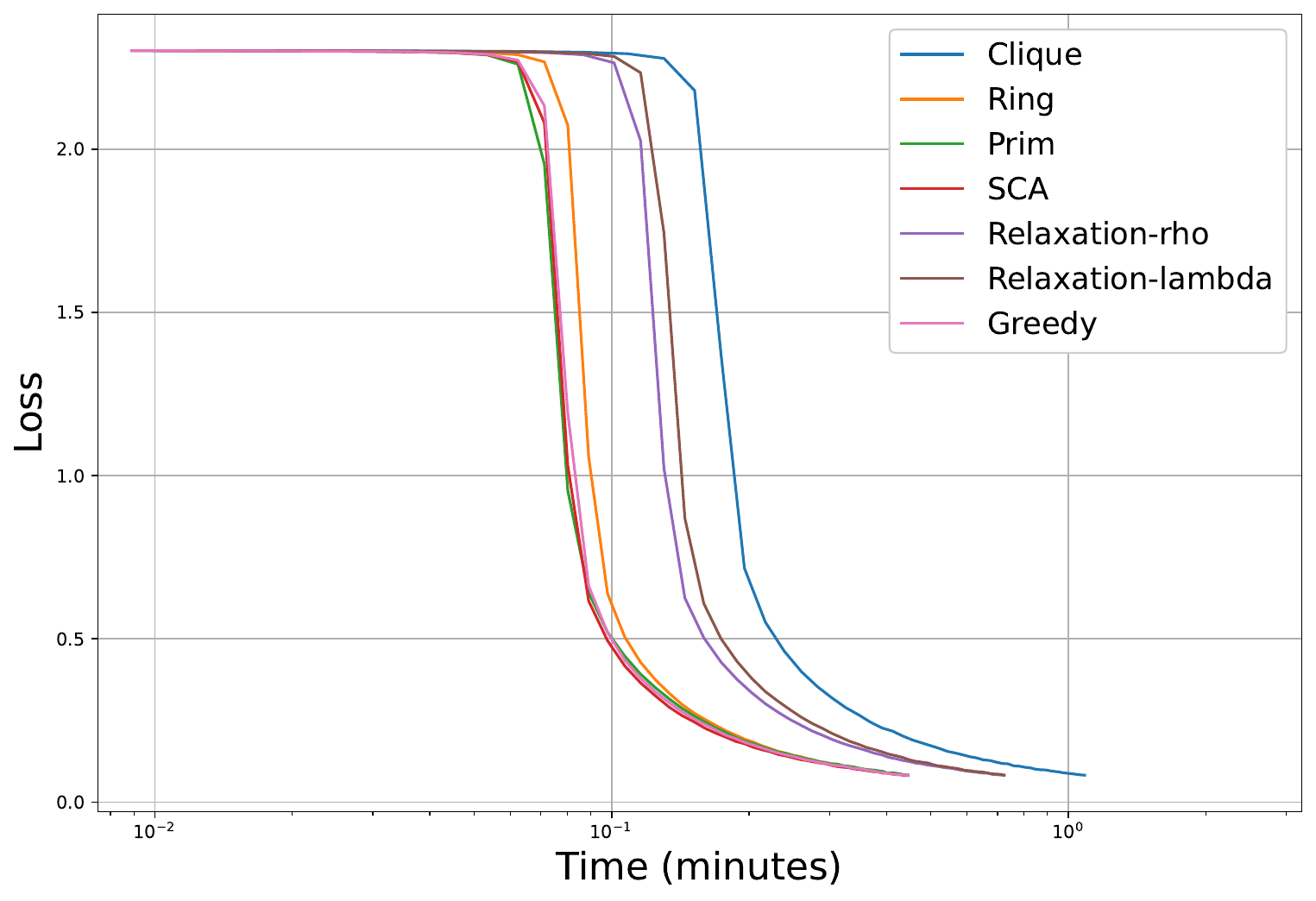}}
\vspace{-.1em}
\end{minipage}
\begin{minipage}{.495\linewidth}
\centerline{
\includegraphics[width=1\linewidth]{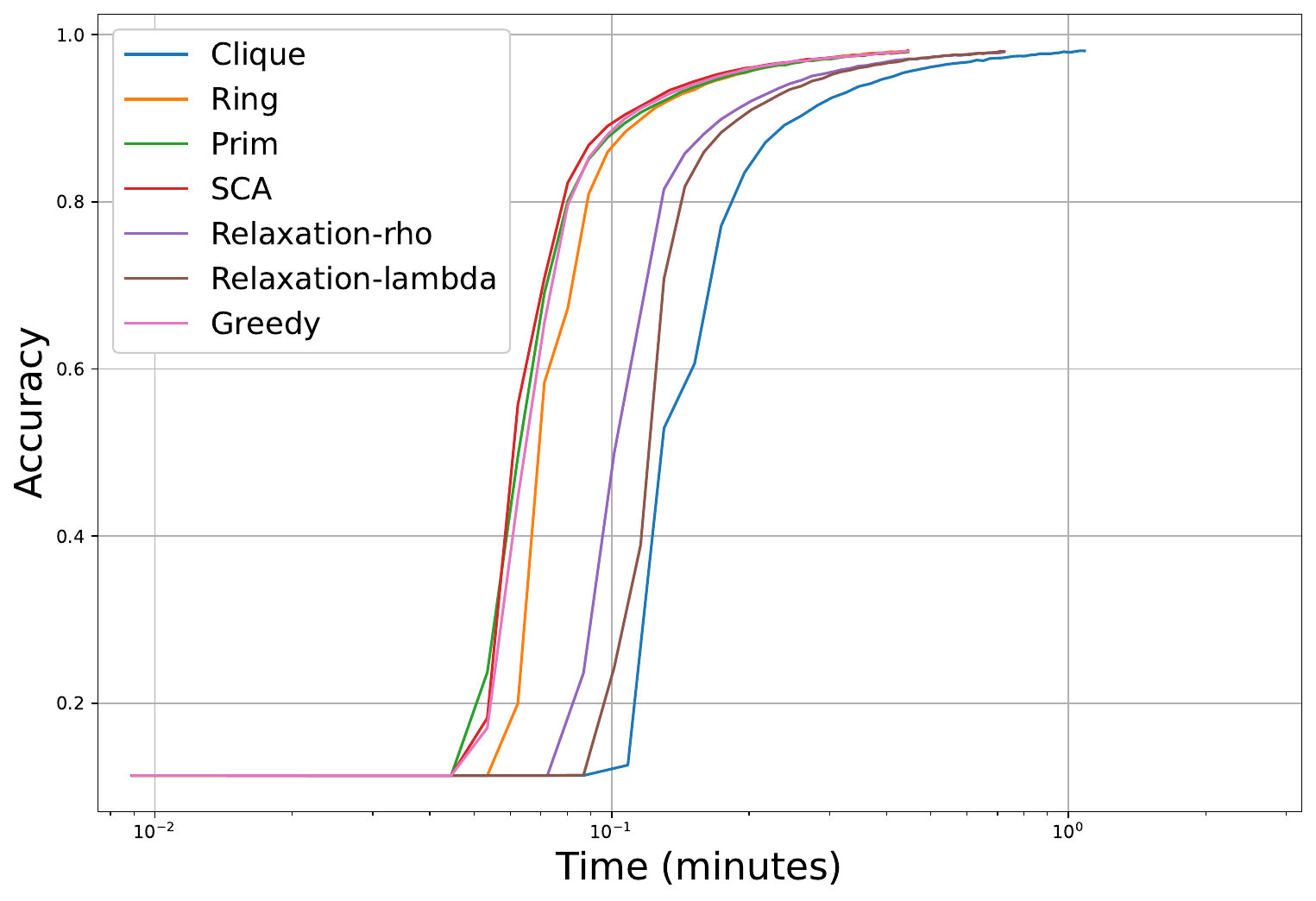}}
\vspace{-.1em}
\end{minipage}
\vspace{-1.25em}
\caption{MNIST over IAB with inference errors (second row: time without overlay routing; third row: time with overlay routing). 
} \label{fig:IAB_MNIST_inference}
\vspace{-.5em}
\end{figure}

\subsection{Summary of Results}

Tables~\ref{tab:cifar10-roofnet}--\ref{tab:mnist-iab} summarize the above results in terms of the total wall clock time for DFL to reach convergence, defined as the point at which the variance of accuracy within a sliding window falls below a specified threshold\footnote{Specifically, we consider convergence to be achieved when the variance of accuracy within a window of three consecutive epochs remains below $0.005$ for at least three consecutive windows.}. The results clearly show that (i) using a sparse topology for parameter exchange can significantly reduce the training time compared to using the complete topology, (ii) overlay routing can further improve the training time, although the improvement is much less, and (iii) the state-of-the-art network tomography algorithm \cite{Huang24TONsub} is accurate enough to support the optimization.

\begin{table}[t!]
\centering
\begin{tabular}{lcccc}
\toprule
 & \multicolumn{2}{c}{With routing} & \multicolumn{2}{c}{Without routing} \\ \cmidrule(r){2-3} \cmidrule(r){4-5}
 & W/o error & W/ error & W/o error & W/ error \\ \midrule
Greedy & 2.5909 & 2.6874 & 2.7034 & 2.7429 \\ 
SCA & 0.6586 & 0.6586 & 0.6717 & 0.6717 \\ 
Relaxation-$\rho$ & 2.1694 & 2.2477 & 2.2610 & 2.2940 \\ 
Relaxation-$\lambda$ & 2.6171 & 2.6132 & 2.7535 & 2.7525 \\ 
Prim & 1.4306 & 1.4306 & 1.4746 & 1.4746 \\ 
Ring & 1.3352 & 1.3352 & 1.3763 & 1.3763 \\ 
Clique & 5.4787 & 5.4787 & 7.8154 & 7.8154 \\ 
\bottomrule
\end{tabular}
\vspace{0.25cm}
\caption{CIFAR10-Roofnet: total wall clock time in seconds (all values \(\times 10^5\)).}
\label{tab:cifar10-roofnet}
\end{table}

\begin{table}[t!]
\centering
\begin{tabular}{lcccc}
\toprule
 & \multicolumn{2}{c}{With routing} & \multicolumn{2}{c}{Without routing} \\ \cmidrule(r){2-3} \cmidrule(r){4-5}
 & W/o error & W/ error & W/o error & W/ error \\ \midrule
Greedy & 7.9746 & 7.9725 & 8.3035 & 8.3035 \\ 
SCA & 5.3744 & 5.3744 & 5.5357 & 5.5357 \\ 
Relaxation-$\rho$ & 7.9767 & 7.9740 & 8.3035 & 8.3083 \\ 
Relaxation-$\lambda$ & 14.9229 & 14.9106 & 15.9683 & 15.9680 \\ 
Prim & 7.5686 & 7.5481 & 7.8568 & 7.8572 \\ 
Ring & 7.2684 & 7.2597 & 7.5546 & 7.5549 \\ 
Clique & 20.8054 & 20.8011 & 23.0065 & 22.9944 \\ 
\bottomrule
\end{tabular}
\vspace{0.25cm}
\caption{MNIST-Roofnet: total wall clock time in seconds (all values \(\times 10^3\)).}
\label{tab:mnist-roofnet}
\end{table}

\begin{table}[t!]
\centering
\begin{tabular}{lcccc}
\toprule
 & \multicolumn{2}{c}{With routing} & \multicolumn{2}{c}{Without routing} \\ \cmidrule(r){2-3} \cmidrule(r){4-5}
 & W/o error & W/ error & W/o error & W/ error \\ \midrule
Greedy & 3.0271 & 3.0308 & 3.0310 & 3.0767 \\ 
SCA & 3.4362 & 3.4404 & 3.4406 & 3.4926 \\ 
Relaxation-$\rho$ & 5.0376 & 5.0376 & 5.0381 & 5.1141 \\ 
Relaxation-$\lambda$ & 4.7918 & 4.7918 & 4.7923 & 4.8647 \\ 
Prim & 3.7635 & 3.7680 & 3.7683 & 3.8252 \\ 
Ring & 3.6817 & 3.6861 & 3.6864 & 3.7421 \\ 
Clique & 11.4697 & 11.4671 & 27.5251 & 27.5251 \\ 
\bottomrule
\end{tabular}
\vspace{0.25cm}
\caption{CIFAR10-IAB: total wall clock time in seconds (all values \(\times 10^2\)).}
\label{tab:cifar10-iab}
\end{table}

\begin{table}[t!]
\centering
\begin{tabular}{lcccc}
\toprule
 & \multicolumn{2}{c}{With routing} & \multicolumn{2}{c}{Without routing} \\ \cmidrule(r){2-3} \cmidrule(r){4-5}
 & W/o error & W/ error & W/o error & W/ error \\ \midrule
Greedy & 14.3586 & 14.4315 & 14.3721 & 14.5881 \\ 
SCA & 14.3586 & 14.4315 & 14.3721 & 14.5881 \\ 
Relaxation-$\rho$ & 19.9575 & 21.7025 & 19.9600 & 21.7225 \\ 
Relaxation-$\lambda$ & 20.7558 & 22.5706 & 20.7584 & 22.5914 \\ 
Prim & 13.8372 & 13.8970 & 13.8398 & 14.0578 \\ 
Ring & 13.8372 & 13.8970 & 13.8398 & 14.0578 \\ 
Clique & 34.5904 & 33.8650 & 83.0336 & 84.2764 \\ 
\bottomrule
\end{tabular}
\vspace{0.25cm}
\caption{MNIST-IAB: total wall clock time in seconds.}
\label{tab:mnist-iab}
\end{table}

\fi

\section{Conclusion}\label{sec:Conclusion}

We considered, for the first time, communication optimization for running DFL on top of a bandwidth-limited underlay network. To this end, we formulated a framework for jointly optimizing the hyperparameters controlling the communication demands between learning agents and the communication schedule (including routing and flow rates) to fulfill such demands, without cooperation from the underlay. We showed that the resulting problem can be decomposed into a set of interrelated subproblems, and developed efficient algorithms through carefully designed convex relaxations. Our evaluations based on real topologies and datasets validated the efficacy of the proposed solution in significantly reducing the training time without compromising the quality of the trained model. 
Our results highlight the need of network-application co-design in supporting DFL over bandwidth-limited networks, and our overlay-based approach facilitates the deployment of our solution in existing networks without changing their internal operations.


\bibliographystyle{ACM-Reference-Format}
\bibliography{references.bib, survey_backup.bib}

\if\thisismainpaper1

\else
\appendix

\addcontentsline{toc}{section}{Appendices}
\renewcommand{\thesubsection}{\Alph{subsection}}

\subsection{Supporting Proofs}\label{appendix:Proofs}

\begin{proof}[Proof of Lemma~\ref{lem:equal bandwidth allocation}]
The rate of each multicast flow $h\in H$ is determined by the minimum rate of the unicast flows constituting it. Consider the bottleneck underlay link $\ue^* := \argmin_{\ue\in \uE} C_{\ue}/t_{\ue}$. Since there are $t_{\ue^*}$ unicast flows sharing a total bandwidth of $C_{\ue^*}$ at $\ue^*$, the slowest of these flows cannot have a rate higher than $C_{\ue^*}/t_{\ue^*}$. Thus, the multicast flow containing this slowest unicast flow cannot have a rate higher than $C_{\ue^*}/t_{\ue^*}$, which means that the completion time for all the multicast flows is no smaller than \eqref{eq:tau - special case, per-link}. 

Meanwhile, if the bandwidth of every link is shared equally among the activated unicast flows traversing it, then each unicast flow will receive a bandwidth allocation of no less than $C_{\ue^*}/t_{\ue^*}$ at every hop, and thus can achieve a rate of at least $C_{\ue^*}/t_{\ue^*}$. Hence, each multicast flow $h\in H$ can achieve a rate of at least $C_{\ue^*}/t_{\ue^*}$, yielding a completion time of no more than \eqref{eq:tau - special case, per-link}. 
\end{proof}

\begin{proof}[Proof of Lemma~\ref{lem:equal bandwidth allocation - category}]
According to Lemma~\ref{lem:equal bandwidth allocation}, it suffices to prove that $\min_{F\in \mathcal{F}} C_F/t_F = \min_{\ue\in \uE} C_{\ue}/t_{\ue}$. To this end, we first note that by Definition~\ref{def: category}, all the underlay links in the same category must be traversed by the same set of overlay links and thus the same set of activated unicast flows, i.e., $t_{\ue} = t_F$ $\forall \ue\in \Gamma_F$. By the definition of the category capacity $C_F$, we have 
\begin{align}
\min_{\ue\in \Gamma_F} {C_{\ue}\over t_{\ue}} = \min_{\ue\in \Gamma_F} {C_{\ue}\over t_F} = {C_F\over t_F}.
\end{align}
Thus, we have
\begin{align}
\min_{\ue\in \uE} {C_{\ue}\over t_{\ue}} = \min_{F\in \mathcal{F}} \min_{\ue\in \Gamma_F} {C_{\ue}\over t_{\ue}} = \min_{F\in \mathcal{F}} {C_F\over t_F}. 
\end{align}
\end{proof}

\begin{proof}[Proof of Corollary~\ref{cor:optimal link weights}]
As $K(p,1)$ decreases with $p$, its minimum is achieved at the maximum value of $p$ that satisfies \eqref{eq:condition on p} for $t=1$ and any value of $\bm{X}$, i.e., 
\begin{align}
    p := \min_{\bm{X}\neq \bm{0}} \left(1-{\E[\|\bm{X}(\bm{W}-\bm{J})\|_F^2]\over \|\bm{X}(\bm{I}-\bm{J})\|_F^2}\right).\label{eq:p-maximization objective}
\end{align}
By \cite[Lemma~3.1]{Xusheng24ICASSP}, $p$ defined in \eqref{eq:p-maximization objective} satisfies $p = 1-\tilde{\rho}$ for $\tilde{\rho} := \|\E[\bm{W}^\top\bm{W}]-\bm{J}\|$. By Jensen's inequality and the convexity of $\|\cdot\|$, $\tilde{\rho} \leq \E[\|\bm{W}^\top \bm{W} - \bm{J}\|]$. 
For every realization of $\bm{W}$ that is symmetric with rows/columns summing to one, 
we have $\bm{W}^\top\bm{W}-\bm{J}=(\bm{W}-\bm{J})^2$. Based on the eigendecomposition $\bm{W}-\bm{J} = \bm{Q}\diag(\lambda_1,\ldots,\lambda_m)\bm{Q}^\top$, we have\looseness=-1 
\begin{align}
\|\bm{W}^\top\bm{W}-\bm{J}\| &= \|\bm{Q}\diag(\lambda_1^2,\ldots,\lambda_m^2)\bm{Q}^\top \| \nonumber\\
& = \max_{i=1,\ldots,m} \lambda_i^2 = \|\bm{W}-\bm{J}\|^2, 
\end{align}
where we have used the fact that $\|\bm{W}-\bm{J}\| = \max_{i=1,\ldots,m}|\lambda_i|$. 
Thus, $K(p,1)$ for $p$ defined in \eqref{eq:p-maximization objective} is upper-bounded by $K(1- \E[\|\bm{W}-\bm{J}\|^2], 1)$, which is a sufficient number of iterations for D-PSGD to achieve $\epsilon_0$-convergence by Theorem~\ref{thm:new convergence bound}. 

The matrix inequality \eqref{wo cost:matrix} implies that $\rho\geq |\lambda_i|$ for all $i=1,\ldots,m$, and thus the optimal value of \eqref{eq:min rho wo cost} must satisfy $\rho = \max_{i=1,\ldots,m}|\lambda_i| = \|\bm{W}-\bm{J}\|$. Hence, the optimal value $\rho^*$ of \eqref{eq:min rho wo cost} is the minimum value of $\| \bm{W}-\bm{J}\|$ for any realization of $\bm{W}$ that only activates the links in $E_a$. Therefore, $1-\E[\|\bm{W}-\bm{J}\|^2]\leq 1- \mathop{\rho^*}^2$ and \eqref{eq:relaxed bound on K} $\geq K(1- \mathop{\rho^*}^2,1)$, with ``$=$'' achieved at $\bm{W}^* = \bm{I}-\bm{B}\diag(\bm{\alpha}^*)\bm{B}^\top$. 
\end{proof}

\begin{proof}[Proof of Lemma~\ref{lem:equivalence to bilevel}]
Let $(\beta^*, E_a^*)$ be the optimal solution to the RHS of \eqref{eq:equivalance to bilevel}, and $E_a^o$ be the optimal solution to the LHS of \eqref{eq:equivalance to bilevel}. Let $\beta^o:= \overline{\tau}(E_a^o)$. Then
\begin{align}
&\min_{\overline{\tau}(E_a)\leq \beta^o}\overline{K}(E_a) \leq \overline{K}(E_a^o) \\
\Rightarrow& \beta^o \cdot \left(\min_{\overline{\tau}(E_a)\leq \beta^o}\overline{K}(E_a) \right) \leq \overline{\tau}(E_a^o)\cdot \overline{K}(E_a^o) \\
\Rightarrow& \min_{\beta\geq 0} \beta \cdot \left(\min_{\overline{\tau}(E_a)\leq \beta} \overline{K}(E_a) \right)\leq \overline{\tau}(E_a^o)\cdot \overline{K}(E_a^o). \label{eq:equivalence proof - 1}
\end{align}
Meanwhile, $\beta^*$ must equal $\overline{\tau}(E_a^*)$, as otherwise we can reduce $\beta^*$ to further reduce the value of $\beta \cdot \left(\min_{\overline{\tau}(E_a)\leq \beta} \overline{K}(E_a) \right)$, contradicting with the assumption that $(\beta^*, E_a^*)$ is optimal. Therefore, by the definition of $E_a^o$,
\begin{align}
\min_{\beta\geq 0} \beta \cdot \left(\min_{\overline{\tau}(E_a)\leq \beta} \overline{K}(E_a) \right) &= \overline{\tau}(E_a^*) \cdot \overline{K}(E_a^*) \nonumber\\
&\geq \overline{\tau}(E_a^o)\cdot \overline{K}(E_a^o), \label{eq:equivalence proof - 2}
\end{align}
which together with \eqref{eq:equivalence proof - 1} proves \eqref{eq:equivalance to bilevel}.

Moreover, \eqref{eq:equivalance to bilevel} implies that ``$=$'' must hold for \eqref{eq:equivalence proof - 2}, i.e., $E_a^*$ is also optimal for the LHS of \eqref{eq:equivalance to bilevel}. 
\end{proof}

\begin{proof}[Proof of Lemma~\ref{lem:reduction to algebraic connectivity}]
It suffices to prove that under condition \eqref{eq:requirement on alpha^0}, \eqref{eq:bound rho - 2} is achieved at $1-\lambda_2(\bm{L}(E_a)),\: \forall E_a\subseteq \widetilde{E}$, i.e., $\lambda_2(\bm{L}(E_a)) + \lambda_m(\bm{L}(E_a))\leq 2$. 

By definition, $\lambda_m(\bm{L}(E_a)) = \max\{\bm{v}^\top \bm{L}(E_a)\bm{v}:\: \|\bm{v}\|=1\}$. Also by definition, $\bm{L}(E_a) = \bm{D}(E_a) - \bm{A}(E_a)$, where $\bm{D}(E_a)$ and $\bm{A}(E_a)$ are the degree matrix and the adjacency matrix for a weighted graph with link weights $\bm{\alpha}^{(0)}(E_a)$. We have
\begin{align}
\bm{v}^\top \bm{D}(E_a) \bm{v} &= \sum_{i=1}^m v_i^2 \sum_{j: (i,j)\in E_a}\alpha^{(0)}_{ij} \nonumber\\
&\leq \max_{i\in V}\sum_{j: (i,j)\in \widetilde{E}}\alpha^{(0)}_{ij}, \label{eq:proof - algebraic - 1}
\end{align}
because $\sum_{i=1}^m v_i^2 = 1$ and $\sum_{j: (i,j)\in E_a}\alpha^{(0)}_{ij} \leq \sum_{j: (i,j)\in \widetilde{E}}\alpha^{(0)}_{ij}$. 
Moreover, we also have
\begin{align}
-\bm{v}^\top \bm{A}(E_a) \bm{v} &\leq \sum_{i,j=1}^m |(A(E_a))_{ij}|\cdot |v_i|\cdot |v_j| \nonumber \\
&\leq (\max_{(i,j)\in \widetilde{E}}|\alpha^{(0)}_{ij}|) \sqrt{\sum_{i,j=1}^m v_i^2} \sqrt{\sum_{i,j=1}^m v_j^2} \label{eq:proof - algebraic - 2} \\
&= m \cdot \max_{(i,j)\in \widetilde{E}}|\alpha^{(0)}_{ij}|, \label{eq:proof - algebraic - 3}
\end{align}
where \eqref{eq:proof - algebraic - 2} is because of $|(A(E_a))_{ij}|\leq \max_{(i,j)\in \widetilde{E}}|\alpha^{(0)}_{ij}|$ and the Cauchy-Schwarz inequality, and \eqref{eq:proof - algebraic - 3} is because $\sum_{i=1}^m v_i^2 = \sum_{j=1}^m v_j^2 = 1$. 
Combining \eqref{eq:requirement on alpha^0}, \eqref{eq:proof - algebraic - 1}, and \eqref{eq:proof - algebraic - 3} implies that $\bm{v}^\top \bm{L}(E_a)\bm{v} \leq 1$ for any unit-norm vector $\bm{v}$. Thus, $\lambda_2(\bm{L}(E_a))\leq \lambda_m(\bm{L}(E_a))\leq 1$, completing the proof. 
\end{proof}

\begin{proof}[Proof of Theorem~\ref{thm:overall solution}]
Under the assumption of $\widehat{\mathcal{F}}\supseteq \mathcal{F}$ and $\widehat{C}_F \leq C_F$ ($\forall F\in \mathcal{F}$), every real (per-category) capacity constraint is ensured by a capacity constraint we formulate based on the inferred parameters $\widehat{\mathcal{F}}$ and $(\widehat{C}_F)_{F\in \widehat{\mathcal{F}}}$, and thus any communication schedule that is feasible under the inferred constraints remains feasible under the actual constraints. This implies that the proposed design  predicted to complete each iteration in time $\overline{\tau}(E_a^*)$ can actually complete each iteration within this time. 
Moreover, by Corollary~\ref{cor:optimal link weights}, D-PSGD under the designed mixing matrix achieves $\epsilon_0$-convergence within $K(E_a^*)$ iterations, which is further bounded by $\overline{K}(E_a^*)$ according to \eqref{eq:K_bound}. Thus, D-PSGD under the proposed design can achieve $\epsilon_0$-convergence within time $\overline{\tau}(E_a^*) \cdot \overline{K}(E_a^*)$.  
\end{proof}

\fi

\end{document}